\newcommand{\classstyle}{1}

\ifnum\classstyle=1
\documentclass[12pt]{article}
\usepackage[utf8]{inputenc}
\usepackage[pdftex,dvipsnames]{xcolor}  
\usepackage{xargs}                
\usepackage{pdfpages}

\usepackage{fullpage}
\usepackage{times}
\usepackage[normalem]{ulem}
\usepackage{pifont}
\usepackage{graphicx,amsmath,amssymb, mathtools, scrextend, titlesec, enumitem}
\usepackage{bm}
\usepackage{tcolorbox}
\usepackage{amsthm}
\usepackage{mathrsfs}

\usepackage{multicol}

\usepackage[hypertexnames=false]{hyperref}
\hypersetup{
    colorlinks = true,
    linkbordercolor = {white},
    linkcolor = {blue!50!black},
citecolor     = {blue!50!black},
}

\usepackage{caption}
\usepackage{subcaption}

\usepackage{algpseudocode}

\usepackage{tikz}
\usepackage[mode=buildnew]{standalone}
\usetikzlibrary{shapes.geometric, arrows,calc}
\usepackage{pgfplots}

\tikzstyle{startstop} = [rectangle, rounded corners, 
minimum width=3cm, 
minimum height=1cm,
text centered, 
draw=black, 
fill=blue!10!white]

\tikzstyle{io} = [trapezium, 
trapezium stretches=true, 
trapezium left angle=70, 
trapezium right angle=110, 
minimum width=3cm, 
minimum height=1cm, text centered, 
draw=black, fill=blue!30]

\tikzstyle{process} = [rectangle, 
minimum width=3cm, 
minimum height=1cm, 
text centered, 
text width=3cm, 
draw=black, 
fill=orange!30]

\tikzstyle{decision} = [diamond, 
minimum width=3cm, 
minimum height=1cm, 
text centered, 
draw=black, 
fill=green!30]
\tikzstyle{arrow} = [thick,->,>=stealth]

\usepackage[colorinlistoftodos,prependcaption,textsize=tiny]{todonotes}
\newcommandx{\unsure}[2][1=]{\todo[linecolor=red,backgroundcolor=red!25,bordercolor=red,#1]{#2}}
\newcommandx{\change}[2][1=]{\todo[linecolor=blue,backgroundcolor=blue!25,bordercolor=blue,#1]{#2}}
\newcommandx{\info}[2][1=]{\todo[linecolor=OliveGreen,backgroundcolor=OliveGreen!25,bordercolor=OliveGreen,#1]{#2}}
\newcommandx{\improvement}[2][1=]{\todo[linecolor=Plum,backgroundcolor=Plum!25,bordercolor=Plum,#1]{#2}}
\newcommandx{\thiswillnotshow}[2][1=]{\todo[disable,#1]{#2}}

\usepackage{chngcntr}
\counterwithin{figure}{section}

\usepackage{ulem}

\usepackage{yhmath}

\newcommand{\R}{\mathbb R}
\newcommand{\N}{\mathbb N}

\DeclareMathOperator{\E}{\mathbb E}

\newcommand{\dd}{\ensuremath{\mathrm{d}}}

\makeatletter
\@namedef{subjclassname@2020}{%
  \textup{2020} Mathematics Subject Classification}
\makeatother

\usepackage{enumitem}
\usepackage{amsfonts}
\usepackage{ stmaryrd }
\usepackage[ruled]{algorithm2e}
\usepackage[nameinlink,capitalize]{cleveref}

\usepackage{thmtools}

\declaretheorem[
    name=Theorem,
    numberwithin=section]{theorem}
\declaretheorem[
    name=Assumption,
    sibling=theorem]{assumption}

\numberwithin{equation}{section}
\newtheorem{definition}[theorem]{Definition}

\newtheorem{remark}[theorem]{Remark}
\newtheorem{lemma}[theorem]{Lemma}
\newtheorem{proposition}[theorem]{Proposition}

\usepackage{nicematrix}

\newlist{thmlist}{enumerate}{1}
\setlist[thmlist]{label=(\roman{thmlisti}),
ref=\thetheorem~(\roman{thmlisti}),
noitemsep}

\newtheorem{assump}[theorem]{Assumption}

\Crefname{assump}{Assumption}{Assumptions}
\Crefname{listthm}{Theorem}{Theorems}
\Crefname{listassump}{Assumption}{Assumptions}
\Crefname{listdef}{Definition}{Definition}

\addtotheorempostheadhook[theorem]{\crefalias{thmlisti}{listthm}}
\addtotheorempostheadhook[assump]{\crefalias{thmlisti}{listassump}}
\addtotheorempostheadhook[definition]{\crefalias{thmlisti}{listdef}}

\usepackage{arxiv}
\fi

\newcommand{\ownKeywords}{Langevin Monte Carlo, Approximate sampling, Rates of
convergence, Neural Network approximation, Wasserstein distance}

\newcommand{\ownAMS}{
  62F15,   	
  65N75,   	
  65C30,   	
  60H35, 
  62H12, 
  65C05, 
  60H35, 
  68T07 
}

\ifnum\classstyle=1
\title{Approximating Langevin Monte Carlo with ResNet-like Neural Network architectures}

\author{
    \textbf{Charles Miranda}\thanks{equal contribution} \\
    Weierstrass Institute for \\ Applied Analysis and Stochastics \\
    Berlin, Germany \\
    Centrale Nantes, Nantes Université \\Laboratoire de Mathématiques Jean Leray\\ UMR CNRS 6629\\ France\\
    \texttt{miranda@wias-berlin.de}
    \And
	\textbf{Janina Enrica Schütte}\footnotemark[1] \\
	Weierstrass Institute for \\ Applied Analysis  and Stochastics \\
	Berlin, Germany \\
	\texttt{schuette@wias-berlin.de}
	\And
	\textbf{David Sommer}\footnotemark[1]~~\thanks{corresponding author} \\
	Weierstrass Institute for \\ Applied Analysis  and Stochastics \\
	Berlin, Germany \\
	\texttt{sommer@wias-berlin.de} 
 \And
 \textbf{Martin Eigel} \\
	Weierstrass Institute for \\ Applied Analysis  and Stochastics \\
	Berlin, Germany \\
	\texttt{eigel@wias-berlin.de}
}

\renewcommand{\shorttitle}{Langevin NNs}
\fi

\begin{document}
\maketitle
\setcounter{page}{1}

 \begin{abstract}
 We analyse a method to sample from a given target distribution by constructing a neural network which maps samples from a simple reference distribution, e.g. the standard normal, to samples from the target distribution. For this, we propose using a neural network architecture inspired by the Langevin Monte Carlo (LMC) algorithm. Based on LMC perturbation results, approximation rates of the proposed architecture for smooth, log-concave target distributions measured in the Wasserstein-$2$ distance are shown. The analysis heavily relies on the notion of sub-Gaussianity of the intermediate measures of the perturbed LMC process. In particular, we derive bounds on the growth of the intermediate variance proxies under different assumptions on the perturbations. Moreover, we propose an architecture similar to deep residual neural networks (ResNets) and derive expressivity results for approximating the sample to target distribution map.
\end{abstract}

\ifnum\classstyle=1
   \textit{\textbf{2020 Mathematics Subject Classification}} \ownAMS
\fi

\section{Introduction and scope}

In recent years there has been a lot of attention on research of sampling problems, which we define in this work as the task of sampling from an  (unnormalized) density, given in functional form. 
A particularly interesting application of such methods is the computation of statistical properties of posterior measures in the framework of Bayesian inference, see e.g.~\cite{freeenergy}.
This work is concerned with the analysis of a sampling method based on the Langevin Monte Carlo (LMC) algorithm for particle transport as e.g. developed in~\cite{roberts1996convergence, garbuno2020affine, eigel2024less}.
Such algorithms have become increasingly popular recently as a means to simulate the transport of a discrete probability measure, typically from a simple reference to some complicated target. 
Recent interest shows a focus on practical improvements, e.g. by introducing particle interaction and affine invariance, see \cite{garbuno2020affine, garbuno2020interacting, eigel2024less}.
In this work, we consider sub-Gaussian properties of LMC trajectories and the representation complexity of the vanilla LMC sampler by means of an appropriate Neural Network (NN).
NNs have become ubiquitous in basically all scientific disciplines, in particular in the area of scientific computing.
In addition to the development of advantageous NN architectures for practical problems, their expressivity analysis has also attracted significant attention, see e.g.~\cite{gühring_raslan_kutyniok_2022, Berner_2022, yarotsky2017error} and references therein.
A main reason is that expressivity results provide clear indications with respect to required network complexities and enable to infer the approximation quality NNs exhibit for specific tasks and methods. 
While this work is of theoretical nature and focuses on deriving complexity bounds for NN approximations of the LMC algorithm, we give clear numerical motivation for this type of surrogate in~\cref{sec:Numerics}.

\subsection{Related work}\label{sec:relatedWork}

\paragraph{Deep Neural Networks} The training of deep neural networks (DNNs) to sample from distributions has become widespread, for instance in the field of generative modelling (GM)~\cite{ruthotto2021introduction, bond2021deep}. Popular approaches in deep generative modelling (DGM) include normalizing
flows~\cite{ruthotto2021introduction}, variational autoencoders \cite{NIPS2016_ddeebdee,kingma2013auto,pmlr-v32-rezende14}, and generative adversarial networks (GANs) \cite{goodfellow2016nips}. Beyond DGM, there is also growing interest in GM via score-based diffusion models \cite{yang2022diffusion,song2020score}, where DNNs are trained to approximate the score function, i.e. the gradient log-density of the forward diffusion process.
While most research has been focused on image, video and text generation, generative models can also be used in the context of differential equations related to engineering and the natural sciences~\cite{dandekar2022bayesian}.

From a mathematical perspective, there has been a lot of work on the expressivity analysis of fully connected neural networks (FCNNs), providing qualitative results regarding the complexity of representations.
For this, classical and new approximation classes and function spaces are considered, such as in~\cite{gühring_raslan_kutyniok_2022}. Important approximation results include~\cite{bach2017breaking, barron1993universal, barron1994approximation, bolcskei2019optimal, elbrachter2022dnn, perekrestenko2018universal, petersen2018optimal, yarotsky2017error, yarotsky2018optimal, yarotsky2020phase}.
In \cite{Jentzen_2021} FCNNs have been shown to beat the curse of dimensionality in approximation of Kolmogorov backward equations (KBE) provided that the drift term can be approximated without curse of dimensionality. The central idea is the use of the Feynman-Kac formula, linking the KBE solution to the expectation of an observable subject to an underlying It\^o diffusion process. 
Under suitable assumptions on the observable (which coincides with the initial condition of the KBE) and the drift term, this diffusion process can be approximated by adding and composing FCNN layers in an imitation of the Euler-Maruyama discretization.

While standard feed-forward FCNNs offer conceptual simplicity, there are architectures arguably better suited to the approximation of differential equations. One such architecture is given by deep residual networks, also called \textit{ResNets},~\cite{he2016deep}, which instead of the full mapping from input to output learn a residual component (respective to the layer input) in each layer. Due to this residual structure, ResNets have interesting theoretical connections to time discretizations of differential equations. In particular, the forward propagation of the inputs through the residual layers can be interpreted as time-discretization of an underlying (stochastic) differential equation. The continuous-time equivalent of a ResNet is called a neural ODE~\cite{chen2018neural,sander2022residual}.

\paragraph{Sampling and Langevin Monte Carlo} Sampling from probability densities is a common problem e.g. in Bayesian inference \cite{stuart2010inverse,kaipio2006statistical} and generative modelling \cite{song2020score}. There exists a vast literature on sampling strategies, including Markov Chain Monte Carlo (MCMC)~\cite{roberts2004general,robert2011mcmc,brooks2011handbook}, Sequential Monte Carlo (SMC)~\cite{del2006sequential} and Langevin dynamics~\cite{roberts2002langevin,roberts1996convergence}.
The Langevin method has strong historical connections to statistical physics~\cite{rossky1978brownian} and can be seen as a stochastic analogue to gradient descent. Extensions of these methods include Metropolis adjusted Langevin and Hamiltonian Monte Carlo (HMC) sampling methods defined on Riemannian manifolds~\cite{girolami2011riemann} as well as ensemble methods~\cite{garbuno2020interacting,garbuno2020affine} ensuring affine invariance~\cite{goodman2010ensemble}.
Under smoothness and growth conditions on the log-density, one can define a simple first order overdamped Langevin process, which admits the target density as invariant measure and which contracts exponentially (in relative entropy) to that invariant measure~\cite{markowich2000trend}. Methods obtained by discretization of this process are called Langevin Monte Carlo (LMC) methods. Errors bounds of LMC in case of a $M$-Lipchitz, $m$-strongly convex potential have been extensively studied in Wasserstein-$2$ distance, relative entropy and total variation distance, cf.~\cite{Dalayan2016theoretical,dalalyan2017stronger, durmus2017unadjusted, Dalalyan2017UserfriendlyGF, cheng2018kl, durmus2019high, dwivedi2018fast}.
There are also works aiming to extend the convergence analysis beyond the restricted log-concave setting~\cite{luu2021sampling, cheng2018sharp, majka2020nonasymptotic, raginsky2017non, chau2021stochastic, zhang2023nonasymptotic}. A good overview of the different approaches can be found in~\cite{zhang2020wasserstein}. Another interesting work is~\cite{altschuler2022concentration}, where a bound for the variance proxy of the sub-Gaussian invariant distribution of LMC is derived.

\subsection{Methodology}

This work is concerned with sampling from measures of the form
\begin{equation*}
    \mathrm{d}\mu_\infty(x) = Z^{-1}e^{-V(x)}\mathrm{d}x,
\end{equation*} 
where $V$ is known and $Z$ is an (unknown) normalization constant. The goal is to derive complexity bounds for a neural network architecture, which takes samples 
distributed according to a sub-Gaussian reference distribution $\mu_0$ as inputs and outputs samples from $\mu_\infty$ up to an epsilon error in the Wasserstein-$2$ distance.
To achieve this, a ResNet-like architecture inspired by the LMC algorithm is proposed. In its simplest form, the LMC algorithm with step size $h>0$ takes the form $\Tilde{X}_k = \tilde{X}_{k-1} - h\nabla V(\Tilde{X}_{k-1}) + \xi_k$, $\Tilde{X}_0 \sim \mu_0$, where the $\xi_k$ are independent normally distributed increments in $\mathbb{R}^d$ with mean $0$ and covariance matrix $2hI_d$. 
A sketch of the architecture to approximate this process is shown in \Cref{fig:drn}. 
The introduced network has two important properties.
By imitating the Langevin algorithm, the number of parameters in the expressivity results mainly depends on how well the drift $-\nabla V$ can be approximated by networks $\psi_k$ as well as on the number of realized time discretization steps.
Furthermore, the architecture allows one to train the small parts $\psi_k$ of the complete network separately. The analysis combines results for Wasserstein-$2$ convergence of LMC with perturbation arguments for the FCNN-approximated drift terms. 
\begin{figure}
    \centering
    \includestandalone{images/tikz/DRN_sketch.tex}
    \caption{Sketch of the ResNet-like architecture used in this work.}
    \label{fig:drn}
\end{figure}
The analysis of the derived networks is carried out for potentials satisfying the following assumption.
\begin{assumption}\label{assump:potential}We make the following assumption on the potential $V$ of $\mu_\infty$: 
    \begin{itemize}
        \item $V$ has a $M$-Lipschitz gradient: $\forall x,y \in \mathbb R^d,\ \|\nabla V(x)-\nabla V(y)\|_{\ell^2} \leq M\|x-y\|_{\ell^2}$.
        \item $V$ is strongly convex with parameter $m$: $\forall x,y \in \mathbb R^d,\ V(x)-V(y)-\langle\nabla V(y),x-y\rangle \geq \frac m 2 \|x-y\|_{\ell^2}^2$.
    \end{itemize}
\end{assumption}
We henceforth assume that the drift $-\nabla V$ can be approximated by a NN either locally or globally, leading to different complexity bounds. We do not confine the choice of activation function of the NN at this point.
However, later results specifically require the use of the ReLU activation function.
\begin{assump} \label{assump: nn approx}
    We assume that at least one of the following assumptions holds regarding availability of FCNN approximations of the drift.
    \begin{thmlist}
\item\label{assump:nn_approx_b} For any $\varepsilon > 0$ there exists an FCNN $\phi_{\varepsilon}$ with realization $\mathcal{R}\phi_{\varepsilon}$ and $N(d,\varepsilon,m,M)$ parameters  such that
\begin{equation}
    \| -\nabla V(x) - \mathcal{R}\phi_{\varepsilon}(x) \|_{\ell^2} \leq \varepsilon(1+\|x\|_{\ell^2})
\end{equation}
for all $x\in\mathbb{R}^d$.
\item\label{assump:nn_approx_a} For any $\varepsilon > 0$ and $r>0$ there exists an FCNN $\phi_{\varepsilon,r}$ with realization $\mathcal{R}\phi_{\varepsilon,r}$, number of parameters $N(d,r,\varepsilon,m,M)$ and depth $L(d,r,\varepsilon,m,M)$ such that
$$\| -\nabla V-\mathcal{R}\phi_{\varepsilon,r} \|_{L^{\infty}(B_r(0);\mathbb{R}^d)} \leq \varepsilon/\sqrt{2},$$
where $B_r(0)$ is the closed $\ell^2$-ball with radius $r$ and center in $0$ in $\mathbb{R}^d$.
    \end{thmlist}
\end{assump}

In the analysis, a perturbed LMC process is considered, where $-\nabla V$ is replaced by neural network approximations. It is then shown that under either of the assumptions in~\cref{assump: nn approx} the global $L^2$-error of the approximation can be bounded with respect to the current measure at any time in the process.
The analysis differs depending on which item of~\cref{assump: nn approx} is assumed.  \cref{assump:nn_approx_b}  presupposes global approximation, with the error growing at most linearly and with sufficiently small slope $\varepsilon$. 
While this is a strong assumption, note that the set of functions satisfying it is non-empty for ReLU networks, since it includes quadratic potentials $V$. 
The linear growth bound on the error is desirable since it allows to uniformly bound the variance proxies of the sub-Gaussian distributions induced by the approximate LMC process (where the drift $-\nabla V$ is replaced by FCNN approximations) by the sum of the variance proxies of the starting and target distributions. 
On the other hand, in~\cref{assump:nn_approx_a} only approximation on a ball of arbitrary radius is presupposed. Somewhat mitigating the strong assumption of linear error growth, we show that~\cref{assump:nn_approx_a} and an additional constraint on the Lipschitz constant $M$ of the potential gradient implies a similar uniform bound on the variance proxies. Before we state the main result, we briefly summarize the main contributions of this work.

\subsection{Contribution}
\begin{itemize}
    \item This work derives complexity bounds for neural network architectures approximating the Langevin Monte Carlo process and some target density to arbitrary accuracy in Wasserstein-$2$ distance. In particular, we show in \cref{thm:lmc_approx_ii} and \cref{thm:dnn_lmc_lipschitz2} that under suitable assumptions on the target density, this architecture does not suffer from the curse of dimensionality, i.e. the complexity does not grow like $\varepsilon^{-d}$ with the error $\varepsilon$ and dimension $d$. We provide an informal summary of the main complexity results in the next section.
    \item In the analysis, we are able to show an interesting property of the underlying LMC process, which is that the variance proxies of the sub-Gaussian intermediate measures are uniformly bounded, both in the unperturbed case (\cref{prop:subgaussian_LMC}) as well as with perturbed drift (\cref{prop:subgaussian_nn_process}). The bound we obtain for the intermediate variance proxies is an intuitive convex combination of the bounds for the initial and the invariant measure respectively. To the best of our knowledge, these results are not yet known to the community, whereas bounds on the variance proxy of the invariant measure were shown recently in~\cite{altschuler2022concentration}.
    \item In the proof of \cref{thm:dnn_lmc_lipschitz2}, we provide a detailed theoretical construction of neural networks to approximate suitable functions globally with respect to a sub-Gaussian measure, while only assuming local approximation properties of the networks. 
    \item In \cref{sec:Numerics}, we showcase the advantage of NN approximations of Langevin Monte Carlo by estimating the ground truth in an inverse problem defined by a parametric partial differential equation. 
\end{itemize}

\subsection{Main result}

The following is an informal version of the two main complexity results of this work, namely~\cref{thm:lmc_approx_ii,thm:dnn_lmc_lipschitz2}.

\begin{theorem}[Main convergence result]\label{thm:main}
Assume that $V\colon \mathbb{R}^d\rightarrow\mathbb{R}$ is an $M$-Lipschitz, $m$-strongly convex potential as in \cref{assump:potential}. 
Let $\mu_{0}$ be sub-Gaussian with variance proxy $\sigma_0^2 > 0$ and $Y_0\sim \mu_0$. Then, for $\varepsilon > 0$, $h\in(0,\frac{2}{m+M})$ and $K \in \mathbb{N}$, there exists a ResNet-like network $\Psi$ such that the measure $\mu^{\Psi}$ of $\Psi(Y_0)$ satisfies
     \begin{equation}\label{eq:wasserstein_result}
         \mathcal W_2(\mu_{\infty},\mu^{\Psi}) \leq (1-mh)^K\mathcal W_2(\mu_{\infty}, \mu_{0}) + \frac{7\sqrt 2}{6} \frac M m \sqrt{hd} + \dfrac{1-(1-mh)^K}{m}\varepsilon.
     \end{equation}
     Furthermore, the complexity of $\Psi$ can be bounded as follows.
\begin{thmlist}
     \item \label{thm:main_a} Under \cref{assump:nn_approx_b}, there exists a constant $c(d) \in \mathcal{O}(d^{-1})$ and a ResNet-like network $\Psi$ satisfying \eqref{eq:wasserstein_result} with number of parameters bounded by $KN(d,c(d)\varepsilon,m,M)$.
     \item \label{thm:main_b} Under \cref{assump:nn_approx_a} and the additional assumption that $M < \sqrt{2}m$, there exists a ResNet-like network $\Psi$ satisfying \eqref{eq:wasserstein_result} with number of parameters in 
     \ifnum\classstyle=1{\footnotesize
     \begin{equation*}
        K\cdot \mathcal{O}\left( d\log(2d\max\{ 1,r\sqrt{M^2-m^2} \}/\varepsilon) + N(d,r,\sqrt{2}\varepsilon/\sqrt{d},m,M) + dL(d,r,\sqrt{2}\varepsilon/\sqrt{d},m,M) + 2d^2 \right),
     \end{equation*}}
     \fi
     \ifnum\classstyle=0
     \begin{equation*}
     \begin{aligned}
        K\cdot \mathcal{O}\big( &d\log(2d\max\{ 1,r\sqrt{M^2-m^2} \}/\varepsilon) +\\ 
        &N(d,r,\sqrt{2}\varepsilon/\sqrt{d},m,M) + dL(d,r,\sqrt{2}\varepsilon/\sqrt{d},m,M) + 2d^2 \big),
     \end{aligned}
     \end{equation*}
     \fi
 where $r \in \mathcal{O}(d(1+\ln(d^2\varepsilon^{-4})^{-\frac{1}{2}}))$.
\end{thmlist}
\end{theorem}

We provide some remarks and intuitions on this result. Consider first the perturbation in Wasserstein distance~\eqref{eq:wasserstein_result}. The first two terms result from the LMC process, which the \textit{ResNet-like} network $\Psi$ imitates. This process has step-size $h$ and total number of steps $K$. The first term $(1-mh)^K\mathcal W_2(\mu_{\infty}, \mu_{0})$ results from the contraction property of the continuous Langevin dynamics. Note in particular that it decreases when increasing either the step-size or the number of steps, both of which correspond to letting the Langevin process run for a longer period $Kh$. If we fix a terminal time $T$ such that $h = T/K$ and consider the continuous limit $K\rightarrow\infty$, this term recovers the exponential convergence of the continuous system, since $(1-mT/K)^K\rightarrow \exp(-mT)$. It can be clearly seen that the strong convexity constant $m$ defines the rate of the contraction in the continuous (and small step-size) limit. The second term $\frac{7\sqrt 2}{6} \frac M m \sqrt{hd}$ is a discretization error of the LMC algorithm, decreasing with the step size $h$. The third term results from a neural network approximation of the drift in every step, with error depending on the approximation accuracy $\varepsilon$.

We obtain different upper bounds for the complexity of the network depending on the assumption. 
\cref{thm:main_a} presupposes global approximation with linear error growth. Here, the network $\Psi$ is constructed by ``composing'' $K$ times the network $\phi_{c(d)\varepsilon}$ from \cref{assump:nn_approx_b}, i.e. $\psi_1=\psi_2=\ldots = \psi_K = h\phi_{c(d)\varepsilon}$ in the sense of \Cref{fig:drn}. In this case, the variance proxies of the intermediate measures (of the random variables $Y_k$ in \Cref{fig:drn}) can be uniformly bounded, and the error incurred by the network in every step can be estimated by upper bounds on second moments derived from the variance proxies. The resulting complexity is simply $K$ times the complexity of the drift approximation $\phi_{c(d)\varepsilon}$. The analysis for this case is presented in \cref{sec:linear_error_growth}. \cref{thm:main_b} presupposes only local approximations of the gradient potential. Assuming no additional structure of the potential, we cannot derive uniform bounds on the intermediate variance proxies, which would lead to a complexity growing exponentially in the number of steps $K$. An analysis with ``worst case'' upper bounds on the variance proxies can be performed using globally bounded neural networks with the upper bound growing from step to step. This analysis however is not very insightful and thus omitted.
In \cref{thm:main_b}, we require instead additional structure of our target in order to receive a uniform bound on the variance proxies. Namely, $V$ must not be too far away from a quadratic function, which is encoded in the additional condition that $M < \sqrt{2}m$. An alternative way to view this condition is that $\nabla V$ can be approximated globally with a linear function, incurring a sufficiently small error. Under this additional assumption, we can again show the existence of a network $\phi$ such that ResNet-like network $\Psi$ given by blocks $\psi_1=\psi_2=\ldots =\psi_K = h\phi$ (in the sense of \Cref{fig:drn}) satisfies \eqref{eq:wasserstein_result}. The construction of the network $\phi$ is quite technical, involving cut-offs and multiplication with approximate indicator functions of the networks provided by \cref{assump:nn_approx_a}. Hence, the complexity is higher in this case than under the assumption of global linear error growth. Note however that the complexity again only grows linearly in the number of steps $K$. This analysis is performed in \cref{sec:combination}.

\subsection{Structure of the paper}
We begin with the definition of the Langevin process, Wasserstein spaces and feed-forward neural networks and introduce our notation in~\cref{sec:preliminaries}. We continue by introducing the \emph{ResNet-like} architecture in~\cref{sec:ResNet-like}. A convergence result for the perturbed Langevin process with approximate drifts in every step is derived in~\cref{sec:perturbed Langevin}. Our main results on the expressivity of the ResNet-like architecture to sample from the given distribution can be found in~\cref{sec:linear_error_growth,sec:combination}. 
Herein,~\cref{sec:linear_error_growth} is based on~\cref{assump:nn_approx_b} and~\cref{sec:combination} is based on~\cref{assump:nn_approx_a}. The paper is structured such that the proofs of all results are found in the appendix. 
In~\cref{sec:Numerics}, numerical experiments for a Gaussian, a Gaussian mixture and a posterior distributions can be found. We summarize and discuss our results in~\cref{sec:Conclusion}. 

\section{Definitions and notation}
\label{sec:preliminaries}

We briefly recall some definitions and results that are used throughout the paper. An analysis of the Langevin Monte Carlo approximation is carried out involving the notion of Wasserstein spaces. Moreover, a multi-dimensional sub-Gaussian random vector is defined and a notion of Lyapunov functions is introduced as well as a formal notation for neural networks.

\subsection{Langevin Monte Carlo and Wasserstein space}

Throughout this manuscript, let $d\in\mathbb{N}$ be the dimension of the space, on which the target distribution lives. Let $(\Omega,\mathcal{F},\mathbb{P})$ be a probability space and $W:[0,\infty)\times \Omega\rightarrow \mathbb{R}^d$ be standard Brownian motion. Let $V\in\mathcal{C}^1(\mathbb{R}^d,\mathbb{R})$ satisfy \cref{assump:potential} with Lipschitz constant $M\in (0,\infty)$ and convexity constant $m\in (0,\infty)$. Let $h \in (0,\infty)$. Let $\chi_{h}: [0,\infty)\rightarrow [0,\infty)$, $\chi_{h}(s) = \max\{ kh : k \in\mathbb{N}_0, kh \leq s \}$ and $\xi_k = \sqrt 2 (W_{kh} - W_{(k-1)h}) \sim \mathcal N_d(0, 2h I_d)$ for $k \in \mathbb{N}$. Let $\mu_0$ be a distribution on $\mathbb{R}^d$ with finite moments $\mathbb{E}_{x\sim\mu_0}[|x|^p]^{\frac{1}{p}} < \infty$ for all $p\in\mathbb{N}$ and $X_0,\tilde{X}_0\colon \Omega\rightarrow \mathbb{R}^d$ be random variables independent of the Brownian motion $W$ such that $X_0, \tilde{X}_0\sim \mu_0$. Let $X, \tilde{X}\colon [0,\infty)\times\Omega\rightarrow\mathbb{R}^d$ be stochastic processes satisfying

\begin{subequations}
\begin{align}\label{eq:X_t}
    X_t &= X_0 - \int_0^t \nabla V(X_s) \mathrm{d}s + \sqrt{2}W_t,
    \\
    \tilde{X}_t &= \tilde{X}_0 - \int_0^t \nabla V(\tilde{X}_{\chi_{h}(s)})\mathrm{d}s + \sqrt{2}W_t,
    \label{eq:tildeX_t}
\end{align}
\end{subequations}

for all $t\in[0,\infty)$ and denote the distributions of $X_t$ and $\tilde X_t$ by $\mu_t^X$, and $\mu^{\tilde X}_t$, respectively. We call $\tilde{X}$ the \emph{Langevin-Monte-Carlo} process (LMC), which is the time discretization of the \emph{Langevin process} $X$. Let $\mu_{\infty}$ be the probability distribution on $\mathbb{R}^d$ defined by $\mathrm{d}\mu_{\infty}(x) = \frac{1}{Z}e^{-V(x)}\mathrm{d}x$, where $Z\in (0,\infty)$ is a normalization constant such that $\int_{\mathbb{R}^d}\mathrm{d}\mu_{\infty}(x) = 1$.

Furthermore, the notion of balls and spheres is needed.

\begin{remark}[Balls and spheres]
    For some normed space $(U, \|\cdot\|_U)$ we denote the  the \emph{closed ball} of radius $r > 0$ centered at a point $x \in U$ by $B_r(x)$ and the \emph{sphere} of radius $r > 0$ centered at a point $x \in U$ by $\mathbb S_r(x)$.
    We henceforth use $U = \R^d$ and $\|\cdot\|_U = \|\cdot\|_{\ell^p}$.
    In this case, we write $B_r^p(x)$ and $\mathbb S_r^p(x)$ to denote the norm dependence. In the special case of $p=2$ we often simply write $B_r^2(x) = B_r(x)$ and $\mathbb S_r^2(x) = \mathbb S_r(x)$.
\end{remark}

We henceforth use the \textit{Kantorovich–Rubinstein metric} (or \textit{Wasserstein-$p$ distance}) of measures for $p=2$.

\begin{definition}[Wasserstein space]\label{def:wasserstein_space}
    For $p \geq 1$, denote by $\mathcal D_p(\R^d)$ the set of probability measures on $\R^d$ endowed with the Wasserstein-$p$ distance
    \begin{equation*}
        \mathcal W_p(\mu, \nu)^p := \min_{\gamma \in \Pi(\mu, \nu)}\int_{\R^d}\|x-y\|_{\ell^p}^p \dd\gamma(x,y),
    \end{equation*}
    where $\Pi(\mu,\nu) := \left\{\gamma \in \mathcal P(\R^d \times \R^d) : (\pi_0)_\sharp \gamma = \mu, (\pi_1)_\sharp \gamma = \nu \right\}$ is the set of \textit{transport plans} with marginals $\mu$ and $\nu$.
\end{definition}

The next theorem recalls Wasserstein-$2$ convergence results for the LMC scheme defined by \Cref{eq:tildeX_t}.

\begin{theorem}[{Guarantees for the constant-step LMC~\cite[Theorem~1]{Dalalyan2017UserfriendlyGF}}]\label{thm:LMC_approx}
    Assume that $h \in (0,\frac 2 M)$ and that $V$ satisfies \cref{assump:potential}. For any $K\in\mathbb{N}_0$, the following claims hold: 
    \begin{itemize}
        \item If $h \leq \frac{2}{m+M}$ then $\mathcal W_2(\mu_{\infty},\mu^{\tilde X}_{Kh}) \leq (1-mh)^K\mathcal W_2(\mu_{\infty}, \mu_{0}) + \frac{7\sqrt 2}{6} \frac M m \sqrt{hd}$.
        \item If $h \geq \frac{2}{m+M}$ then $\mathcal W_2(\mu_{\infty},\mu^{\tilde X}_{Kh}) \leq (Mh-1)^K\mathcal W_2(\mu_{\infty}, \mu_{0}) + \frac{7\sqrt 2}{6} \frac{Mh}{2-Mh} \sqrt{hd}$.
    \end{itemize}
    \end{theorem}

\subsection{Sub-Gaussianity}

In the later analysis, we make use of sub-Gaussian random variables, the distributions of which exhibit a strong tail decay dominated by the tails of a Gaussian.

\begin{definition}[Sub-Gaussian random variable]\label{def:subgaussian_rv}
    Let $Z$ be a random variable on $\R$. $Z$ is said to be \emph{sub-Gaussian} with variance proxy $\sigma^2 > 0$ if it satisfies one of these equivalent conditions:
    \begin{thmlist}
        \item\label{def:subGauss_1} For any $s \in \R$,\ $\E[\exp(sZ)] \leq \exp\left(\frac{\sigma^2 s^2}{2}\right)$.
        \item\label{def:subGauss_2} For any $r > 0$,\ $\mathbb P(Z \geq r) \leq \exp\left(-\frac{r^2}{2\sigma^2}\right)$ and $\mathbb P(Z \leq -r) \leq \exp\left(-\frac{r^2}{2\sigma^2}\right)$.
        \item\label{def:subGauss_3} For any $q \in \N$,\ $\E[|Z|^{q}] \leq (\sqrt 2 \sigma)^q q \Gamma(q/2)$.
    \end{thmlist}
    We then write $Z \sim \operatorname{subG}(\sigma^2)$.
\end{definition}

\begin{definition}[Sub-Gaussian random vector]\label{def:subgaussian_rvector}
    A random vector $Z \in \R^d$ is said to be \emph{sub-Gaussian} with variance proxy $\sigma^2 > 0$ if for any $u \in \mathbb S_1(0)$ the real random variable $\langle u, Z\rangle$ is sub-Gaussian with variance proxy $\sigma^2$. We then write $Z \sim \operatorname{subG}(\sigma^2)$.
\end{definition}

\begin{proposition}[$\ell^p$-norm of a sub-Gaussian random vector is sub-Gaussian]\label{prop:lp_subgaussian}
    Let $Z\in\mathbb{R}^d$ 
    be a sub-Gaussian random vector with variance proxy $\sigma^2 > 0$. Then, $\|Z\|_{\ell^p}$ is a sub-Gaussian random variable for any $p \geq 1$ with variance proxy bounded by $d^2\sigma^2$.
\end{proposition}

\subsection{Lyapunov functions}\label{sec:lyapunov}

There is an interesting connection between sub-Gaussianity and Lyapunov functions, which we exploit in our analysis. This connection presented in \cref{rem:lyapunov_subgauss} was for instance used in \cite{altschuler2022concentration} to bound the variance proxy of the invariant measure of the LMC algorithm in the smooth log-concave setting.

\begin{definition}[Lyapunov function {\cite[Definition 3.1]{altschuler2022concentration}}]\label{def:lyapunov_function}
    For any weight $\lambda > 0$ the Lyapunov function $\mathcal{L}_{\lambda}\colon \mathbb{R}^d \rightarrow \mathbb{R}$ is defined by
    \begin{equation*}
        \mathcal{L}_{\lambda}(x) = \mathbb{E}_{v\sim \mathbb{S}_1(0)}\left[ 
 e^{\lambda \langle v,x \rangle} \right],
    \end{equation*}
    where $v\sim \mathbb{S}_1(0)$ denotes uniform sampling from the $\ell^2$-unit sphere in $\mathbb{R}^d$. Furthermore, let $\ell\colon \mathbb{R}_{\geq 0}\rightarrow \mathbb{R}_{\geq 0}$ be defined by
    \begin{equation*}
        \ell(z) = \mathbb{E}_{v\sim \mathbb{S}_1(0)}\left[ 
 e^{z \langle v,e_1 \rangle} \right]
    \end{equation*}
    such that it holds $\mathcal{L}_{\lambda}(x) = \ell(\lambda \| x \|)$ due to the rotational invariance of the Lypunov function.
\end{definition}

The following lemma establishes the fact that sub-Gaussianity of a random variable is equivalent to an appropriate exponential upper bound on the expectation of the Lyapunov function.
\begin{lemma}[Connection between sub-Gaussianity and Lyapunov functions]
\label{rem:lyapunov_subgauss}
$X$ is a sub-Gaussian random vector with variance proxy $\sigma^2$ if and only if it holds for all $\lambda > 0$ that
\begin{equation*}
\mathbb{E}_{X}\mathcal{L}_{\lambda}(X) \leq e^{\frac{\sigma^2 \lambda^2}{2}}. 
\end{equation*}
\end{lemma}

\subsection{Neural networks}\label{sec:NN}

We recall the formal mathematical notation of neural networks as introduced in \cite{Petersen2018}. The neural network is defined as a set of weights and biases, and the corresponding function is defined as the realization.

\begin{definition}[Neural network architectures]
Let
\begin{equation}
    \mathcal{N} = \bigcup_{L=2}^\infty \bigcup_{(W_0,W_1,\ldots,W_{L+1})\in \mathbb{N}^{L+1}} \left(\bigtimes^L
    _{\ell=0}(\mathbb{R}^{W_{\ell+1}\times W_\ell} \times \mathbb{R}^{W_{\ell+1}} )\right)
\end{equation}
be the set of all \emph{fully connected neural networks} (FCNNs).
We call $\sigma\in\mathcal{C}(\mathbb{R},\mathbb{R})$ the \emph{activation function} and for every $n\in\mathbb{N}$ let $\sigma_n\in\mathcal{C}(\R^n,\R^n)$ be the function satisfying $\mathbf \sigma_n(x) = (\sigma(x_1),\ldots,\sigma(x_n))^{\intercal}$ for every $x\in\R^n$. Let $\mathcal{P} \colon \mathcal{N} \rightarrow \mathbb{N}$, $\mathbb{L}:\mathcal{N} \to \mathbb{N}$, $\mathcal{R} \colon \mathcal{N} \rightarrow \mathcal{C}(\mathbb{R}^{W_0},\mathbb{R}^{W_{L+1}})$ be the \emph{number of nonzero parameters}, \emph{number of layers} and the \emph{realization}, respectively, which satisfy for all $L,W_0,\ldots,W_{L+1} \in \mathbb{N}$, $\phi = ((A_0, b_0), \ldots , (A_L, b_L)) \in \times^L
_{\ell=0}(\mathbb{R}^{W_{\ell+1}\times W_\ell} \times \mathbb{R}^{W_{\ell+1}})$ and for all $x\in\mathbb{R}^{W_0}$
\begin{subequations}
\begin{align}
    x_0 &= x, \\
    x_{\ell+1} &= \sigma_{W_{\ell+1}}(A_{\ell}x_{\ell}+b_{\ell}), \qquad \ell = 0,\ldots,L-1,\\
    \mathcal{R}\phi(x) &= A_L x_L + b_L.
\end{align}
\end{subequations}
If not specified otherwise, we let $\sigma\colon \mathbb{R}\rightarrow \mathbb{R}, x \mapsto \max\{0,x\}$ be the ReLU activation function. 
\end{definition}

For a simpler notation, we also make use of the set of FCNNs of a fixed number of layers and fixed maximal numbers of parameters per layer.

\begin{definition}[Fully connected networks with fixed width and depth]\label{assumption:NN_structure}
    Let $n_0,n_1,W,L\in\mathbb{N}$ and
    \begin{equation}
    \begin{aligned}
        \mathcal{N}_{n_0,n_1}(W,L) \coloneqq \big\{ &((A_0, b_0), \ldots , (A_L, b_L)) \in (\bigtimes^{L}_{\ell=0}(\mathbb{R}^{W_{\ell+1}\times W_{\ell}} \times \mathbb{R}^{W_{\ell+1}} )) \\
        &\colon W_0=n_0, W_{L+1}=n_1, W_{\ell}\leq W ~~\forall~~ \ell\in\{1,\ldots,L\} \big\}
        \end{aligned}
    \end{equation}
    be the set of \emph{fully connected neural networks with fixed width and depth} with $n_0$ inputs, $n_1$ outputs, and a maximum of $W$ neurons in each layer $\ell \in \{1,\dots, L\}$.
\end{definition}

\section{ResNet-like architectures}\label{sec:ResNet-like}

We define neural network realizations that resemble residual neural networks as introduced in~\cite{he2016deep} including multiple skip connections from the input to intermediate results, see \Cref{fig:drn} for an illustration. This architecture allows to efficiently approximate LMC and performs well in our numerical experiments.
    
\begin{definition}[ResNet-like realization]\label{def:weird_realization}
Let $K,n\in \mathbb{N}$  and let $ \Phi \coloneqq \{ \phi_i \}_{i=1}^{K} \subset \mathcal{N}$ with $W_0=W_{L+1}=n$ fixed for every network. 
A \emph{ResNet-like realization} $\tilde{\mathcal{R}}\Phi$
$ \in \mathcal{C}(\mathbb{R}^{n}\times \bigtimes_{i=1}^K\R^{n},\mathbb{R}^{n})$ is defined for $x\in\R^n$ and $y=(y_1,\ldots,y_K)\in\bigtimes_{i=1}^K \R^n$ by
\begin{subequations}
\begin{align}
    x_0 &\coloneqq x, \\
    x_i &\coloneqq x_{i-1} + \mathcal{R}\phi_i(x_{i-1}) + y_i\quad \text{for } i=1,\dots, K,\label{eq:set_realization_weird}\\
    \widetilde{\mathcal{R}}\Phi
    (x, y) &\coloneqq x_K.
\end{align}
\end{subequations}
\end{definition}

The ResNet-like architecture is defined in a way such that it emulates a perturbed version of a Langevin Monte Carlo process, where the drift in the $k$-th step is replaced by the realization of a FCNN $\mathcal{R}\phi_k$. The following definition makes this perturbed process concrete.

\begin{definition}[Stochastic process driven by $\Phi$]\label{def:NN_driven_process}
Let $K \in \mathbb{N}$, $\phi_1,\ldots,\phi_K \colon \mathbb{R}^d\rightarrow\mathbb{R}^d$ be Lipschitz-continuous and $\Phi = \{\phi_i\}_{i=1}^K$. 
Let $Y^{\Phi}\colon [0,Kh]\times\Omega \rightarrow\R^d$ be the stochastic process defined by
\begin{equation}\label{eq:driven_process}
    Y^{\Phi}_t = Y_0 + \int_{0}^t \phi_{1+\frac{1}{h}\chi_h(s)}\left(Y^{\Phi}_{\chi_h(s)}\right) \mathrm{d}s + \sqrt{2}W_t.
\end{equation}
In this case $Y^{\Phi}$ is called a \emph{stochastic process driven by $\Phi$}. We denote the law of the process $Y^{\Phi}$ by $\mu^{\Phi}_{t}$. When considering a set of neural networks $\{ \phi_1,\ldots,\phi_K \}$, we denote the stochastic process driven by the realizations $\{\mathcal{R}\phi_k,\ldots,\mathcal{R}\phi_K\}$ by $Y^{\Phi}$, suppressing the realizations for the sake of brevity.
\end{definition}

To prepare for approximations of LMC with neural networks, we first show that the previously defined ResNet-like realizations in \cref{def:weird_realization} are the appropriate architecture to represent stochastic processes driven by neural networks in the sense of \cref{def:NN_driven_process}.

\begin{proposition}\label{prop:resnet_equals_driven_process}
Let $K\in\mathbb{N}$ 
and $\Phi = \{\phi_i\}_{i=1}^K \subset \mathcal{N}$.

Let $\xi =(\xi_1,\ldots,\xi_{K})$, 
where $\xi_i$ are the Brownian increments as defined in \cref{sec:preliminaries}. 
Let $Y_0 \sim\mu_0$ and let $Y^{\Phi}$ 
be a stochastic process driven by $\Phi$ in the sense of \cref{def:NN_driven_process}. 
Then, there exists $K$ neural networks $\psi_1,\ldots,\psi_K$ such that $\psi_i$ has the same width and number of layers as $\phi_i$ 
for $i=1,\ldots,K$, and $\Psi = \{ \psi \}_{i=1}^K$ such that
\begin{equation}
   Y^{\Phi}_{Kh} = \tilde{\mathcal{R}}\Psi (Y_0,\xi).
\end{equation}
\end{proposition}

The proof can be found in~\cref{sec:Proofof-prop:resnet_equals_driven_process}.
\section{Perturbed Langevin Monte Carlo}\label{sec:perturbed Langevin}

In this section, we derive a perturbation result in Wasserstein-$2$ distance for the standard Langevin process, when the drift is replaced by approximations with small global $L^2$-error w.r.t. the current measure. This is more or less a direct consequence of \cref{thm:LMC_approx} and the proof follows similar arguments. Upper bounds for the convergence of the unperturbed Langevin process $\tilde{X}$ can be found in~\cite{dalalyan2017stronger}.

\begin{theorem}[Perturbed LMC]\label{thm:nn_approx_lmc}
    Assume $\varepsilon > 0$, $h \in (0,\frac 2 M)$ and $V$ satisfies \cref{assump:potential}. Let $Y_0\sim\mu_0$, $K\in\N$ and $\Phi = \{ \phi_i: \mathbb{R}^d \to \mathbb{R}^d\}_{i=1}^K$.
    Let $Y^{\Phi}$
    be the stochastic process driven by $\Phi$ 
    and assume that for $i=0,\ldots,K-1$,
\begin{equation}\label{eq: network approx drift}
    \| -\nabla V- \phi_{i+1} \|_{L^2_{\mu_{ih}^\Phi}(\R^d;\R^d)} < \varepsilon
\end{equation}
is satisfied. Then the law $\mu_t^{\Phi}$ of the process $Y^{\Phi}$ satisfies
    \begin{itemize}
        \item If $h \leq \frac{2}{m+M}$ then 
        \ifnum\classstyle=1
        $\mathcal W_2(\mu_{\infty},\mu_{Kh}^{\Phi}) \leq (1-mh)^K\mathcal W_2(\mu_{\infty}, \mu_{0}) + \frac{7\sqrt 2}{6} \frac M m \sqrt{hd} + \dfrac{1-(1-mh)^K}{m}\varepsilon$.
        \fi
        \ifnum\classstyle=0
        {\footnotesize
        $\mathcal W_2(\mu_{\infty},\mu_{Kh}^{\Phi}) \leq (1-mh)^K\mathcal W_2(\mu_{\infty}, \mu_{0}) + \frac{7\sqrt 2}{6} \frac M m \sqrt{hd} + \dfrac{1-(1-mh)^K}{m}\varepsilon$.
        }
        \fi
        \item If $h \geq \frac{2}{m+M}$ then 
        \ifnum\classstyle=1
        $W_2(\mu_{\infty},\mu^{\Phi}_{Kh}) \leq (Mh-1)^K\mathcal W_2(\mu_{\infty}, \mu_{0}) + \frac{7\sqrt 2}{6} \frac{Mh}{2-Mh} \sqrt{hd} + \dfrac{1-(Mh-1)^{K}}{2-Mh}h\varepsilon$.
        \fi
        \ifnum\classstyle=0
        {\footnotesize
        $W_2(\mu_{\infty},\mu^{\Phi}_{Kh}) \leq (Mh-1)^K\mathcal W_2(\mu_{\infty}, \mu_{0}) + \frac{7\sqrt 2}{6} \frac{Mh}{2-Mh} \sqrt{hd} + \dfrac{1-(Mh-1)^{K}}{2-Mh}h\varepsilon$.
        }
        \fi
    \end{itemize}
\end{theorem}

The proof can be found in~\cref{sec:proofPerturbedLMC}.
\section{Approximation of Langevin Monte Carlo under linear error \ifnum\classstyle=0 \\ \fi growth constraints}\label{sec:linear_error_growth}

The analysis in this section is inspired by \cite{altschuler2022concentration}, where bounds on the sub-Gaussian variance proxy of the invariant measure of the LMC \eqref{eq:tildeX_t} are derived. We use similar arguments based on the connection between sub-Gaussianity and Lyapunov functions established in \cref{rem:lyapunov_subgauss}.

\subsection{Sub-Gaussianity of perturbed Langevin Monte Carlo}
First, we treat the standard LMC process \eqref{eq:tildeX_t}, without any approximation of the drift and without requiring any additional assumption. Instead of only bounding the variance proxy of the invariant measure, we derive a bound for all intermediate measures $\mu^{\tilde{X}}_{kh}$. In the limit of steps $k\rightarrow\infty$, our result coincides with the one obtained in \cite{altschuler2022concentration} for the invariant measure.
\begin{proposition}[Sub-Gaussianity of LMC]\label{prop:subgaussian_LMC}
    Let $h\in(0,\frac{2}{M})$ and $\tilde{X}_0 \sim \mu_0$ be sub-Gaussian with variance proxy $\sigma_0^2$. Then, for $k\in\mathbb{N}$, $\tilde{X}_{kh}$ is sub-Gaussian with variance proxy
    \begin{equation}
        \sigma_k^2 = 2h\dfrac{1-c^k}{1-c} + \sigma_0^2 c^k,
    \end{equation}
    where $c = \max_{\rho\in \{ 
m,M \}} |1-\rho h|$.
\end{proposition}

Note that in the limit $k\rightarrow\infty$, this estimate leads exactly to the bound in \cite{altschuler2022concentration} for the invariant measure, i.e.,
\begin{align}
    \sigma_k^2 \longrightarrow \frac{2h}{1-c},
\end{align}
since $c < 1$.

\begin{remark}\label{remark:subgaussian_LMC}
    Recall that the process $\tilde X$ for all $t \in [0,Kh]\setminus \N h$ can be written as
    \begin{equation*}
        \tilde X_t = \tilde X_{\chi_h(t)} - h\nabla V(\tilde X_{\chi_h(t)}) + \sqrt 2 (W_t - W_{\chi_h(t)}).
    \end{equation*}
    Therefore, since $\tilde X_{kh}$ is sub-Gaussian for all $k \in \N$ by \cref{prop:subgaussian_LMC}, by linear interpolation $\tilde X_t$ is sub-Gaussian for all $t \in [0, Kh]$. Indeed, combining the proof of \cref{prop:subgaussian_LMC} with  \cref{lem:gauss_lyap} with $\sigma^2 = 2(t-\chi_h(t))$ leads to the variance proxy
    \begin{equation*}
        \sigma_t^2 = 2(t - \chi_h(t)) + 2h\frac{1-c^{\chi_h(t)/h}}{1-c} + \sigma_0^2 c^{\chi_h(t)/h}
    \end{equation*}
    for $\tilde{X}_t$.
\end{remark}

A main ingredient of the proof of \cref{prop:subgaussian_LMC} is the fact that the contractivity constant $c$ can be factored out of the expectation of the Lyapunov function via Jensen's inequality. The basic idea of our analysis for the DNN driven LMC is to ensure that similar arguments can be applied for the perturbed process. The following theorem provides bounds on the variance proxies of the perturbed process under the condition that the global error of the neural network approximations grows at most linearly. The proof follows similar arguments as the one of \cref{prop:subgaussian_LMC}. In particular, we again make frequent use of Jensen's inequality for concave functions. The additional assumption on the networks can be seen as a way to ensure that Jensen's inequality can be applied.

\begin{proposition}[Sub-Gaussianity of DNN driven LMC]\label{prop:subgaussian_nn_process}
    For $k\in\mathbb{N}$, let $\phi_k\colon \mathbb{R}^d\rightarrow\mathbb{R}^d$ and $\Phi = \{\phi_k\}_{k\in\mathbb{N}}$. Assume that there exist $\delta > 0$, $G < m$ such that
\begin{equation}\label{eq:jentzen_assump}
        \| -\nabla V (x) - \phi_k(x) \|_{\ell^2} \leq \delta + G\|x\|_{\ell^2}, \qquad \forall x\in \mathbb{R}, k\in\mathbb{N}.
    \end{equation}
    Let $h\in (0,\frac{2}{m+M})$. Then, the stochastic process $Y^{\Phi}$ driven by $\Phi$ and given by \Cref{eq:driven_process} with $Y^{\Phi}_0 \sim \mu_0$ is sub-Gaussian for all $t$.
    In particular, at time $kh$, $k\in\mathbb{N}$, $Y^{\Phi}_{kh}$ is sub-Gaussian with variance proxy
    \begin{equation}
        \sigma_k^2 = \left(\frac{2h}{1-(c + hG)}+h^2\delta^2\right)[1-(c + hG)^k]+\sigma_0^2[c+hG]^k \leq \frac{2h}{1-(c + hG)}+h^2\delta^2 +\sigma_0^2,
    \end{equation}
    where $c = \max_{\rho\in \{ 
m,M \}} |1-\rho h|$. 
\end{proposition}

The proofs of the results in this subsection can be found in~\cref{sec:subGaussianLMC}.

\subsection{Neural network driven LMC with approximate drift with global linear error growth}

We can now prove a first theorem on  approximations of the LMC process using ResNet-like ReLU networks. The goal is to approximate the gradient $-\nabla V$ in each step of the LMC process with an FCNN on the whole domain and apply \cref{thm:nn_approx_lmc}. We use \cref{assump:nn_approx_b} on the existence of neural networks approximating the drift with a controlled linear growth. When considering the stochastic processes driven by these neural networks, the linear growth allows us to bound the variance proxy in every step according to \cref{prop:subgaussian_nn_process}. Due to the bounded variance proxies and the arbitrarily small error growth, the sub-Gaussianity of the intermediate measures ensures the desired global errors \eqref{eq: network approx drift} of the drift approximations.

\begin{theorem}[ResNet-like realization approximated LMC (I)]\label{thm:lmc_approx_ii}
We presuppose the conditions in \cref{assump:potential} for the potential $V$ and \cref{assump:nn_approx_b} for the existence of FCNN approximations of $-\nabla V$ with parameters bounded by $N(d,\varepsilon,m, M)$. Let $h\in(0,\frac{2}{m+M})$ and $\mu_{0}$ be sub-Gaussian with variance proxy $\sigma_0^2 > 0$. Then, there exists for any $K\in\mathbb{N}$ and any $\varepsilon >0$ a neural network $\psi$ with number of parameters bounded by $N(d,c(d)\varepsilon,m,M)$, where $c(d)\in \mathcal{O}(d^{-1})$ such that the measure $\mu^{\Psi}$ of the ResNet-like realisation of $\Psi \coloneqq \{ \psi_k\coloneqq \psi \}_{k=1}^K$ with input $(Y_0,\xi)$, i.e. the law of $\tilde{\mathcal{R}}\Psi(Y_0,\xi)$, satisfies
\begin{equation}
    \mathcal W_2(\mu_{\infty},\mu^{\Psi}) \leq (1-mh)^K\mathcal W_2(\mu_{\infty}, \mu_{0}) + \frac{7\sqrt 2}{6} \frac M m \sqrt{hd} + \dfrac{1-(1-mh)^K}{m}\varepsilon.
\end{equation}
\end{theorem}

The proof of this result can be found in~\cref{sec:proofsGlobalApprox}.

\section{Approximation of Langevin Monte Carlo under local error- and Lipschitz constraints}\label{sec:combination}

Under strong Lipschitz constraints on the potential gradient $\nabla V$ it is possible to derive uniform bounds on the sub-Gaussian variance proxies even when assuming only local approximation of the drift. In this section, we construct networks that fulfill both \cref{assump:nn_approx_a} and \eqref{eq:jentzen_assump}, under the additional constraint that $M < \sqrt{2}m$, meaning that $V$ can not be ``too far away'' from a quadratic function. In this section, we assume without loss of generality that the unique minimizer of $V$ is given by $x^\ast = 0$ to simplify notation. First, we show that the drift can be approximated by a linear function with the required linear error growth.

\begin{proposition}[Linear approximations of potential gradient with bounded error]\label{thm:linear_approx}
    Assume \cref{assump:potential} with $m \leq M < \sqrt{2}m$ and assume that the unique minimizer of $V$ is given by $x^\ast = 0$. Then for all $x\in\mathbb{R}^d$ 
    \begin{equation}
        \| -\nabla V(x) - mx \|_{\ell^2} < \sqrt{M^2-m^2}\|x\|_{\ell^2} < m\|x\|_{\ell^2}.
    \end{equation}
\end{proposition}

Second, we need to be able to truncate a neural network, such that the output is uniformly bounded from above and below but the approximation quality on a specified domain remains the same. The resulting function is still a neural network with just 2 more ReLU layers, as the following lemma shows. The construction of this network is visualized in the one-dimensional case in \Cref{fig:cutoff}.

\begin{figure}
    \centering
    \begin{minipage}{.4\textwidth}
\resizebox{5cm}{!}{
\begin{tikzpicture}
    \begin{axis}[
        xmin=-2,xmax=2,
        ymin=-2,ymax=2,
        axis x line=middle,
        axis y line=middle,
        axis line style=-,
        xlabel={$x$},
        ylabel={$f(x)$},
        ]
        \addplot[no marks,blue,-] expression[domain=-pi:pi,samples=500]{-max(-x+1,0)+1} ;
    \end{axis}
\end{tikzpicture}}
\end{minipage}
\begin{minipage}{.4\textwidth}
\resizebox{5cm}{!}{
\begin{tikzpicture}
    \begin{axis}[
        xmin=-2,xmax=2,
        ymin=-2,ymax=2,
        axis x line=middle,
        axis y line=middle,
        axis line style=-,
        xlabel={$x$},
        ylabel={$g(x)$},
        ]
        \addplot[no marks,blue,-] expression[domain=-pi:pi,samples=500]{max(x+1,0)-1} ;
    \end{axis}
\end{tikzpicture}}
\end{minipage}
\begin{minipage}{.4\textwidth} 
\resizebox{5cm}{!}{
\begin{tikzpicture}
    \begin{axis}[
        xmin=-2,xmax=2,
        ymin=-2,ymax=2,
        axis x line=middle,
        axis y line=middle,
        axis line style=-,
        xlabel={$x$},
        ylabel={$g(f(x))$},
        ]
        \addplot[no marks,blue,-] expression[domain=-pi:pi,samples=500]{max(-max(-x+1,0)+1+1,0)-1} ;
    \end{axis}
\end{tikzpicture}}
\end{minipage}
    \caption{\textit{Visualization of the bounding of the neural network used in \cref{thm: bounded network} in one dimension. Here, $f(x)=-\sigma(-x+1)+1$ defines a bound from above by $1$ and $g(x) = \sigma( x+1) -1$ defines a bound from below by 1. Note that $g\circ f$ is the identity on $[-1,1]$ and bounded by $\pm 1$ on $\mathbb{R}$. The bounding of the network in \cref{thm: bounded network} corresponds to an application of similar functions in every dimension.}}
    \label{fig:cutoff}
\end{figure}

\begin{proposition}\label{thm: bounded network}
    Let $p\in [1,\infty]$, $\mu$ be a measure on $\mathbb{R}^d$, $\Omega \subset \mathbb{R}^d$ and $\nabla V$ bounded on $\Omega$.
    Let $\phi_{L-2}$ be a neural network. 
    There exists a neural network $\phi_L$ entrywise bounded by $c\in\mathbb{R}^d$ with $c_i\coloneqq \|\nabla 
    V_i\|_{L^\infty(\Omega;\mathbb{R}^d)}$ with two more layers than $\phi_{L-2}$ and $2d^2 + 2$ additional weights such that
    \ifnum\classstyle=1
    \begin{align*}
        \|- \nabla V - \mathcal{R}\phi_L\|^p_{L^{p}_{\mu}(\Omega; \mathbb{R}^d)} &\leq\|-\nabla V - \mathcal{R}\phi_{L-2}\|^p_{L^p_\mu(\Omega, \mathbb{R}^d)} \quad \text{and}\\ 
        \|- \nabla V - \mathcal{R}\phi_L\|^p_{L^{p}_{\mu}(\mathbb{R}^d; \mathbb{R}^d)}
        &\leq \|-\nabla V - \mathcal{R}\phi_{L-2}\|^p_{L^p_\mu(\Omega, \mathbb{R}^d)} 
        + 2^{p-1}( \|\nabla V\|_{L^p_\mu(\R^d \setminus \Omega; \R^d)}^p + \|c\|_{L^p_\mu(\R^d \setminus \Omega; \R^d)}^p).
    \end{align*}
    \fi
    \ifnum\classstyle=0
    {\footnotesize
    \begin{align*}
        \|- \nabla V - \mathcal{R}\phi_L\|^p_{L^{p}_{\mu}(\Omega; \mathbb{R}^d)} &\leq\|-\nabla V - \mathcal{R}\phi_{L-2}\|^p_{L^p_\mu(\Omega, \mathbb{R}^d)} \quad \text{{\normalsize and}}\\ 
        \|- \nabla V - \mathcal{R}\phi_L\|^p_{L^{p}_{\mu}(\mathbb{R}^d; \mathbb{R}^d)}
        &\leq \|-\nabla V - \mathcal{R}\phi_{L-2}\|^p_{L^p_\mu(\Omega, \mathbb{R}^d)} 
        + 2^{p-1}( \|\nabla V\|_{L^p_\mu(\R^d \setminus \Omega; \R^d)}^p + \|c\|_{L^p_\mu(\R^d \setminus \Omega; \R^d)}^p).
    \end{align*}
    }
    \fi
\end{proposition}

With these two results in place, we can prove our final theorem on  approximations of the LMC process using ResNet-like ReLU networks. The goal is to approximate the gradient $-\nabla V$ in each step of the LMC process with an FCNN on the whole domain and apply \cref{thm:nn_approx_lmc}. First, we use \cref{thm:linear_approx} and assumption \cref{assump:nn_approx_a} to construct appropriate neural networks approximating the drift arbitrarily well on a ball and with a controlled linear growth of the error outside the ball. The key ideas are the following. First, by the bound on the Lipschitz constant $M$, the potential gradient $-\nabla V$ can be globally approximated by the linear function $x \mapsto mx$ with pointwise error growing at most linearly and with slope $m$. According to our assumptions and results on the addition of neural networks, $-\nabla V - m\cdot$ can be approximated to arbitrary accuracy on any ball. With two additional ReLU layers, the resulting network can be ``cut off'', according to \cref{thm: bounded network}, so that it stays bounded even outside the considered ball. The output of the network is thus confined to a hypercube, say $[-c,c]^d$ for some $c>0$. The next step is to construct a network with the same output on the ball but with zero output outside of it. This can be done by approximating the multiplication of the cut-off network with an (approximate) indicator function constructed with ReLUs. Note that the cut-off is required for this approximation, because it ensures that all the inputs of the multiplication $(x,y)\mapsto xy$, i.e. the indicator function and the output of the cut-off network, live on compact sets. 
Adding to the resulting network the neural network representation of the function $x \mapsto mx$ provides us with networks which satisfies \eqref{eq:jentzen_assump} in addition to \cref{assump:nn_approx_a}. In particular, the error is controlled both on the ball ($\varepsilon/2$-accuracy) and outside of it (linear error growth). The ingredients for the construction of this neural network are visualized in \Cref{fig:Network construction} for the one-dimensional case. The existence of such a network and its formal derivation are the subjects of the following proposition. 

\begin{figure}[ht]
\begin{center}
\resizebox{10cm}{6cm}{
\begin{tikzpicture}[
declare function={
func(\x)= (\x < -pi/2) * (0)   +
              and(\x >= -pi/2, \x < pi/2) * (pi/2-abs(\x))     +
              (\x >= pi/2) * (0)
   ;
   indicator(\x) = 1-(max(0,abs(\x)-1)-max(0,abs(\x)-1.2))/(1.2-1);
   linear_func(\x) = \x;
   network(\x) = (\x < 0.33)*1.1*\x 
                + and(\x < 0.66, \x >= 0.33) *( 1.2*\x -0.1*0.33)
                + and(\x >= 0.66, \x < 1 ) * (1.3*\x - 0.3*0.33)
                + and(\x >= 1, \x < 1.2) * ((1-(max(0,abs(\x)-1)-max(0,abs(\x)-1.2))/(1.2-1))*1.2 + 1.2*(\x-1)/(1.2-1))
                + (\x >= 1.2) * ((\x-1)+1)
                ;
    true_gradient(\x) = 1.1*\x;
  }]

    \begin{axis}[
        height=5cm,
        width=8.05cm,
        xmin=0,xmax=4,
        ymin=0,ymax=2,
        axis x line=middle,
        axis y line=middle,
        axis line style=-,
        xlabel={$x$},
        ytick = {0,1},
        yticklabels={\footnotesize{$0$},\footnotesize{$1$}},
        xtick = {1,1.2},
        xticklabels={\footnotesize{$r$},\footnotesize{$b$}},
    samples=1001
        ]
        \node at (axis cs:1.2,1.6) {{\footnotesize{\textcolor{blue}{$-\nabla V$}}}};
        \node at (axis cs:1.5,1) {{\footnotesize{\textcolor{cyan}{$x\mapsto mx$}}}};
        \node at (axis cs:1.8,1.4) {{\footnotesize{\textcolor{purple}{FCNN
        }}}};
        \node at (axis cs:2.2,0.5) {{\footnotesize{Appr. indicator on $B_r^{1}(0)$}}};
        
        \addplot [black,dashed] {indicator(x)};
        \addplot [cyan, dashed] {linear_func(x)};
        \addplot [blue, dashed] {true_gradient(x)};
        \addplot [purple, thick] {network(x)};
    \end{axis}
\end{tikzpicture}}
\end{center}
\captionsetup{width=\linewidth}
\caption{
\textit{Sketch of the construction of the network from \cref{prop:combined_assump}. On the $\ell^1$-ball of radius $r$, $B^1_r(0)$, the network approximates $-\nabla V$. On $B_b^1(0)\setminus B_r^1(0)$, where $b > r$, the network (approximately) interpolates linearly between $-\nabla V$ and $x\mapsto mx$. On $\mathbb{R}^d\setminus B_b^1(0)$, the network is identical to $x\mapsto mx$. In this way, global approximation with linearly growing error is achieved. For the precise construction and the choice of $b$, we refer to the proof of \cref{prop:combined_assump}.
}
\label{fig:Network construction}
}
\end{figure}

\begin{proposition}\label{prop:combined_assump}
    Presuppose \cref{assump:potential}, let $\varepsilon>0$ and $r>0$. Assume $m\leq M<\sqrt{2}m$. Furthermore, assume the existence of FCNN approximations, as in \cref{assump:nn_approx_a} with ReLU activation function. Let $\delta = 9\varepsilon/\sqrt{2d}$ and $G = \sqrt{M^2-m^2}$. Then, there exists a FCNN $\phi$ with ReLU activation function and number of nonzero parameters
    \ifnum\classstyle=1
    \begin{equation}
    \begin{aligned}\label{eq:complexity_compnet}
        \mathcal{P}(\phi) = \mathcal{O}\left( d\log(2d\max\{ 1,rG \}/\varepsilon) + N(d,r,\sqrt{2}\varepsilon/\sqrt{d},m,M) + dL(d,r,\sqrt{2}\varepsilon/\sqrt{d},m,M) + 2d^2 \right),
    \end{aligned}
    \end{equation}
    \fi
    \ifnum\classstyle=0
    \footnotesize{
    \begin{equation}
    \begin{aligned}\label{eq:complexity_compnet}
        \mathcal{P}(\phi) = \mathcal{O}\left( d\log(2d\max\{ 1,rG \}/\varepsilon) + N(d,r,\sqrt{2}\varepsilon/\sqrt{d},m,M) + dL(d,r,\sqrt{2}\varepsilon/\sqrt{d},m,M) + 2d^2 \right),
    \end{aligned}
    \end{equation}}
    \fi
    such that
    \begin{align*}
         &\| -\nabla V - \mathcal{R}\phi \|_{L^{\infty}(B^2_r(0))} \leq \varepsilon/\sqrt{2d}, \\
         &\|-\nabla V (x) - \mathcal{R}\phi(x)\|_{\ell^2} \leq \delta + G\|x\|_{\ell^2}, \qquad \text{ for } x\in \mathbb{R}^d.
    \end{align*}
\end{proposition}

In a nutshell, \cref{prop:combined_assump} states that \cref{assump:nn_approx_a} is sufficient to get global linear error bounds on the network approximations, provided that the Lipschitz constant $M$ is not too far away from the convexity parameter $m$. This makes sense intuitively: The closeness of $M$ and $m$ implies that $V$ is close to a quadratic function, and hence $\nabla V$ can be well approximated by a linear function, as \cref{thm:linear_approx} states. The important ingredient for proving the above FCNN result is the fact that this linear function can be represented by a ReLU network. More precisely, in the proof, we construct a network that combines the local approximation on $B_r^2(0)$ granted by \cref{assump:nn_approx_a} with the global approximation property of the function $x\mapsto -mx$ on the complement of $B_r^2(0)$.

When considering the stochastic processes driven by these neural networks, the linear growth allows us to bound the variance proxy in every step, according to \cref{prop:subgaussian_nn_process}. Due to the bounded variance proxies and the arbitrarily small errors on any ball, 
properties of the sub-Gaussian measures can be applied to obtain the desired global errors \eqref{eq: network approx drift} of the drift approximations. 

\begin{theorem}[ResNet-like realization approximated LMC (II)]\label{thm:dnn_lmc_lipschitz2}
    We presuppose the conditions in \cref{assump:potential} for the potential $V$ and \cref{assump:nn_approx_b} for the existence of FCNN approximations of $-\nabla V$ with parameters bounded by $N(d,r,\varepsilon,m, M)$. Furthermore, we assume $M < \sqrt{2}m$. Let $h\in(0,\frac{2}{m+M})$ and $\mu_{0}$ be sub-Gaussian with variance proxy $\sigma_0^2 > 0$. Then, there exists for any $K\in\mathbb{N}$ and any $\varepsilon >0$ a neural network $\psi$ with number of parameters bounded by
\eqref{eq:complexity_compnet} with $r = \mathcal{O}(d(1+\ln(d^2\varepsilon^{-4})^{\frac{1}{2}}))$,
    such that the measure $\mu^{\Psi}$ of the ResNet-like realization of $\Psi \coloneqq \{ \psi_k\coloneqq \psi \}_{k=1}^K$ with input $(Y_0,\xi)$, i.e. the law of $\tilde{\mathcal{R}}\Psi(Y_0,\xi)$, satisfies
\begin{equation}
    \mathcal W_2(\mu_{\infty},\mu^{\Psi}) \leq (1-mh)^K\mathcal W_2(\mu_{\infty}, \mu_{0}) + \frac{7\sqrt 2}{6} \frac M m \sqrt{hd} + \dfrac{1-(1-mh)^K}{m}\varepsilon.
\end{equation}
\end{theorem}

The proof of the results in this section can be found in~\cref{sec:proofsLocalApprox}.
\section{Experiments}\label{sec:Numerics}

The theoretical results show that if the Langevin dynamics can be approximated correctly, exponential convergence to the target distribution $\mu_{\infty}$ can be expected~(\cref{thm:nn_approx_lmc}). With the following experiments, we want to prove that such an approximation has practical use in computing quantities of interest of the target distribution. As target distributions, we consider a simple Gaussian distribution, a Gaussian mixture and a Bayesian posterior based on a forward problem defined by the Darcy problem with random data, which can be formulated as parametric partial differential equation (PDE). Note that only the Gaussian density is covered by the theory since any nontrivial Gaussian mixture violates the strong convexity assumption and the Darcy posterior is not available in closed form. However, we can experimentally show that the proposed approach is also feasible for more complex densities. The Wasserstein distance is calculated with the python package POT~\cite{flamary2021pot}.

\subsection{Gaussian distribution}

As a first example, we define $\mu_{\infty} = \mathcal{N}((2,2)^T,I_d)$ to be a simple Gaussian distribution, which is just a translation of a standard normal. The initial distribution is chosen as standard normal, i.e. $\mu_0 = \mathcal{N}(0,I_d)$.

We use $K=1000$ LMC steps with a step-size of $h=0.01$. In \cref{fig:LMC_gauss1}, the green curve shows the mean Wasserstein distance of an ensemble of $100$ LMC particles, $\hat{\mu}_t^{(LMC)}$, to the discrete measure defined by an ensemble of $100$ samples from the true target distribution, which we denote by $\hat{\mu}_{\infty}$. Averages are taken over $200$ runs over the LMC distribution and over target ensembles. The upper and lower boundaries of the shaded region are the mean plus/minus one standard deviation, respectively. The constant black line is the mean distance of $100$-particle target ensembles $\hat{\mu}_{\infty}$ to other (independently sampled) $100$-particle target ensembles $\tilde{\mu}_{\infty}$, again averaged over $200$ different draws. Note that, first, this distance is expected to be larger than zero for finite numbers of particles and, second, the best that the LMC sampler can achieve is to meet the same larger-than-zero distance. After $1000$ steps of the sampler, LMC seems to have converged to the target, measured in terms of the mean Wasserstein distance and its variance.

We use the average of the LMC samples at $K=1000$ steps to estimate the mean of the target $\mu_\infty$, $(2,2)^T$. The green line in 
\cref{fig:LMC_gauss2} shows the mean error (averaged over $200$ independent LMC runs) of this estimate over the used number of LMC samples. Until a slight increase of the error from $80$ to $100$ samples, which we attribute to the variance of the estimate, the error is decreasing with increasing sample size. 
For each number of LMC samples, a FCNN is trained to approximate the drift using all drift evaluations from the whole trajectories of the LMC solutions. This FCNN is then used in a ResNet-like architecture as described in the previous sections. 
The trained ResNet-like network can now be used to generate an arbitrary amount of samples from the (approximate) target. In this example, we average over $1000$ samples generated by the ResNet-like architecture to estimate the mean of the target distribution. The resulting errors of the mean estimate are given by the blue line in \Cref{fig:LMC_gauss2}. 

\begin{figure}
    \centering
    \begin{subfigure}{0.45\textwidth}
        \centering
        \includegraphics[width=\linewidth]{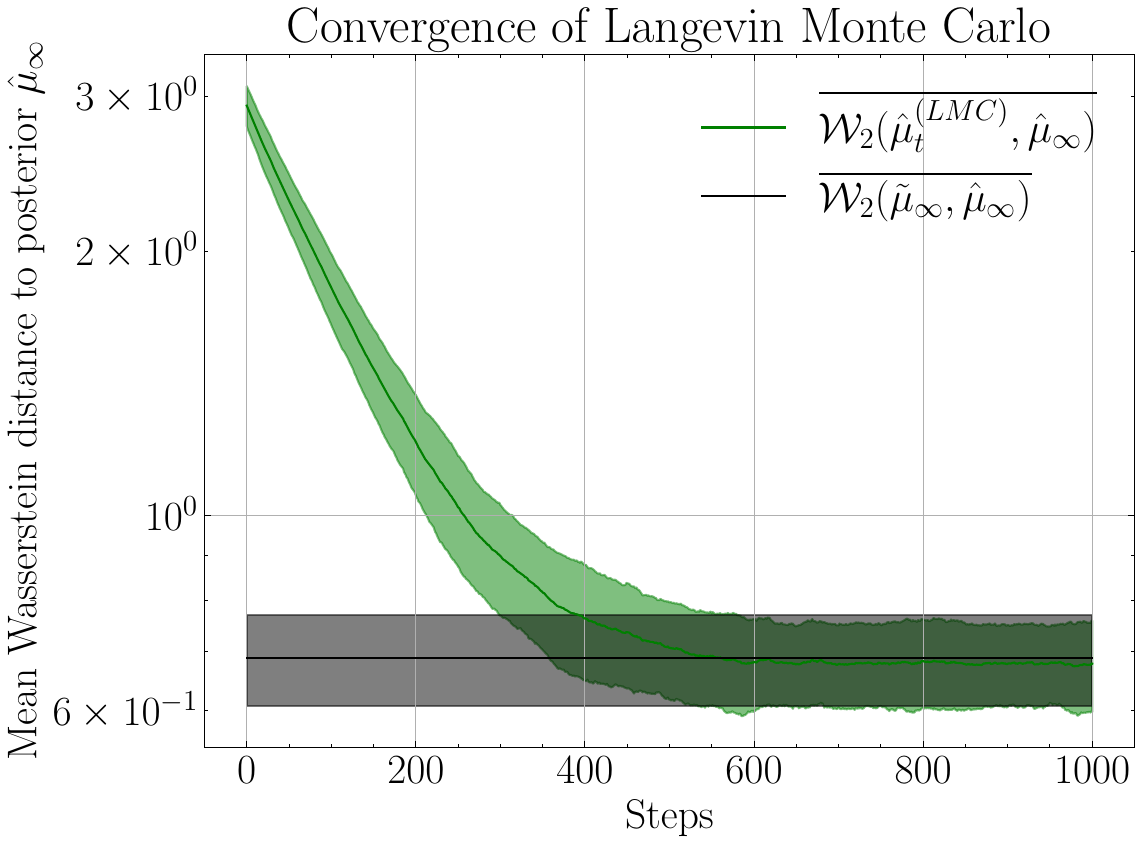}
        \caption{The convergence of the LMC sample ensemble to the Gaussian target distribution is shown. The black line shows the mean Wasserstein distance between particle ensembles drawn from the Gaussian, the shaded region is $\pm$ its standard deviation. The green line is the mean Wasserstein distance of the LMC ensemble to ensembles drawn from the Gaussian $\pm$ the standard deviation. Both means are approximated over $200$ draws.}
        \label{fig:LMC_gauss1}
    \end{subfigure}
    \hfill
    \begin{subfigure}{0.45\textwidth}
        \centering
        \includegraphics[width=\linewidth]{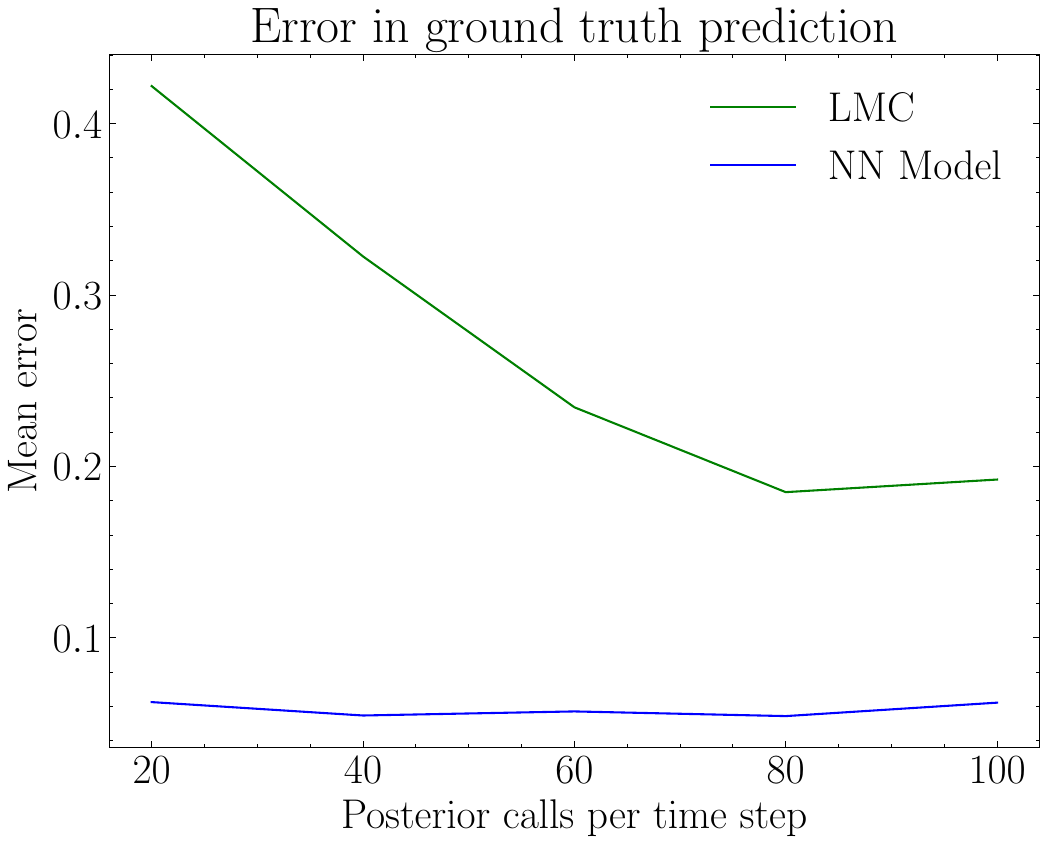}
        \caption{The averaged $\ell^2$ error of the approximation of the mean of the Gaussian using LMC (green) and an NN approximation (blue) and using the same number of target distribution calls per time step for both methods is shown. The NN approximation can produce additional ``cheap'' samples for approximation of the mean. Hence, the approximation is better than for LMC for all numbers of used distribution calls.}
        \label{fig:LMC_gauss2}
    \end{subfigure}
    \caption{Comparison between LMC and NN model on a Gaussian target distribution.}
    \label{fig:GaussianNumerics}
\end{figure}

In all cases, the estimate of the NN is more accurate than the LMC estimate. We stress that this is due to the simple fact that the network uses \textit{more} samples for its estimate compared to LMC while being trained with the same amount of evaluations of the target. In the Gaussian case, evaluations of the target are of course trivial.
Hence, there is no reason why we should not also run LMC with $1000$ samples. However, consider the case where evaluation of the target is expensive, i.e. because it involves the solution of a PDE. In this case, increasing the number of samples for LMC would lead to a significantly increased computational complexity. In contrast to this,  once it is trained trained, the NN model can produce additional samples with no additional cost (measured in required target calls).
We argue that the simple statistical advantage of being able to produce more samples from the NN surrogate without having to invest in more target calls is a strong motivation for such a model in practical uses. In the last example of this section, we demonstrate this property for an inverse problem arising from a parametric PDE model.

\subsection{Gaussian mixture distribution}

To showcase that the proposed approach may also viable for problems which do not adhere to the strong convexity assumption imposed by the theory, we consider a target measure $\mu_{\infty}$ defined by a mixture of Gaussians in $d=2$ dimensions with two modes in $(-1,-1)^T$ and $(1,1)^T$ and with identity covariance each. The initial distribution is again standard normal $\mu_0 = \mathcal{N}(0,I_d)$. Note that the multimodality breaks the strong convexity required in \cref{assump:potential}.

We use $K=1000$ LMC steps with a step-size of $h=0.002$. \Cref{fig:LMC_MM1} shows the mean Wasserstein distance of an ensemble of $100$ LMC particles, $\hat{\mu}_t^{(LMC)}$, to an ensemble of $100$ samples from the true target, $\hat{\mu}_{\infty}$ and compares this to the ``target-target'' distance $\mathcal{W}_{2}(\tilde{\mu}_\infty,\hat{\mu}_\infty)$ just like described in the previous section. Averages are taken over $1000$ runs over both the LMC distribution as well as over target ensembles. This high number of runs is chosen because of the high variance of the estimates, resulting from the multimodality. After $1000$ steps of the sampler, LMC seems to have converged to the target as measured by the mean Wasserstein distance and its variance.

We use the average of the LMC samples at $K=1000$ steps to estimate the mean of the target $\mu_\infty$, which is $(0,0)^T$. The green line in 
\Cref{fig:LMC_MM2} shows the mean error (averaged over $70$ independent LMC runs) of this estimate over the number of LMC samples used. 
We train ResNet-like NNs with the LMC samples as described in the previous section to generate an arbitrary amount of samples from the (approximate) target. Again, we average over $1000$ samples generated by the ResNet-like architecture to estimate the mean of the target distribution. The resulting errors of the mean estimate are given by the blue line in \Cref{fig:LMC_MM2}. Just like in the Gaussian case, the estimates produced by the NN are more accurate than the LMC estimates with the same number of used target calls.

\begin{figure}
    \centering
    \begin{subfigure}{0.45\textwidth}
        \centering
        \includegraphics[width=\linewidth]{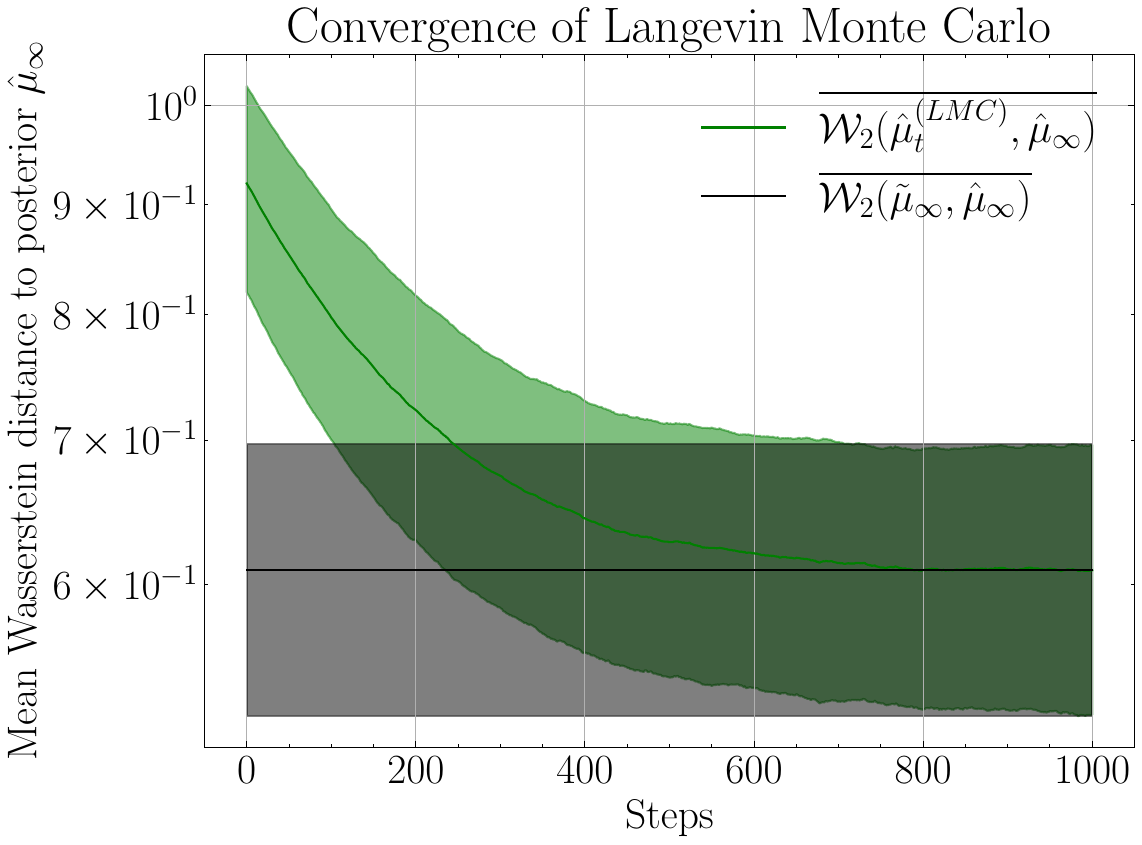}
        \caption{The convergence of the LMC sample ensemble to the Gaussian mixture target distribution is shown. The black line shows the mean Wasserstein distance between particle ensembles drawn from the Gaussian mixture, the shaded region is $\pm$ its standard deviation. The green line is the mean Wasserstein distance of the LMC ensemble to ensembles drawn from the Gaussian mixture $\pm$ the standard deviation. Both means are approximated over $1000$ draws.}
        \label{fig:LMC_MM1}
    \end{subfigure}
    \hfill
    \begin{subfigure}{0.45\textwidth}
        \centering
        \includegraphics[width=\linewidth]{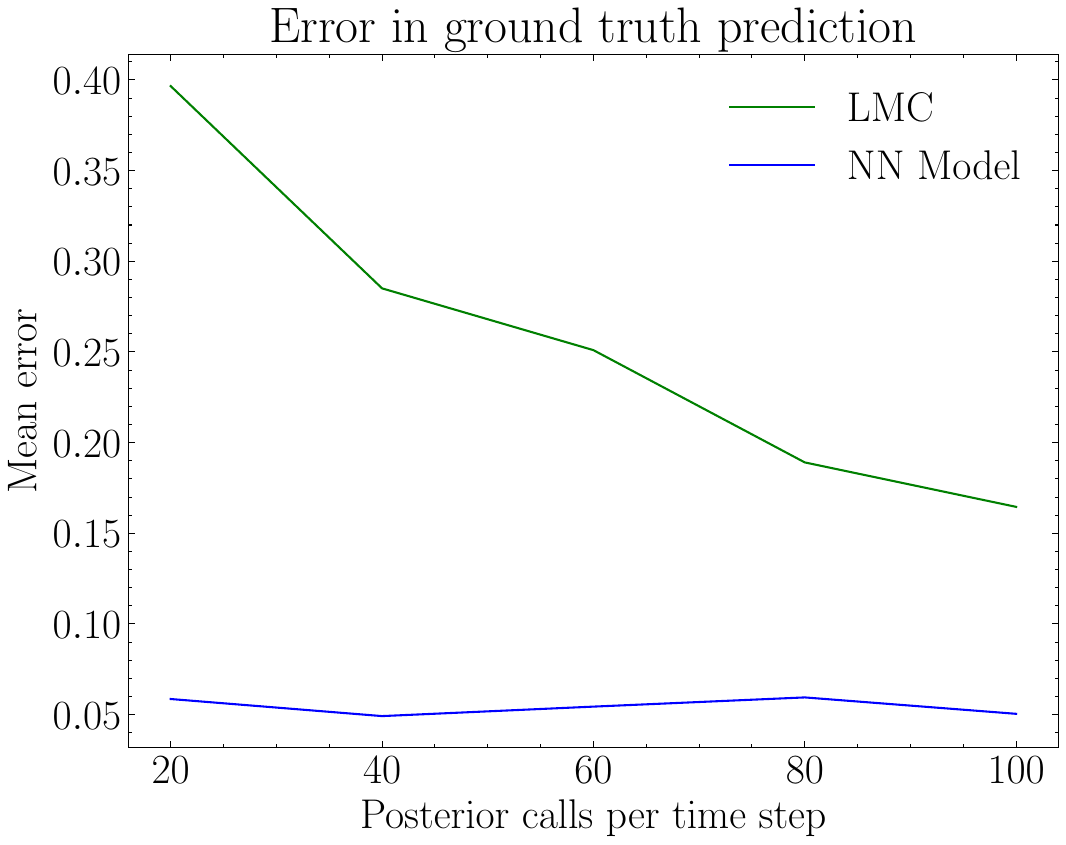}
        \caption{The averaged $\ell^2$ error of the approximation of the mean of the Gaussian mixture using LMC (green) and an NN approximation (blue) and using the same number of target distribution calls per time step for both methods is shown. The NN approximation can produce additional ``cheap'' samples for approximation of the mean. Hence, the approximation is better than for LMC for all numbers of used distribution calls.}
        \label{fig:LMC_MM2}
    \end{subfigure}
    \caption{Comparison between LMC and NN model for a Gaussian mixture target density.}
    \label{fig:GMMNumerics}
\end{figure}

\subsection{Parametric PDE - inverse parametric Darcy problem}

As a last experiment, we consider a linear elliptic partial differential equation (PDE) of the form
\begin{align}
    -\nabla_x \cdot (a(\cdot, y)\nabla_x u) &= f, \quad \text{on}~D\label{eq:darcy-flow}\\
    u_{|\partial D} &= 0, \label{eq:darcy-flow-bc}
\end{align}
where $D\subset \mathbb{R}^d$ is a bounded Lipschitz domain and $f\colon D \to \mathbb{R}$ is a \textit{source term} or \textit{force}.
The interesting aspect, which makes this a common benchmark in the field of Uncertainty Quantification (UQ), is the diffusion coefficient $a(x,y)$ defined for $x\in D$ and a $p$-dimensional parameter vector $y \in \mathcal Y \subseteq \R^p$, rendering this a high-dimensional problem.
There is a proven repertoire of numerical methods to efficiently solve this forward problem such as~\cite{EigelGittelson2014asgfem, eigel2023adaptive, carstensen2012review, eigel2018adaptive, stochafem} and also NN approaches~\cite{cosi, schutte2024multilevelcnnsparametricpdes, schutte2024adaptive, marcati2022exponential, elbrachter2022dnn, gitta, moritz}, to name but a few.

A problem closely related to and even more challenging than this forward problem is the \emph{Bayesian inverse problem}. Here, the task is to infer the parameter $y$ based on a set of noisy measurements of the solution $u$ for a specific parameter realization.
The representation of this problem with NN has been approached in the generative modeling framework with various architectures, e.g. continuous normalizing flows \cite{ruthotto2021introduction}, diffusion models~\cite{yang2022diffusion}, invertible neural networks~\cite{ardizzone2018analyzing} and others, see~\cref{sec:relatedWork}.
For the Bayesian framework used here, the inverse problem is treated by associating a prior distribution (which we represent by a density $f_{\mathrm{prior}}$, assuming tacitly that it exists) with $y$ and by updating it successively via Bayes' rule \cite{kaipio2006statistical}.
To that end, assume an observation given by $$\mathbf u_{\mathrm{obs}} = \mathcal G(y) + \eta$$ of $u$, where $\mathcal G := L \circ \mathcal S : \mathcal Y \to \mathbb{R}^m$ is the \emph{forward operator}, $L : H_0^1(D) \to \R^m$ is a bounded \emph{linear observation operator} such as pointwise evaluations of $u$, $\mathcal S : \mathcal Y \to H_0^1(D)$ is the \emph{solution operator} that maps $y$ to the solution of \eqref{eq:darcy-flow}, and $\eta \sim \mathcal N_m(0, \Sigma)$ is zero mean Gaussian observation noise with covariance $\Sigma\in\mathbb{R}^{m\times m}$.
The resulting \emph{posterior density} over the parameters has the form
\begin{equation}
    \mu_{\infty} \coloneqq f_{\mathrm{post}}(y|\mathbf u_{\mathrm{obs}}) \propto \exp\left(-\dfrac{1}{2}\| \mathcal{G}(y)-\mathbf u_{\mathrm{obs}} \|^2_\Gamma \right)f_{\text{prior}}(y).
    \label{eq:darcy-flow-ivp-post}
\end{equation}

Even though $L$ is computationally cheap, the \emph{forward operator} $\mathcal G$ is expensive due to the \emph{solution operator} $\mathcal S$ that is used to solve the PDE \eqref{eq:darcy-flow} given some parameter $y$. The PDE can for instance be solved by the common \emph{finite elements method} (FEM)~\cite{Hackbusch, adaptivenochetto}.
In practice, one often wants to sample from \eqref{eq:darcy-flow-ivp-post} in order to perform a Monte Carlo approximation of "quantities of interest" such as moments of the distribution. Thus, it often requires a lot of solution computations to get accurate results. 

We give a simple example of the usefulness of a learned NN approximation of LMC, by considering approximations of the ground truth diffusion coefficient $a(\cdot,y)$ in \eqref{eq:darcy-flow}. In our example, the PDE is one-dimensional, $d=1$, and discretized on a \emph{computational grid} $x_i = 2\pi i / D$, $i=0,\ldots,D-1$, with $D=20$ points and mesh size $h=\frac{2\pi}{D}$.
The forcing on each point of the computational grid is given by $f_i = \exp\left(-(2x_i-2\pi)^2/40\right) - \frac{3}{5}$. The ground truth on each grid point is given by $a_i(y) = \exp(y_i)$, where $y_i = \frac{1}{2}\sin(x_i-\frac{h}{2})$. The solution operator $\mathcal{S}$ is now defined via a finite difference (FD) approximation of \eqref{eq:darcy-flow} on the computational grid. For the observation operator $L$, we define an \textit{observation grid} consisting of every second point in the computational grid, i.e. $x_{2i+1}$ for $i=0,\ldots,9$. The observation operator $L$ then maps the FD approximation of the PDE computed on the computation grid to its restriction to the observation grid. Thus, our forward operator $\mathcal{G} = L\circ \mathcal{S}$ is defined. Additionally, we add i.i.d. Gaussian noise $\eta_i\sim \mathcal{N}(0, 10^{-4}), i=0,\dots,9$ to every entry of the solution on the observation grid. Generating an observation $\mathbf{u}_{\mathrm{obs}}$ by the procedure described above and specifying a prior density $f_{\mathrm{prior}}$ on the parameter $y$, we have fixed a posterior density \eqref{eq:darcy-flow-ivp-post}. For details regarding the FD approximation and a suitable prior, we refer to \cite{eigel2024less}.

We use LMC to generate samples from \eqref{eq:darcy-flow-ivp-post}. To compute the potential gradient, which involves the gradient of the forward operator, we use the inbuilt \texttt{autograd} method in \texttt{pytorch}. We use $K=2000$ LMC steps with a step-size of $h=10^{-4}$. In \Cref{fig:LMC_darcy1}, the green curve shows the mean Wasserstein distance of an ensemble of $100$ LMC particles, $\hat{\mu}_t^{(LMC)}$, to an ensemble of $100$ samples from the true posterior, $\hat{\mu}_{\infty}$. Averages are taken over $70$ runs over both the LMC distribution as well as over posterior ensembles. The upper and lower boundaries of the shaded region are the mean plus/minus one standard deviation respectively. For the posterior reference, we use an affine invariant Markov Chain Monte Carlo (MCMC) sampler provided by the Python package \texttt{emcee}~\cite{2013PASP..125..306F}. The mean distance can be seen to decrease over time to just under $10^{-1}$. The constant black line at this level is the mean distance of $100$-particle posterior ensembles, $\hat{\mu}_{\infty}$, to other (independently sampled) $100$-particle posterior ensembles, $\tilde{\mu}_{\infty}$, again averaged over $70$ different draws. Note that, first, this distance is expected to be larger than zero for finite numbers of particles and that, second, the best that the LMC sampler can achieve is meet the same larger-than-zero distance. Around the $2000$ step mark, the LMC sampler seems to have converged with small relative error compared to a true posterior sampler. We stress that this is only a heuristic. To really track convergence in distribution, higher moments than expectation and variance would have to be considered. However, we are not interested in achieving perfect convergence and are content with reasonable approximations. Hence, we interpret \Cref{fig:LMC_darcy1} as sufficient evidence of convergence for the sake of this section.

\begin{figure}
    \centering
    \begin{subfigure}{0.45\textwidth}
        \centering
        \includegraphics[width=\linewidth]{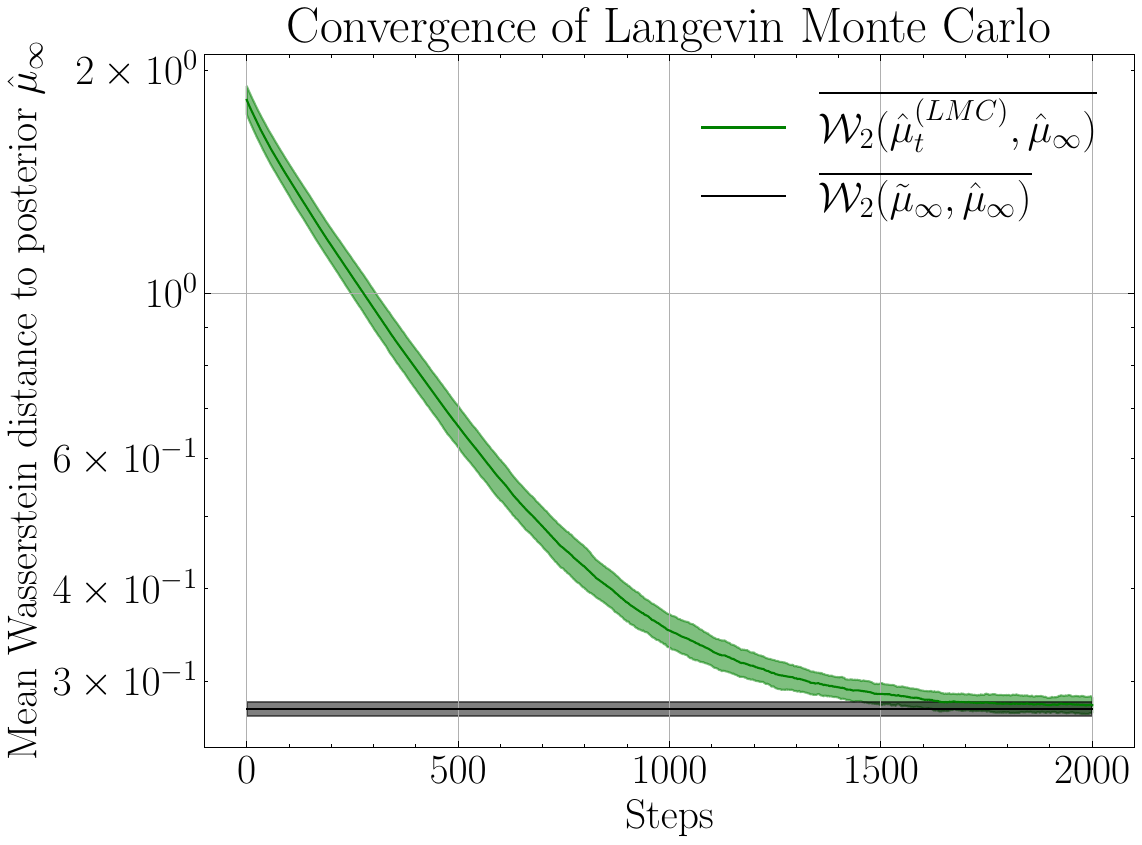}
        \caption{The convergence of the LMC sample ensemble to the Darcy posterior distribution is shown. The black line shows the mean Wasserstein distance between particle ensembles drawn from the Darcy posterior (using the MCMC package \texttt{emcee}~\cite{2013PASP..125..306F}), the shaded region is $\pm$ its standard deviation. The green line is the mean Wasserstein distance of the LMC ensemble to ensembles drawn from the Darcy posterior $\pm$ the standard deviation. Both means are approximated over $70$ draws.}
        \label{fig:LMC_darcy1}
    \end{subfigure}
    \hfill
    \begin{subfigure}{0.45\textwidth}
        \centering
        \includegraphics[width=\linewidth]{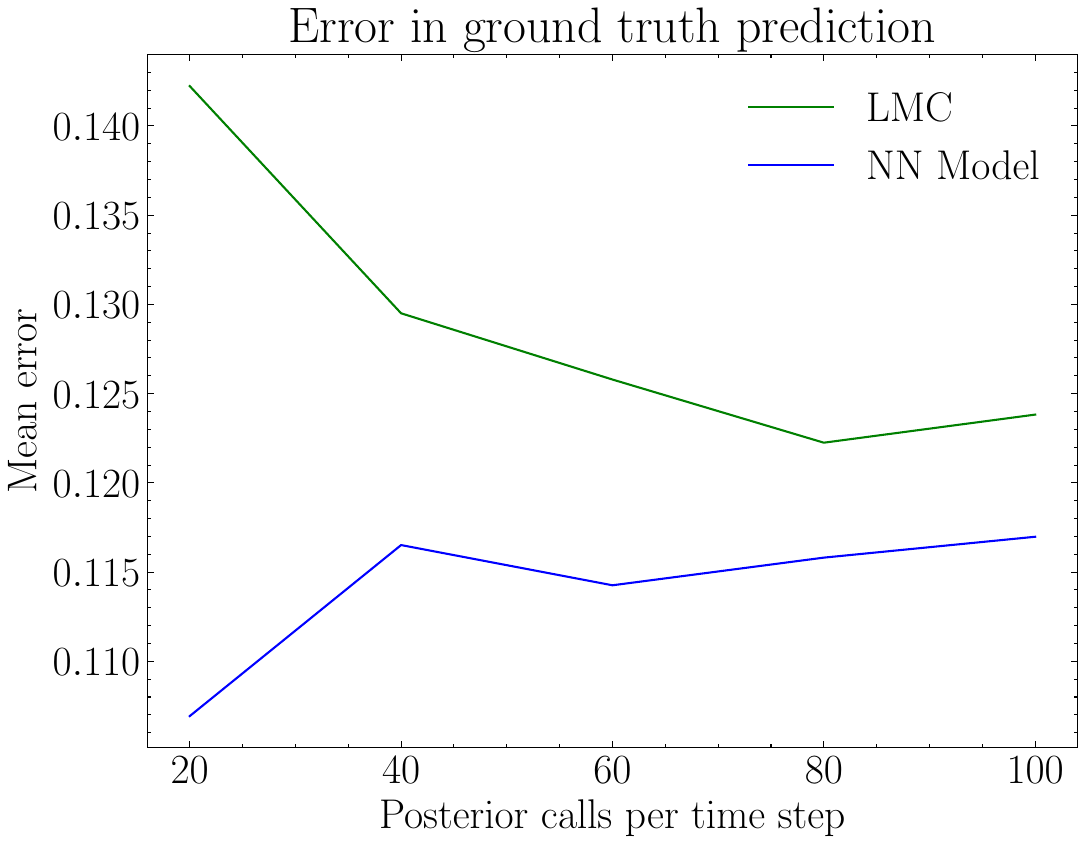}
        \caption{The averaged $\ell^2$ error of the approximation of the mean of the posterior distribution using LMC (green) and an NN approximation (blue) and using the same number of target distribution calls per time step for both methods. The NN approximation can produce additional ``cheap'' samples for approximation of the mean. Hence, the approximation is better than for LMC for all numbers of used distribution calls.}
        \label{fig:LMC_darcy2}
    \end{subfigure}
    \caption{Comparison between LMC and NN model on the inverse parametric Darcy problem.}
    \label{fig:DarcyNumerics}
\end{figure}

We use the average of the LMC samples at $K=2000$ steps to estimate the mean of the posterior $\mu_\infty$ as an approximation to the ground truth $y$. The green line in \Cref{fig:LMC_darcy2} shows the mean error (averaged over $70$ independent LMC runs) of this estimate compared to the actual ground truth over the number of LMC samples, which is equal to the number of posterior calls necessary per LMC step. Until a slight increase of the error from $80$ to $100$ samples, which we attribute to a high variance of the estimate, the error is decreasing with increasing sample size. 
As described for the Gaussian problem, we train ResNet-like NNs using the full LMC trajectories to generate more samples from the (approximate) posterior, with which the ground truth estimate is computed. In this example, $1000$ samples generated by the ResNet-like architecture were used. The resulting errors are given by the blue line in \Cref{fig:LMC_darcy2}. In all cases, the estimate of the NN is more accurate than the LMC estimate. We stress again that this is due to the simple fact that the network can produce \textit{more} samples without requiring more calls to the solution operator $\mathcal{G}$. Once the model is trained, it can produce new samples ``for free'', in the sense that no additional calls to the solution operator are required. 

\section{Conclusion}\label{sec:Conclusion}

Expressivity results for ResNet-like neural networks mapping samples from a prior distribution to a smooth log-concave posterior distribution with arbitrary accuracy were derived. To that end, an upper bound for the decay of the Wasserstein-$2$ distance for the perturbed Langevin Monte Carlo process with approximate gradient steps was shown.
Neural networks are used as an approximation for the drift in every step under different approximation assumptions.

In a first approach, global approximation of the drift with linear error growth is assumed. In this case, the variance proxies of the sub-Gaussian intermediate measures of the perturbed LMC process can be bounded uniformly (\cref{prop:subgaussian_nn_process}).
Hence, the size of the ResNet-like neural network only depends linearly on the number of steps carried out, see \cref{thm:lmc_approx_ii}.

If only a local approximation on a compact ball is assumed, we require more structure of the target potential. In particular, if we assume that the Lipschitz constant of the potential gradient is smaller than $\sqrt{2}$ times the strong convexity constant, the local approximation on a ball is sufficient to ensure a uniform bound for the variance proxies of the intermediate distributions of the process. Hence, robust error bounds, depending only linearly on the number of steps taken, can be derived. This is the result of \cref{thm:dnn_lmc_lipschitz2}.

The proposed architecture is verified numerically on a Gaussian, a Gaussian mixture and a parametric PDE posterior as target distributions. We observe that the ResNet-like architecture with intermediate feed-forward networks of the same size shows better expectation approximations due to being able to generate more samples without more posterior evaluations, see \Cref{fig:GaussianNumerics,fig:GMMNumerics,fig:DarcyNumerics}. The architecture of the networks allows to train small networks in every step, which enables short training processes. 
 
When approximating Langevin dynamics with the described ResNet-like neural network, an upper bound on the required steps for an accurate posterior distribution approximation is shown. As an avenue for future research, a potential decrease in complexity could be achieved by learning multiple Langevin steps with a single FCNN approximation.

\section*{Acknowledgements}
ME \& JS acknowledge funding from the Deutsche Forschungsgemeinschaft (DFG, German Research Foundation) in the priority programme SPP 2298 "Theoretical Foundations of Deep Learning". ME \& CS \& DS acknowledge support by the ANR-DFG project \textit{COFNET: Compositional functions networks - adaptive learning for high-dimensional approximation and uncertainty quantification}. This study does not have any conflicts to disclose.

\newpage
\appendix
\section{Proofs of main results}

\setcounter{subsection}{1}

\subsection{Preliminaries: Sub-Gaussianity and Lyapunov functions}

\begin{proof}[Proof of \cref{prop:lp_subgaussian}]
    By definition of a sub-Gaussian random vector, for any $u \in \mathbb S_1(0)$ and any $p\in\mathbb{N}$, we have
    \begin{equation*}
        \E[|\langle u, Z\rangle|^p] \leq (\sqrt 2 \sigma)^p p \Gamma\left(p/2\right).
    \end{equation*}
    Now, let $q \geq 1$ be arbitrary and note that by Jensen's inequality we have
    \begin{equation*}
        \|Z\|_{\ell^p} \leq \max(d^{1/p - 1/q}, 1) \|Z\|_{\ell^q}.
    \end{equation*}
    Furthermore, for any $i\in\{ 1,\ldots,d \}$ and $u=e_i$ we have $\E[|Z_i|^q] = \E[|\langle u, Z\rangle|^q] \leq (\sqrt 2 \sigma)^q q \Gamma\left(q/2\right)$.
    Taking the expectation leads to
    \begin{align*}
        \E[\|Z\|_{\ell^p}^q] &\leq \max(d^{q/p - 1}, 1) \E[\|Z\|_{\ell^q}^q]\\
        &\leq \max(d^{q/p},d) (\sqrt 2 \sigma)^q q \Gamma\left(q/2\right)\\
        &\leq (\sqrt 2 d\sigma)^q q \Gamma\left(q/2\right).
    \end{align*}
    Hence, $\|Z\|_{\ell^p}$ is sub-Gaussian.
\end{proof}

\begin{proof}[Proof of \cref{rem:lyapunov_subgauss}]
    If $X$ is a sub-Gaussian random vector with variance proxy $\sigma^2$ then by \cref{def:subgaussian_rvector} and \cref{def:subgaussian_rv} it holds that $\mathbb{E}\left[e^{s\langle v,X\rangle }\right] \leq e^{\frac{\sigma^2 s^2}{2}}$ for all $s\in\mathbb{R}$ and $v\sim \mathbb{S}_1(0)$. In particular, it holds for any $\lambda > 0$ that
    \begin{equation*}
\mathbb{E}_{X}\mathcal{L}_{\lambda}(X) \leq e^{\frac{\sigma^2 \lambda^2}{2}}.
    \end{equation*}
    The reverse is not obvious, but does in fact also hold. Let $\sigma^2$ be such that $\mathbb{E}_{X}\mathcal{L}_{\lambda}(X) \leq e^{\frac{\sigma^2 \lambda^2}{2}}$. Let $u\in\mathbb{S}_1$ be arbitrary and $\mathcal{R}$ denote the Haar measure over rotations in $\mathbb{R}^d$. Then it holds for $\lambda > 0$ that
    \begin{align*}
\mathbb{E}_{X}\mathcal{L}_{\lambda}(X) = \mathbb{E}_{X}\mathbb{E}_{v\sim \mathbb{S}_1(0)}\left[ 
 e^{\lambda \langle v,X \rangle} \right] = \mathbb{E}_{X}\mathbb{E}_{R\sim \mathcal{R}}\left[ 
 e^{\lambda \langle Ru,X \rangle} \right] = \mathbb{E}_{X}\mathbb{E}_{R\sim \mathcal{R}}\left[ 
 e^{\lambda \langle u,RX \rangle} \right].
    \end{align*}
    Hence,
    \begin{align*}
\mathbb{E}_{X}\mathbb{E}_{R\sim \mathcal{R}}\left[ 
 e^{\lambda \langle u,RX \rangle} \right] \leq e^{\frac{\sigma^2 \lambda^2}{2}} \qquad \text{for all } u\in \mathbb{S}_1,\ \lambda \in \mathbb{R},
    \end{align*}
    where the inequality for negative $\lambda$ follows since the negative sign can be absorbed into the rotation.
    By definition, this implies that the random variable $RX$ is sub-Gaussian with variance proxy $\sigma^2$. Let $R^\ast$ denote the adjoint of the rotation $R$. It follows for any $r>0$ and $u\in\mathbb{S}_1$ that
    \begin{align*}
        \mathbb{P}(|\langle u, X \rangle| \geq r ) = \mathbb{P}(|\langle R^\ast 
 (R^\ast)^{-1}u, X \rangle| \geq r ) = \mathbb{P}(|\langle  
 (R^\ast)^{-1}u, R X \rangle| \geq r ) \leq \exp\left( -\frac{r^2}{2\sigma^2} \right),
    \end{align*}
since $(R^\ast)^{-1} u$ is again an element of $\mathbb{S}_1$. By definition, this yields the sub-Gaussianity of $X$ with variance proxy $\sigma^2$.
\end{proof}
\subsection{ResNet-like architectures}\label{sec:Proofof-prop:resnet_equals_driven_process}

\begin{proof}[Proof of \cref{prop:resnet_equals_driven_process}]
    For $i=1,\ldots,K$, let $L_i,w_0,\ldots,w_{L_i+1}\in\N$, and $$((A^i_0, b^i_0), \ldots , (A^i_{L_i}, b^i_{L_i})) \in \bigtimes^{L_i}
_{\ell=0}(\mathbb{R}^{w_{\ell+1}\times w_\ell} \times \mathbb{R}^{w_{n+1}})$$
    such that $\phi_{i} = ((A^{i}_0, b^{i}_0), \ldots , (A^{i}_{L_i}, b^{i}_{L_i}))$ and let  
\begin{equation}
   \psi_i = ((A^{i}_0, b^{i}_0), \ldots , (hA^{i}_{L_i}, hb^{i}_{L_i})). 
\end{equation}
Then the ResNet-like realization of $\Psi = \{\psi_i\}_{i=1}^K$ satisfies for every $\omega\in\Omega$ with $x \coloneqq Y_0(\omega)$ that
\begin{equation}
\begin{aligned}
    \widetilde{\mathcal R}\Psi( x, \xi(\omega)) &= x_{k-1} + \mathcal R\psi_{k}(x_{K-1}) + \xi_K(\omega) \\
        &= x + \sum_{i=1}^{K}{\mathcal R \psi_i (x_{i-1}) + \xi_i(\omega)}\\
        &= x + \sum_{i=1}^{K}{h\mathcal R \phi_i (x_{i-1}) + \xi_i(\omega)} \\
        &= x + \int_0^{Kh}{\mathcal R\phi_{{1+\frac{1}{h}\chi_{h}(s)}} (Y_{\chi_h(s)}^{\Phi}(\omega))\mathrm ds} + \sqrt{2} W_{Kh}(\omega) = Y_{Kh}^{\Phi}(\omega).
\end{aligned}
\end{equation}
\end{proof}
\subsection{Perturbed Langevin Monte Carlo}

\label{sec:proofPerturbedLMC}

\begin{proof}[Proof of \cref{thm:nn_approx_lmc}]
Let $\rho\colon \left(0,\frac{2}{M}\right)\rightarrow [0,1)$ be defined by
\begin{equation*}
    \rho(h) := 
    \begin{cases}
        1-mh & \text{if}~0 < h < \frac{2}{m+M}\\
        Mh-1 & \text{if}~\frac{2}{m+M} \leq h < \frac 2 M
    \end{cases}.
\end{equation*}
Consider the triangle inequality
\begin{align*}
    \mathcal{W}_2(\mu^\infty, \mu_{Kh}^{\Phi}) \leq  \mathcal{W}_2(\mu_\infty, \mu^{\tilde{X}}_{Kh}) + \mathcal{W}_2(\mu^{\tilde{X}}_{Kh}, \mu^{\Phi}_{Kh}).
\end{align*}
By \cref{thm:LMC_approx}, the first term can be bounded by 
\begin{equation}\label{eq:first_bound}
    \mathcal W_2(\mu_{\infty},\mu^{\tilde X}_{Kh}) \leq \rho(h)^K\mathcal W_2(\mu_{\infty}, \mu_{0}) +\begin{cases}
        \frac{7\sqrt 2}{6} \frac M m \sqrt{hd}, \qquad &h\in\left( 0,\frac{2}{m+M} \right] \\
        \frac{7\sqrt 2}{6} \frac{Mh}{2-Mh} \sqrt{hd}, \qquad &h\in\left[ \frac{2}{m+M}, \frac{2}{M}\right)
    \end{cases}.
\end{equation}
For $i=0,\ldots,K$, define $\Delta_i = \tilde{X}_{ih}-Y_{ih}$ and note that for $i=0,\ldots,K-1$ it holds that
\begin{align*}
    \Delta_{i+1} &= \tilde{X}_{ih}-Y_{ih} + h(-\nabla V(\tilde X_{ih}) - \phi_{i+1}(Y_{ih})) \\
    &= \tilde{X}_{ih}-Y_{ih} - h\underbrace{(\nabla V(\tilde X_{ih})-\nabla V(Y_{ih}))}_{\coloneqq u_i} + h\underbrace{(-\phi_{i+1}(Y_{ih}) - \nabla V(Y_{ih}))}_{\coloneqq v_i}
\end{align*}
and
\begin{align*}
    \mathbb{E}[\| \Delta_{i+1} \|_{\ell^2}^2]^{1/2} \leq \mathbb{E}[\| \Delta_{i} - hu_i \|_{\ell^2}^2]^{1/2} + \mathbb{E}[\| hv_i \|_{\ell^2}^2]^{1/2}.
\end{align*}
By the assumption that $\|-\nabla V - \phi_{i+1}\|_{L_{\mu^{\Phi}_{ih}}^2(\mathbb{R}^d;\mathbb{R}^d)} < \varepsilon$ for all $i=0,\ldots,K-1$, we have
\begin{align*}
    \mathbb{E}[\| hv_i \|_{\ell^2}^2]^{\frac{1}{2}} \leq h\varepsilon.
\end{align*}
Furthermore, we get for all $i=0,\ldots,K$ that
\begin{align*}
    \mathbb{E}[\| \Delta_{i} - hu_i \|_{\ell^2}^2]^{1/2} &\leq (1-mh)\mathbb{E}[\| \Delta_i \|_{\ell^2}^2]^{1/2}, \qquad h\in\left(0,2/(m+M)\right], \\
    \mathbb{E}[\| \Delta_{i} - hu_i \|_{\ell^2}^2]^{1/2} &\leq (Mh-1)\mathbb{E}[\| \Delta_i \|_{\ell^2}^2]^{1/2}, \qquad h\in\left[ 2/(m+M), 2/M\right),
\end{align*}
by \cref{lem:bound_delta}. Combining these estimates, we arrive at
\begin{equation}\label{eq:second_bound}
    \begin{aligned}
        \mathcal{W}_2(\mu^{\tilde{X}}_{Kh}, \mu^{\Phi}_{Kh}) & \leq \mathbb{E}[\| \Delta_{K} \|_{\ell^2}^2]^{1/2} \leq h\varepsilon + \rho(h)\mathbb{E}[\| \Delta_{K-1} \|_{\ell^2}^2]^{1/2}  \\
        &\leq \sum_{\ell=0}^{K-1} \rho(h)^{\ell}h\varepsilon + \rho(h)^K \mathbb{E}[\| \Delta_0 \|_{\ell^2}^2]^{1/2} \\
    &= \dfrac{1-\rho(h)^K}{1-\rho(h)}h\varepsilon + \rho(h)^K \mathbb{E}[\| \Delta_0 \|_{\ell^2}^2]^{1/2}.
    \end{aligned}
\end{equation}
Now, since $\mu_0^{\tilde X} = \mu_0^{\Phi} = \mu_0 $, the term $\mathbb{E}[\| \Delta_0 \|_{\ell^2}^2]^{1/2}$ vanishes. Combining \Cref{eq:first_bound} and \Cref{eq:second_bound}, we arrive at
\begin{align*}
     \mathcal W_2(\mu_{\infty},\mu^{\Phi}_{Kh}) \leq \rho(h)^K\mathcal W_2(\mu_{\infty}, \mu_{0}) +\begin{cases}
        \frac{7\sqrt 2}{6} \frac M m \sqrt{hd} + \frac{1-(1-mh)^K}{m}\varepsilon, \qquad &h\in\left( 0,\frac{2}{m+M} \right] \\
        \frac{7\sqrt 2}{6} \frac{Mh}{2-Mh} \sqrt{hd} + \frac{1-(Mh-1)^K}{2-Mh}h\varepsilon, \qquad &h\in\left[ \frac{2}{m+M}, \frac{2}{M}\right)
    \end{cases}.
\end{align*}
\end{proof}

\subsubsection{Auxiliary results}

\begin{lemma}[{\cite[Lemma~1]{dalalyan2017stronger}}]\label{lem:bound_delta}
    Let $\Delta_i := \tilde X_{ih} - Y_{ih}$ and $u_i := \nabla V(\tilde X_{ih}) - \nabla V(Y_{ih})$. Let
    \begin{equation*}
        \gamma := \begin{cases}
            1-mh & \text{if}~h \leq 2/(m+M)\\
            Mh-1 & \text{if}~h \geq 2/(m+M)
        \end{cases},
    \end{equation*}
    which is in $(0,1)$ since $h \leq 2/M$ by assumption. It holds that
    \begin{equation*}
        \E[\|\Delta_i - hu_i\|_{\ell^2}^2]^{1/2} \leq \gamma \E[\|\Delta_i\|_{\ell^2}^2]^{1/2}.
    \end{equation*}
\end{lemma}
\subsection{Approximation of Langevin Monte Carlo under linear error growth constraints}

\subsubsection{Sub-Gaussianity of perturbed Langevin Monte Carlo}\label{sec:subGaussianLMC}

\begin{proof}[Proof of \cref{prop:subgaussian_LMC}]
    We bound the Lyapunov-function of the LMC algorithm for a generic step. To simplify notation, we write $X_k$ instead of $\tilde{X}_{kh}$ in this proof. Consider the process
\begin{align*}
    X_{k} = X_{k-1} - h\nabla V (X_{k-1}) + \xi_{k}, \quad X_0 \sim \mu_0. 
\end{align*}
Let $c = \max_{\rho\in\{ m,M \}}|1-\rho h|$ and note that $c < 1$. Hence, $x\mapsto x^c$ is a concave function and Jensen's inequality yields $\mathbb{E}[Z^c] \leq (\mathbb{E}[Z])^c$ for any random variable $Z$. We further assume without loss of generality that the unique minimizer of $V$ is given by $x^\ast = 0$. Then, the Lyapunov function satisfies
\begin{align*}
    \mathbb{E}_{\xi_k}[\mathcal{L}_\lambda(X_{k})] &= e^{h\lambda^2}  \mathcal{L}_\lambda(X_{k-1} - h\nabla V(X_{k-1}))  \\
    &\leq e^{h\lambda^2} \ell (c\lambda \| X_{k-1} \|_{\ell^2}) \leq  e^{h\lambda^2} \ell(\lambda \| X_{k-1} \|_{\ell^2})^c \\
    &= e^{h\lambda^2} \left( \mathcal{L}_\lambda(X_{k-1}) \right)^c,
\end{align*}
where we have used the behaviour of the Lyapunov function under Gaussian convolution (\cref{lem:gauss_lyap}), the contractivity of the gradient descent step (\cref{lem:contract}) and Jensen's inequality. Furthermore,
\begin{align*}
    \mathbb{E}_{\xi_k,\xi_{k-1}}[\mathcal{L}_\lambda(X_{k})] &\leq e^{h\lambda^2} \mathbb{E}_{\xi_{k-1}}\left[\left( \mathcal{L}_\lambda(X_{k-1}) \right)^c\right] \\
    &\leq e^{h\lambda^2} \left(\mathbb{E}_{\xi_{k-1}}\left[ \mathcal{L}_\lambda(X_{k-1}) \right]\right)^c
\end{align*}
by Jensen's inequality. Iteratively, this leads to
\begin{align*}
    \mathbb{E}_{\xi_k,\xi_{k-1},\ldots,\xi_1}[\mathcal{L}_\lambda(X_{k})] \leq e^{h\lambda^2\sum_{i=0}^{k-1}c^i }\left( \mathcal{L}_\lambda (X_0) \right)^{c^k} = e^{h\lambda^2\frac{1-c^k}{1-c} }\left( \mathcal{L}_\lambda (X_0) \right)^{c^k}.
\end{align*}
Finally, taking the expectation with respect to $X_0$ on both sides we find (again using Jensen's inequality) that
\begin{align*}
    \mathbb{E}_{X_k}[\mathcal{L}_\lambda(X_{k})] \leq e^{h\lambda^2\frac{1-c^k}{1-c} }\left( \mathbb{E}_{X_0}\left[\mathcal{L}_\lambda (X_0)\right] \right)^{c^k}.
\end{align*}
Since $X_0 \sim \mu_0$ is sub-Gaussian with variance proxy $\sigma_0^2$, this yields
\begin{align*}
    \mathbb{E}_{X_k}[\mathcal{L}_\lambda(X_{k})] \leq \exp\left(h\lambda^2\frac{1-c^k}{1-c}\right) \exp\left(\frac{\sigma_0^2\lambda^2c^k}{2}\right) = \exp\left(h\lambda^2\frac{1-c^k}{1-c}+\frac{\sigma_0^2\lambda^2c^k}{2}\right).
\end{align*}
By \cref{rem:lyapunov_subgauss} it follows that $X_k$ is a sub-Gaussian RV with variance proxy
\begin{align*}
    \sigma_k^2 = 2h\frac{1-c^k}{1-c}+\sigma_0^2c^k.
\end{align*}
\end{proof}

\begin{proof}[Proof of \cref{prop:subgaussian_nn_process}]
    We simplify notation and consider the process $Y_k = Y_{k-1} + h\phi_{k}(Y_{k-1}) + \xi_k$. A similar analysis as in the proof of \cref{prop:subgaussian_LMC} leads to 
\begin{align*}
    \mathbb{E}_{\xi_k}\left[ \mathcal{L}_\lambda (Y_k) \right] &= e^{h\lambda^2}  \mathcal{L}_\lambda(Y_{k-1} + h\phi_k(Y_{k-1})) \\
    &= e^{h\lambda^2}  \mathcal{L}_\lambda(Y_{k-1} - h\nabla V(Y_{k-1}) + h\nabla V (Y_{k-1}) + h\phi_k(Y_{k-1})) \\
    &= e^{h\lambda^2}  \ell(\lambda \|Y_{k-1} - h\nabla V(Y_{k-1}) + h\nabla V (Y_{k-1}) + h\phi_k(Y_{k-1})\|_{\ell^2}) \\
    &\leq e^{h\lambda^2}  \ell(\lambda \|Y_{k-1} - h\nabla V(Y_{k-1})\|_{\ell^2} + \lambda \|h\nabla V (Y_{k-1}) + h\phi_k(Y_{k-1})\|_{\ell^2}) \\
    &\leq e^{h\lambda^2} \ell(\lambda (c + hG) \|Y_{k-1}\|_{\ell^2} + \lambda h \delta),
\end{align*}
where we have used the behaviour of the Lyapunov function under Gaussian convolution (\cref{lem:gauss_lyap}), the fact that $\ell$ is monotonically increasing (\cref{lem:ell_monoton}), assumption \Cref{eq:jentzen_assump} and the contractivity of the gradient descent step (\cref{lem:contract}).
Now, note that $(c+hG) < 1$ since for $h < \frac{2}{m+M}$,\begin{align*}
    c + hG = 1-mh +hG < 1.
\end{align*}
The fact that $\ell$ is a log-convex function (\cref{lem:ell_logconvex}) together with the upper bound $\ell(z) \leq \cosh(z)$ (\cref{lem:lyap_bound}) and the fact that $\cosh(z) \leq e^{\frac{z^2}{2}}$ for all $z\in \mathbb{R}$ yields
\begin{equation*}
\begin{aligned}
    \mathbb{E}_{\xi_k}\left[ \mathcal{L}_\lambda (Y_k) \right]
    &\leq e^{h\lambda^2} \ell\left(\lambda \|Y_{k-1}\|\right)^{(c + hG)} \ell\left(\dfrac{\lambda h \delta}{1-(c + hG)}\right)^{1-(c + hG)} \\
    &\leq \exp\left(h\lambda^2 + \frac{\lambda^2h^2 \delta^2}{2(1-(c + hG)
    )} \right) \mathcal{L}_{\lambda}\left(  Y_{k-1}\right)^{(c + hG)}.
\end{aligned}
\end{equation*}
Jensen's inequality further yields
\begin{align*}
\mathbb{E}_{\xi_k}\mathbb{E}_{\xi_{k-1}}\left[ \mathcal{L}_\lambda (Y_k) \right] &\leq \exp\left(h\lambda^2 + \frac{\lambda^2h^2 \delta^2}{2(1-(c + hG)
    )} \right) \left(\mathbb{E}_{\xi_{k-1}} [\mathcal{L}_{\lambda}\left(  Y_{k-1}\right)]\right)^{(c + hG)}.
\end{align*}
Iteratively it holds that
\begin{equation*}
\begin{aligned}
    &\mathbb{E}_{\xi_k,\xi_{k-1},\ldots, \xi_1}\left[ \mathcal{L}_\lambda (Y_k) \right] \\
    &\leq \exp\left(h\lambda^2 + \frac{\lambda^2h^2 \delta^2}{2(1-(c + hG)
    )} \right)^{\sum_{i=0}^{k-1}(c + hG)^i} \left( \mathcal{L}_{\lambda}\left(  Y_{0}\right)\right)^{(c + hG)^k} \\
    &= \exp\left(h\lambda^2 + \frac{\lambda^2h^2 \delta^2}{2(1-(c + hG)
    )} \right)^{\frac{1-(c + hG)^k}{1-(c + hG)}} \left( \mathcal{L}_{\lambda}\left(  Y_{0}\right)\right)^{(c + hG)^k}.
\end{aligned}
\end{equation*}
Finally, taking the expectation with respect to $Y_0\sim \mu_0$, which is sub-Gaussian with variance proxy $\sigma_0^2$ and Jensen's inequality lead to
\begin{equation*}
\begin{aligned}
    \mathbb{E}_{Y_{k}}\left[ \mathcal{L}_\lambda (Y_k) \right] &\leq \exp\left(h\lambda^2 + \frac{\lambda^2h^2 \delta^2}{2(1-(c + hG)
    )} \right)^{\frac{1-(c + hG)^k}{1-(c + hG)}} \left( \mathbb{E}_{Y_0}[\mathcal{L}_{\lambda}\left(  Y_{0}\right)]\right)^{(c + hG)^k} \\
    &\leq \exp\left(\frac{1-(c + hG)^k}{1-(c + hG)}\left( 
h\lambda^2 + \dfrac{\lambda^2h^2\delta^2}{2(1-(c + hG))} \right)+\frac{\sigma_0^2\lambda^2[c+hG]^k}{2}\right).
\end{aligned}
\end{equation*}
 By \cref{rem:lyapunov_subgauss}, $Y_k$ is sub-Gaussian with variance proxy
\begin{equation}
    \sigma_k^2 = \left(\frac{2h}{1-(c + hG)}+h^2\delta^2\right)[1-(c + hG)^k]+\sigma_0^2[c+hG]^k.
\end{equation}
As a consistency check, note that for $k\rightarrow\infty$ we have
\begin{equation*}
    \sigma_k^2 \longrightarrow \dfrac{2h}{1-c-hG}+ h^2 \delta^2,
\end{equation*}
which in the case of ``perfect aproximation'' with $\delta=0$ and $G=0$ leads to the known formula of $\frac{2h}{1-c}$ for the variance proxy of the invariant measure of LMC. 

Since $c+hG < 1$, the sequence of sub-Gaussian proxies is bounded by
\begin{align}
    \sigma^2_k \leq \underbrace{\left(\frac{2h}{1-(c + hG)}+h^2\delta^2\right)}_{\textnormal{proxy of target dist.}}+\underbrace{\sigma_0^2}_{\textnormal{proxy of initial dist.}}.
\end{align}

Sub-Gaussianity for all $t$ can now be shown in exactly the same way as in \cref{remark:subgaussian_LMC} for the standard LMC process.
Recall that the process $Y$ can be written for all $t \in [0,Kh]\setminus \N h$ as
\begin{equation*}
    Y_t = Y_{\chi_h(t)} + h\phi_{\chi_h(t)+1}(Y_{\chi_h(t)}) + \sqrt 2 (W_t - W_{\chi_h(t)}).
\end{equation*}
Therefore, since $Y_{kh}$ is sub-Gaussian for all $k \in \N$, then by linear interpolation, $Y_t$ is sub-Gaussian for all $t \in [0, Kh]$. Indeed, applying \cref{lem:gauss_lyap} with $\sigma^2 = 2(t-\chi_h(t))$ leads to the variance proxy
\begin{equation*}
    \sigma_t^2 = 2(t - \chi_h(t)) + \left(\frac{2h}{1-(c + hG)}+h^2\delta^2\right)[1-(c + hG)^{\chi_h(t)/h}]+\sigma_0^2[c+hG]^{\chi_h(t)/h}
\end{equation*}
for $Y_t$.
\end{proof}

\subsubsection{Neural network driven LMC with approximate drift with global linear error growth}\label{sec:proofsGlobalApprox}

\begin{proof}[Proof of \cref{thm:lmc_approx_ii}]
    \cref{assump:nn_approx_b} guarantees for any $\delta_0 > 0$ the existence of a neural network $\phi_{\delta_0}$ with $N(d,\delta_0,m,M)$ parameters such that
    \begin{align*}
        \| -\nabla V -\phi_{\delta_0} \|_{\ell^2} \leq \delta_0 (1+\|x\|_{\ell^2})
    \end{align*}
    for all $x\in\mathbb{R}^d$. Let $\phi \coloneqq \phi_{\delta}$, where 
    \begin{align*}
         \delta = \varepsilon \left( 1 + \left(2\pi d\sqrt{\dfrac{2}{m}} 
+ \dfrac{4d^2}{m} \right)\left[ 4 + \frac{64}{m(m+M)^2} + m\sigma_0^2
 \right]\right)^{-1/2}.
    \end{align*}
    Let $\Phi \coloneqq \{ \phi \}_{k=0}^{K-1}$ and $Y^{\Phi}\colon \Omega \times [0,Kh]\rightarrow\mathbb{R}^d$ be the stochastic process driven by $\Phi$, i.e.,
\begin{equation*}
    Y^{\Phi}_t = Y_0 + \int_{0}^t \mathcal{R}\phi_{\frac{1}{h}\chi_h(s)}\left(Y^{\Phi}_{\chi_h(s)}\right) \mathrm{d}s + \sqrt{2}W_t.
\end{equation*}
By \cref{prop:subgaussian_nn_process}, $Y^{\Phi}_{kh} \sim \mu^{\Phi}_{kh}$ is sub-Gaussian for all $k=0,\ldots,K$ with variance proxy $\sigma_k^2$ bounded by
\begin{align*}
    \sigma_k^2 \leq \frac{2h}{1-(c + hG)}+h^2\delta^2 +\sigma_0^2.
\end{align*}
We have 
    \begin{equation}
        \begin{aligned}
            \|-\nabla V - \phi_{k+1}\|^2_{L_{\mu^{\Phi}_{kh}}^2(\mathbb{R}^d;\mathbb{R}^d)} &= \int_{\mathbb{R}^d} \|-\nabla V(x) - \phi_{k+1}(x)\|^2_{\ell^2} \mathrm{d}\mu^{\Phi}_{kh}(x) \\
            &\leq \int_{\mathbb{R}^d} \delta^2  + 2\delta^2 \| x \|_{\ell^2} + \delta^2 \|x\|_{\ell^2}^2~\mathrm{d}\mu^{\Phi}_{kh}(x) \\
            &= \delta^2 + 2\delta^2 \mathbb{E}_{Y^{\Phi}_{kh}}[\|Y^{\Phi}_{kh}\|_{\ell^2}] + \delta^2 \mathbb{E}_{Y^{\Phi}_{kh}}[\|Y^{\Phi}_{kh}\|_{\ell^2}^2].
        \end{aligned}
    \end{equation}
    Furthermore, by \cref{prop:lp_subgaussian}, it holds for all $q\in\mathbb{N}$ that
    $\E_{Y^{\Phi}_{kh}}[\|Y^{\Phi}_{kh}\|_{\ell^2}^{q}] \leq (\sqrt 2 d\sigma_k)^q q \Gamma(q/2)$, leading to
    \begin{align*}
        \|-\nabla V - \phi_{k+1}\|^2_{L_{\mu^{\Phi}_{kh}}^2(\mathbb{R}^d;\mathbb{R}^d)} \leq \delta^2 + 2\delta^2 \sqrt{2}d\sigma_k\Gamma(1/2) + 4\delta^2 d^2\sigma_k^2\Gamma(1).
    \end{align*}
    A simple calculation shows that this term is bounded from above by $\varepsilon^2$ for all $k$  with the chosen $\delta$ (see \cref{lem:eps_error_linear_growth}). Applying \cref{thm:nn_approx_lmc} for $h< \frac{2}{m+M}$, we get
    \begin{equation*}
        \mathcal W_2(\mu_{\infty},\mu_{Kh}^{\Phi}) \leq (1-mh)^K\mathcal W_2(\mu_{\infty}, \mu_{0}) + \frac{7\sqrt 2}{6} \frac M m \sqrt{hd} + \dfrac{1-(1-mh)^K}{m}\varepsilon.
    \end{equation*}
    Finally, \cref{prop:resnet_equals_driven_process} guarantees the existence of a network $\psi$ with number of parameters equal to the number of parameters of $\phi$ such that for $\Psi \coloneqq \{\psi\}_{k=0}^{K-1}$ it holds that $\mu^{\Phi}_{Kh} = \mu^{\Psi}$, where $\xi = (\xi_1,\ldots,\xi_{K})$ and $\mathcal{R}\Psi(Y_0,\xi) \sim \mu^{\Psi}$. This yields the claim.\qedhere    
\end{proof}

\subsubsection{Auxiliray results: Lyapunov functions and contractivity}

\begin{lemma}[Monotonicity of $\ell$, {\cite[Lemma 3.4]{altschuler2022concentration}}]\label{lem:ell_monoton}
    For any $d\in\mathbb{N}$ the function $\ell$ in \cref{def:lyapunov_function} is monotonically increasing on $\mathbb{R}_{\geq 0}$.
\end{lemma}

\begin{lemma}[Log-convexity of $\ell$]\label{lem:ell_logconvex}
    For any $d\in\mathbb{N}$ the function $\ell$ is log-convex.
\end{lemma}
\begin{proof}
    By H\"older's inequality it holds that
$$\mathbb{E}(UV)\leq (\mathbb{E}|U|^p)^{1/p}(\mathbb{E}|V|^q)^{1/q}$$
for any $1<p,q<\infty$
 with $1/p+1/q=1$. Now, for $\theta\in(0,1)$ let $U = \exp(z_1\theta\langle v, e_1 \rangle )$, $V = \exp(z_2(1-\theta)\langle v, e_1 \rangle )$, $p = 1/\theta$, $q = 1/(1-\theta)$.
 Then 
 \begin{align}
     \ell(\theta z_0 + (1-\theta)z_1) &= \mathbb{E}_{v\sim \mathbb{S}^{2}_1(0)}\left[ 
 e^{(\theta z_0 + (1-\theta)z_1) \langle v,e_1 \rangle} \right] \\
 &\leq \left(\mathbb{E}_{v\sim \mathbb{S}^{2}_1(0)}\left[ 
 e^{z_0\langle v,e_1 \rangle} \right] \right)^{\theta} \left(\mathbb{E}_{v\sim \mathbb{S}^{2}_1(0)}\left[ 
 e^{z_1\langle v,e_1 \rangle} \right] \right)^{1-\theta} \\
 &= \ell(z_0)^{\theta} \ell(z_1)^{1-\theta}.
 \end{align}
 Taking the logarithm on both sides yields
 \begin{equation}
     \log\ell(\theta z_0 + (1-\theta)z_1) \leq \theta\log(\ell(z_0)) + (1-\theta)\log(\ell(z_1)),
 \end{equation}
 showing log-convexity.
\end{proof}

\begin{lemma}[Behavior of the Lyapunov function under Gaussian convolution, see {\cite[Lemma 3.3]{altschuler2022concentration}}]\label{lem:gauss_lyap}
    For any dimension $d\in\mathbb{N}$, point $x\in\mathbb{R}^d$, weight $\lambda > 0$ and noise variance $\sigma^2$ it holds that
    \begin{equation}
        \mathbb{E}_{Z\sim\mathcal{N}(0,\sigma^2I)}\left[ \mathcal{L}_\lambda(x+Z) \right] = e^{\frac{\lambda^2\sigma^2}{2}}\mathcal{L}_\lambda (x).
    \end{equation}
\end{lemma}

In order to derive upper and lower bounds on the Lyapunov function, an explicit formula is useful.

\begin{lemma}[Explicit formula for the Lyapunov function, see {\cite[Lemma 3.2]{altschuler2022concentration}}]\label{lem:exp_lyap}
    For any dimensions $d\geq 2$ and argument $z > 0$ it holds that
    \begin{equation}
        \ell(z) = \Gamma(\alpha + 1)\cdot \left( 
\dfrac{2}{z} \right)^\alpha \cdot I_\alpha (z),
    \end{equation}
    where $\alpha = (d-2)/d$, $\Gamma$ is the Gamma function and $I_n$ is the modified Bessel function of the first kind. For $d=1$, it holds that $\ell(z) = \frac{1}{2}(e^{-z}+e^z) = \cosh(z)$.
\end{lemma}

The following bounds are shown in \cite{LUKE197241}.

\begin{lemma}[Lower and upper bound of the Lyapunov function \cite{LUKE197241}]\label{lem:lyap_bound}
    For $z > 0$ and $\alpha > -\frac{1}{2}$ it holds that
    \begin{equation}
        1 < \Gamma(\alpha + 1)\cdot \left( 
\dfrac{2}{z} \right)^\alpha \cdot I_\alpha (z) < \cosh(z).
    \end{equation}
    In particular it holds for all $d$ and $z > 0$ that
    \begin{equation}
        1 \leq \ell(z) \leq \cosh(z).
    \end{equation}
\end{lemma}

\begin{lemma}[Contractivity of gradient descent step, {\cite[Lemma 4.2]{altschuler2022concentration}}]\label{lem:contract}
Suppose $V$ is $m$-strongly convex and has $M$-Lipschitz gradient and let $h\in(0,\frac{2}{M})$. Then it holds for all $x\in\mathbb{R}^d$ that
\begin{equation}
    \| x - h\nabla V (x) - x^\ast \|_{\ell^2} \leq c\| x-x^\ast \|_{\ell^2},
\end{equation}
where $x^\ast$ is any minimizer of $V$ and $c \coloneqq \max_{\rho\in\{m,M\}}|1-\rho h| < 1$.
\end{lemma}

\begin{lemma}[Achieving $\varepsilon$-error with linear error growth assumption]\label{lem:eps_error_linear_growth}
Let $\varepsilon > 0$. Let $\delta \leq \frac{m}{2}$ satisfy
\begin{equation*}
   \delta \leq  \varepsilon \left( 1 + \left(2\pi d\sqrt{\dfrac{2}{m}} 
+ \dfrac{4d^2}{m} \right)\left[ 4 + \frac{64}{m(m+M)^2} + m\sigma_0^2
 \right]\right)^{-\frac{1}{2}}.
\end{equation*}
    Let $\phi\colon \mathbb{R}^d \rightarrow \mathbb{R}^d$ satisfy
    \begin{equation}
        \|-\nabla V(x) - \phi(x)\|_{\ell^2} \leq \delta (1 + \| x \|_{\ell^2})
    \end{equation}
    for all $x\in\mathbb{R}^d$. Let $h\in (0,\frac{2}{m+M})$ and $\Phi = \{\phi_k=\phi\}_{k\in\mathbb{N}}$. Then, the stochastic process $Y^{\Phi}$ driven by $\Phi$, given by \cref{eq:driven_process}, with $Y^{\Phi}_0 \sim \mu_0$ has intermediate measures $Y^{\Phi}_{ih}\sim \mu^{\Phi}_{ih}$ satisfying for all $i\in\mathbb{N}$ that
    \begin{equation}
         \|-\nabla V - \phi_{i+1}\|_{L_{\mu^{\Phi}_{ih}}^2(\mathbb{R}^d;\mathbb{R}^d)} \leq \varepsilon.
    \end{equation}
\end{lemma}

\begin{proof}
   We have 
    \begin{equation}
        \begin{aligned}
            \|-\nabla V - \phi_{i+1}\|^2_{L_{\mu^{\Phi}_{ih}}^2(\mathbb{R}^d;\mathbb{R}^d)} &= \int_{\mathbb{R}^d} \|-\nabla V(x) - \phi_{i+1}(x)\|^2_{\ell^2} \mathrm{d}\mu^{\Phi}_{ih}(x) \\
            &\leq \int_{\mathbb{R}^d} \delta^2  + 2\delta^2 \| x \|_{\ell^2} + \delta^2 \|x\|_{\ell^2}^2~\mathrm{d}\mu^{\Phi}_{ih}(x) \\
            &= \delta^2 + 2\delta^2 \mathbb{E}_{X\sim \mu^{\Phi}_{ih}}[\|X\|_{\ell^2}] + \delta^2 \mathbb{E}_{X\sim \mu^{\Phi}_{ih}}[\|X\|_{\ell^2}^2].
        \end{aligned}
    \end{equation}
    Now, by \cref{prop:subgaussian_nn_process} it holds that $\mu^{\Phi}_{ih}$ is sub-Gaussian with variance proxy 
    \begin{equation}
        \sigma_i \leq \frac{2h}{1-(c + h\delta)}+h^2\delta^2 + \sigma_0^2.
    \end{equation}
    Furthermore, by \cref{prop:lp_subgaussian}, it holds for all $q\in\mathbb{N}$ that
    $\E_{X\sim \mu^{\Phi}_{ih}}[\|X\|_{\ell^2}^{q}] \leq (\sqrt 2 d\sigma_i)^q q \Gamma(q/2)$. Hence,
    \begin{equation}
    \begin{aligned}
        &\|-\nabla V - \phi_{i+1}\|^2_{L_{\mu^{\Phi}_{ih}}^2(\mathbb{R}^d;\mathbb{R}^d)} \leq \delta^2 + 2\delta^2 \sqrt{2}d\sigma_i\Gamma(1/2) + \delta^2 4d^2\sigma_i^2\Gamma(1) \\ 
        &\leq \delta^2 + 2\delta^2 \sqrt{2}d\left[ \frac{2h}{1-(c + h\delta)}+h^2\delta^2 + \sigma_0^2
         \right]^{\frac{1}{2}}\Gamma(1/2) \\
         &~~~+ \delta^2 4d^2\left[ \frac{2h}{1-(c + h\delta)}+h^2\delta^2 + \sigma_0^2
         \right]\Gamma(1) \\
         &= \delta^2 + 2\delta^2 \sqrt{2}d\left[ \frac{2h}{1-(c + h\delta)}+h^2\delta^2 + \sigma_0^2
         \right]^{\frac{1}{2}}\sqrt{\pi} + \delta^2 4d^2\left[ \left(\frac{2h}{1-(c + h\delta)}+h^2\delta^2\right) + \sigma_0^2
         \right].
    \end{aligned}
    \end{equation}
   Note that $h < \frac{2}{m+M}$ implies $c = 1 - mh$. Since $\delta < \frac{m}{2}$, it holds that
   \begin{align*}
       \frac{2h}{1-(c + h\delta)} = \dfrac{2}{m-\delta} < \dfrac{2}{m-\frac{m}{2}} = \dfrac{4}{m}.
   \end{align*}
   Hence,
   \ifnum\classstyle=1
   \begin{equation}
       \begin{aligned}
           \|-\nabla V - \phi_{i+1}\|^2_{L_{\mu^{\Phi}_{ih}}^2(\mathbb{R}^d;\mathbb{R}^d)} &\leq \delta^2 + 2\delta^2 \sqrt{\dfrac{2\pi}{m}}d\left[ 4 + mh^2\delta^2 + m\sigma_0^2 \right]^{\frac{1}{2}}\sqrt{\pi} + \delta^2 \dfrac{4}{m}d^2\left[ 4 + mh^2\delta^2 + m\sigma_0^2 \right] \\
            &\leq  \delta^2 + 2\delta^2d \sqrt{\dfrac{2\pi}{m}}\left[ 4 + 16h^2/m + m\sigma_0^2 \right]^{\frac{1}{2}}\sqrt{\pi} + \delta^2 \dfrac{4}{m}d^2\left[ 4 + 16h^2/m + m\sigma_0^2 \right] \\
            &= \delta^2\left(1 + 2\pi d\sqrt{\dfrac{2}{m}}\left[ 4 + 16h^2/m + m\sigma_0^2 \right]^{\frac{1}{2}} + \dfrac{4d^2}{m}\left[ 4 + 16h^2/m + m\sigma_0^2 \right]\right) \\
            &\leq \delta^2 \left( 1 + \left(2\pi d\sqrt{\dfrac{2}{m}}+ \dfrac{4d^2}{m} \right)\left[ 4 + 16h^2/m + m\sigma_0^2 \right] \right) \\
            &\leq \delta^2 \left( 1 + \left(2\pi d\sqrt{\dfrac{2}{m}} + \dfrac{4d^2}{m} \right)\left[ 4 + \frac{64}{m(m+M)^2} + m\sigma_0^2\right]\right) \\
            &\leq \varepsilon^2,
       \end{aligned}
   \end{equation}
   \fi
   \ifnum\classstyle=0
   \begin{equation}
       \begin{aligned}
           \|-\nabla V - &\phi_{i+1}\|^2_{L_{\mu^{\Phi}_{ih}}^2(\mathbb{R}^d;\mathbb{R}^d)}\\
           &\leq \delta^2 + 2\delta^2 \sqrt{\dfrac{2\pi}{m}}d\left[ 4 + mh^2\delta^2 + m\sigma_0^2 \right]^{\frac{1}{2}}\sqrt{\pi} + \delta^2 \dfrac{4}{m}d^2\left[ 4 + mh^2\delta^2 + m\sigma_0^2 \right] \\
            &\leq  \delta^2 + 2\delta^2d \sqrt{\dfrac{2\pi}{m}}\left[ 4 + 16h^2/m + m\sigma_0^2 \right]^{\frac{1}{2}}\sqrt{\pi} + \delta^2 \dfrac{4}{m}d^2\left[ 4 + 16h^2/m + m\sigma_0^2 \right] \\
            &= \delta^2\left(1 + 2\pi d\sqrt{\dfrac{2}{m}}\left[ 4 + 16h^2/m + m\sigma_0^2 \right]^{\frac{1}{2}} + \dfrac{4d^2}{m}\left[ 4 + 16h^2/m + m\sigma_0^2 \right]\right) \\
            &\leq \delta^2 \left( 1 + \left(2\pi d\sqrt{\dfrac{2}{m}}+ \dfrac{4d^2}{m} \right)\left[ 4 + 16h^2/m + m\sigma_0^2 \right] \right) \\
            &\leq \delta^2 \left( 1 + \left(2\pi d\sqrt{\dfrac{2}{m}} + \dfrac{4d^2}{m} \right)\left[ 4 + \frac{64}{m(m+M)^2} + m\sigma_0^2\right]\right) \\
            &\leq \varepsilon^2,
       \end{aligned}
   \end{equation}
   \fi
with the choice of $\delta$ in the lemma.
\end{proof}

\subsection{Approximation of Langevin Monte Carlo under local error- and Lipschitz constraints}\label{sec:proofsLocalApprox}

\begin{proof}[Proof of \cref{thm:linear_approx}]
Since the minimizer of $V$ is given by $0$, we have that $\nabla V(0)=0$. The strong convexity then implies for all $x,z\in\mathbb{R}^d$ that
    \begin{align*}
        V(x) \geq m/2 \|x\|_{\ell^2}^2 + V(0) 
    \end{align*}
    and
    \begin{align*}
         -\langle \nabla V(x), z-x \rangle \geq m/2 \| z-x \|_{\ell^2}^2 - V(z) + V(x).
    \end{align*}
    We get for $z=0$ that
    \begin{align*}
        \langle \nabla V(x), x \rangle \geq m/2 \| x \|_{\ell^2}^2  - V (0) + V(x)
    \end{align*}
    and hence
    \begin{align*}
        \langle \nabla V(x), x \rangle \geq m \| x \|_{\ell^2}^2
    \end{align*}
    for all $x\in\mathbb{R}^d$.
 With the Lipschitz continuity, we get for any $a>0$ that
    \begin{align*}
        \| \nabla V (x) - ax \|_{\ell^2}^2 &= \| \nabla V(x) \|_{\ell^2}^2 - 2a \langle \nabla V(x), x \rangle + a^2 \| x \|_{\ell^2}^2\\
        &\leq (M^2 + a^2) \| x\|^2 - 2a m \| x\|^2\\ 
        &= (M^2 + a^2 - 2am) \| x\|^2. 
    \end{align*}
    The inequality $g \coloneqq \sqrt{M^2 + a^2 - 2am} < m$ is fulfilled for $a = m$, if $0<M<\sqrt{2}m$. In this case, $g = \sqrt{M^2 - m^2}$.
\end{proof}

\begin{proof}[Proof of \cref{thm: bounded network}]
    We construct a bounded neural network $\phi_L$, which does not change the approximation on $\Omega$. In the proof, a network $\phi$ and its realization $\mathcal{R}\phi$ are denoted by $\phi$ to avoid overloading notation.
    
    Let $c_i \coloneqq \|\nabla V_i\|_{L^\infty(\Omega)} \geq 0$.
    Define the neural networks $\phi_{L-1}(x) = -\sigma( - \phi_{L-2}(x) + c) + c$ and $\phi_L(x) = \sigma(\phi_{L-1}(x) + c) - c$.
    We show that $\|\phi_L(x)_i\|_{L^\infty(\Omega)}\leq c_i$ for all $x\in\mathbb{R}^d$ by contradiction. 
    \begin{itemize}
        \item First, assume that there exists $i\in[d]$ and $x\in\mathbb{R}^d$ such that $\phi_L(x)_i > c_i$. Then $\sigma(\phi_{L-1}(x)_i + c_i) > 2c_i \geq 0$ implying $\phi_{L-1}(x)_i + c_i = \sigma(\phi_{L-1}(x)_i + c_i) > 2c_i$.
        Hence, $\phi_{L-1}(x)_i > c_i$ 
        and it follows that {$-\sigma(-\phi_{L-2}(x)_i + c_i) > 0$}.
        This is a contradiction that $\sigma$ is non-negative. Therefore, $\phi_L(x)_i \leq c_i$ holds true.
        \item Second, assume that there exist $i\in [d]$ and $x\in\mathbb{R}^d$ such that $\phi_L(x) < -c_i$. Then, $\sigma(\phi_{L-1}(x) + c_i)< 0$. This again contradicts that $\sigma$ is non-negative. 
    \end{itemize}
    Therefore, we have that $-c_i \leq \phi_L(x)_i \leq c_i$ for all $i\in[d]$ and $x\in\mathbb{R}^d$ implying $\|\phi_L(x)_i\|_{L^\infty(\mathbb{R}^d)} \leq c_i$.

    Furthermore, we show that the approximation on the domain $\Omega$ does not get worse. We consider the following cases.
    \begin{itemize}
        \item Let $x\in \Omega$ and $i\in [d]$ such that $\phi_{L-2}(x)_i > c_i$. Then $\phi_{L-1}(x)_i= c_i$ and $\phi_L(x)_i = c_i$. 
        We observe that $-\nabla V(x)_i\leq c_i = \phi_L(x)_i < \phi_{L-2}(x)_i$ implies $|-\nabla V(x)_i - \phi_L(x)_i| = \phi_L(x)_i + \nabla  V(x)_i < \phi_{L-2}(x)_i + \nabla  V(x)_i = |-\nabla  V(x)_i - \phi_{L-2}(x)_i|$.
        \item Let $x\in \Omega, i\in [d]$ such that $\phi_{L-2}(x)_i< -c_i$. Then $\phi_{L-1}(x)_i < -c_i$ and $\phi_L(x)_i = - c_i$. 
        With $-\nabla V(x)_i \geq -c_i =\phi_L(x)_i > \phi_{L-2}(x)_i$, we observe that $|-\nabla V(x)_i - \phi_L(x)_i | = -\nabla V(x)_i - \phi_L(x)_i \leq -\nabla V(x)_i - \phi_{L-2}(x)_i = |-\nabla V(x)_i - \phi_{L-2}(x)_i | $.
        \item Let $x\in \Omega, i\in[d]$ such that $-c_i \leq \phi_{L-2}(x)_i \leq c_i$. Then $\phi_{L-1}(x)_i = \phi_{L-2}(x)_i$ and $\phi_{L}(x)_i = \phi_{L-2}(x)_i$. Therefore by assumption $| -\nabla  V(x)_i - \phi_L(x)_i | = | - \nabla  V(x)_i - \phi_{L-2}(x)_i| $.
    \end{itemize}
    Hence, for all $x\in \Omega$ we get $| -\nabla V(x)_i - \phi_L(x)_i | \leq |-\nabla  V(x)_i - \phi_{L-2}(x)_i|$. 
    We have 
    \ifnum\classstyle=1
    \begin{align*}
        \| -\nabla  V - \phi_L \|_{L^p_\mu(\mathbb{R}^d, \mathbb{R}^d)}^p 
        &= \|-\nabla  V-\phi_L\|^p_{L^p_\mu (\Omega;\mathbb{R}^d)} 
        + \|-\nabla  V-\phi_L\|^p_{L^p_\mu(\mathbb{R}^d \backslash \Omega;\mathbb{R}^d)}\\
        &= \|-\nabla  V-\phi_L\|^p_{L^p_\mu (\Omega;\mathbb{R}^d)} + \int_{\mathbb{R}^d \backslash \Omega} \| -\nabla  V-\phi_L\|_{\ell^p}^p \mathrm{d}x\mu\\
        &= \|-\nabla  V-\phi_L\|^p_{L^p_\mu (\Omega;\mathbb{R}^d)} + 2^{p-1} \int_{\mathbb{R}^d \backslash \Omega} \| -\nabla  V \|_{\ell^p}^p + \|\phi_L\|_{\ell^p}^p \mathrm{d}\mu\\
        &\leq \|-\nabla  V - \phi_{L-2}\|^p_{L^p_\mu(\Omega; \R^d)} + 2^{p-1}( \|\nabla V\|_{L^p_\mu(\R^d \setminus \Omega; \R^d)}^p + \|c\|_{L^p_\mu(\R^d \setminus \Omega; \R^d)}^p).
    \end{align*}
    \fi
    \ifnum\classstyle=0
    {\footnotesize
    \begin{align*}
        \| -\nabla  V - \phi_L \|_{L^p_\mu(\mathbb{R}^d, \mathbb{R}^d)}^p 
        &= \|-\nabla  V-\phi_L\|^p_{L^p_\mu (\Omega;\mathbb{R}^d)} 
        + \|-\nabla  V-\phi_L\|^p_{L^p_\mu(\mathbb{R}^d \backslash \Omega;\mathbb{R}^d)}\\
        &= \|-\nabla  V-\phi_L\|^p_{L^p_\mu (\Omega;\mathbb{R}^d)} + \int_{\mathbb{R}^d \backslash \Omega} \| -\nabla  V-\phi_L\|_{\ell^p}^p d\mu\\
        &= \|-\nabla  V-\phi_L\|^p_{L^p_\mu (\Omega;\mathbb{R}^d)} + 2^{p-1} \int_{\mathbb{R}^d \backslash \Omega} \| -\nabla  V \|_{\ell^p}^p + \|\phi_L\|_{\ell^p}^p d\mu\\
        &\leq \|-\nabla  V - \phi_{L-2}\|^p_{L^p_\mu(\Omega; \R^d)} + 2^{p-1}( \|\nabla V\|_{L^p_\mu(\R^d \setminus \Omega; \R^d)}^p + \|c\|_{L^p_\mu(\R^d \setminus \Omega; \R^d)}^p)
    \end{align*}
    }
    \fi
\end{proof}

\begin{proof}[Proof of \cref{prop:combined_assump}]
We assume without loss of generality that the unique minimizer of $V$ is given by $x^\ast = 0$ and hence $\nabla V (0) = 0$. Throughout this proof, we overload notation by writing $\phi$ instead of $\mathcal{R}\phi$ to denote the realization of a network $\phi$. With \cref{assump:nn_approx_a} there exists for any $\varepsilon, r > 0$ a neural network $\phi^{(3)}_r$
    with number of parameters bounded by $N(d,r,\sqrt{2}\varepsilon/\sqrt{d}, m,M)$ such that
    \begin{align*}
       \| -\nabla V - \mathcal{R}\phi^{(3)}_r
        \|_{L^\infty(B^1_r(0))}\leq  \| -\nabla V - \mathcal{R}\phi^{(3)}_r
        \|_{L^\infty(B^2_r(0))} \leq \varepsilon/\sqrt{d}.
    \end{align*}
By \cref{prop:represent_identity}, we can represent $x \in \R^d \mapsto x$ by a neural network with exactly $4d$ parameters and depth $2$.
    For summing neural networks, we use \cref{prop:sum_nn}. We construct a network $\phi^{(2)}_r$ by adding to  
    $\phi^{(3)}_r$
     a neural network representing the $d$-dimensional identity, whose last layer is multiplied with $-m$.
Then, $\mathcal{R}\phi^{(2)}_r(x) = \mathcal{R}\phi^{(3)}_r
    (x) - mx$ holds and
    \begin{align*}
        \|-\nabla V - m\cdot - \mathcal{R}\phi^{(2)}_r\|_{L^\infty(B^1_r(0))} = \| -\nabla V - \mathcal{R}\phi^{(3)}_r
        \|_{L^{\infty}(B^1_r(0))} \leq \varepsilon/\sqrt{d}.
    \end{align*}
    To keep track of complexity, note that $\phi^{(2)}_r$ is the sum of a network of $4d$ parameters and depth $2$ and a network of $N(d,r,\sqrt{2}\varepsilon/\sqrt{d},m,M)$ parameters and depth $L(d,r,\sqrt{2}\varepsilon/\sqrt{d},m,M)$. Hence, by \cref{prop:sum_nn}, the number of parameters and the depth are bounded by
    \begin{align*}
        \mathcal{P}(\phi^{(2)}_r) &\leq d(L(d,r,\sqrt{2}\varepsilon/\sqrt{d},m,M) - 1) + N(d,r,\sqrt{2}\varepsilon/\sqrt{d},m,M) + 4d,\\
        \mathbb{L}(\phi_r^{(2)}) &\leq L(d,r,\sqrt{2}\varepsilon/\sqrt{d},m,M).
    \end{align*}
    Let $\phi^{(1)}_r$ be the cut-off NN of $\phi^{(2)}_r$ as in \cref{thm: bounded network} with the same accuracy on the $\ell^1$-ball and 
    \begin{align*}
     \max_{x\in \mathbb{R}^d} \| \mathcal{R}\phi_r^{(1)}(x)\|_{\ell^\infty} & \leq \max_{x\in B^1_r(0)} \| -\nabla V - m\cdot \|_{\ell^\infty} \leq \max_{x\in B^1_r(0)} \| -\nabla V - m\cdot \|_{\ell^2} \leq rG, \\
    \max_{x\in\mathbb{R}^d} \|\mathcal{R}\phi_1^{(1)}(x)\|_{\ell^2} &\leq \max_{x\in \mathbb{R}^d} \sqrt{d}\| \mathcal{R}\phi_r^{(1)}(x)\|_{\ell^\infty} 
   \leq \sqrt{d}rG,
    \end{align*}
    where we used \cref{thm:linear_approx} with $G\coloneqq \sqrt{M^2-m^2}$. Again, keeping track of complexity, the cutoff from \cref{thm: bounded network} introduces two more layers and $2d^2 + 2$ more weights to the network. Hence
    \begin{align*}
        \mathcal{P}(\phi^{(1)}_r) &\leq d(L(d,r,\sqrt{2}\varepsilon/\sqrt{d},m,M) - 1) + N(d,r,\sqrt{2}\varepsilon/\sqrt{d},m,M) + 4d + 2d^2 + 2, \\
        \mathbb{L}(\phi_r^{(1)}) &\leq L(d,r,\sqrt{2}\varepsilon/\sqrt{d},m,M) + 2.
    \end{align*}
    Now, define for $b > r$ the following approximation to the indicator function on $B^1_r(0)$,
    \begin{equation}
        f(x) = 1-\dfrac{\mathrm{ReLU}(\|x\|_{\ell^1}-r)-\mathrm{ReLU}(\|x\|_{\ell^1}-b)}{b-r},
    \end{equation}
    with complexity $\mathcal P(f) = 4d+7$ and $\mathbb L(f) = 3$ (for a proof of the complexity, see \cref{prop:nn_approx_indicator}),
    and note that
    \begin{equation}
        f(x)\phi_r^{(1)}(x) = \begin{cases}
            \phi_r^{(1)}(x) , &\|x\|_{\ell^1}\leq r\\
            (1 - \frac{\|x\|_{\ell^1} - r}{b-r})\phi_r^{(1)}(x), &r\leq \|x\|_{\ell^1}\leq b\\
            0, &\|x\|_{\ell^1}\geq b
        \end{cases}.
    \end{equation}
Now, noting that $f(x) \in [0,1]$ and $\phi^{(1)}_r(x) \in [-rG,rG]^d$ for all $x\in \mathbb{R}^d$, this product can be approximated by composing the parallelization of $f$ and $\phi_r^{(1)}$ with an approximation of the multiplication $(x,y)\mapsto xy$ on $\Omega \coloneqq [0,1]\times [-rG,rG]^d$, using \cref{prop:nn_elementwise_mult}.
In particular, we want to approximate the product up to $\varepsilon/\sqrt{d}$ error in $L^{\infty}(\Omega)$. We call the network that accomplishes this $\tilde{\phi}^{(0)}_r$. The complexity of $\tilde{\phi}^{(0)}_r$ is given by 
\begin{align*}
    \mathcal P(\tilde{\phi}^{(0)}_r) &= \mathcal O\left(d\log(2d\max\{ 
1,rG \}/\varepsilon) + \mathcal{P}(\phi_r^{(1)}) + \mathcal{P}(f)\right) \\
    &= \mathcal O\left(d\log(2d\max\{ 
1,rG \}/\varepsilon) + d(L(d,r,\sqrt{2}\varepsilon/\sqrt{d},m,M) - 1) \right.\\
&\quad\qquad\left.+ N(d,r,\sqrt{2}\varepsilon/\sqrt{d},m,M) + 8d + 2d^2 + 9\right),\\
    \mathbb L(\tilde{\phi}^{(0)}_r) &= \mathcal O\left(\log(2d\max\{ 
1,rG \}/\varepsilon) + \mathbb{L}(f)+ \mathbb{L}(\phi^{(1)}_r)\right) \\
    &= \mathcal O\left(\log(2d\max\{ 
1,rG \}/\varepsilon) + L(d,r,\sqrt{2}\varepsilon/\sqrt{d},m,M) + 5\right) 
\end{align*}
Finally, we define $\phi^{(0)}_{r}$ to be the network representing $\phi \coloneqq \tilde{\phi}^{(0)}_r + m\cdot$. This is an addition of $\tilde{\phi}_r^{(0)}$ with a network of $4d$ parameters and depth $2$. By the properties of neural network summation (\cref{prop:sum_nn}) and the properties of the ``big-O'' notation, the complexity of $\phi^{(0)}_{r}$ is given by
\ifnum\classstyle=1
\begin{align*}
    \mathcal{P}(\phi_r^{(0)}) &= d\left[\mathcal{O}\left(\log(2d\max\{ 
1,rG \}/\varepsilon) + L(d,r,\sqrt{2}\varepsilon/\sqrt{d},m,M) + 5\right)  - 1\right]\\
&~~~~+\mathcal{O}\left(d\log(2d\max\{ 
1,rG \}/\varepsilon) + d(L(d,r,\sqrt{2}\varepsilon/\sqrt{d},m,M) - 1)\right. \\
&\qquad\qquad+ \left. N(d,r,\sqrt{2}\varepsilon/\sqrt{d},m,M) + 12d + 2d^2 + 9\right)\\
&= \mathcal{O}\left( d\log(2d\max\{ 
1,rG \}/\varepsilon) + N(d,r,\sqrt{2}\varepsilon/\sqrt{d},m,M) + dL(d,r,\sqrt{2}\varepsilon/\sqrt{d},m,M) + 2d^2 \right).
\end{align*}
\fi
\ifnum\classstyle=0
{\footnotesize
\begin{align*}
    \mathcal{P}(\phi_r^{(0)}) &= d\left[\mathcal{O}\left(\log(2d\max\{ 
1,rG \}/\varepsilon) + L(d,r,\sqrt{2}\varepsilon/\sqrt{d},m,M) + 5\right)  - 1\right]\\
&~~~~+\mathcal{O}\left(d\log(2d\max\{ 
1,rG \}/\varepsilon) + d(L(d,r,\sqrt{2}\varepsilon/\sqrt{d},m,M) - 1)\right. \\
&\qquad\qquad+ \left. N(d,r,\sqrt{2}\varepsilon/\sqrt{d},m,M) + 12d + 2d^2 + 9\right)\\
&= \mathcal{O}\left( d\log(2d\max\{ 
1,rG \}/\varepsilon) + N(d,r,\sqrt{2}\varepsilon/\sqrt{d},m,M) + dL(d,r,\sqrt{2}\varepsilon/\sqrt{d},m,M) + 2d^2 \right).
\end{align*}
}
\fi
The goal now is to show that $\phi$ fulfills the condition
\begin{align*}
    \|-\nabla V(x) - \phi(x) \|_{\ell^2} \leq c\varepsilon + G\|x\|_{\ell^2}, \quad x\in \mathbb{R}^d,
\end{align*}
for some constant $c$ independent of $r$. We treat three cases separately. First, consider points $x$ inside the $\ell^1$-ball of radius $r$, i.e. $\|x\|_{\ell^1}\leq r$. In this case, we have
    \begin{align*}
        \|-\nabla V(x) - \phi(x) \|_{\ell^2} &= \|-\nabla V(x) - \tilde{\phi}^{(0)}_r(x) - mx \|_{\ell^2} \\
        &\leq \|-\nabla V(x) - f(x)\phi_r^{(1)}(x) - mx \|_{\ell^2} + \|f(x)\phi_r^{(1)}(x) - \tilde{\phi}^{(0)}_r(x)\|_{\ell^2} \\
        &\leq \|-\nabla V(x) - f(x)\phi_r^{(1)}(x) - mx \|_{\ell^2} + \sqrt{d}\|f(x)\phi_r^{(1)}(x) - \tilde{\phi}^{(0)}_r(x)\|_{\ell^\infty}
        \\
        &\leq \|-\nabla V(x) - f(x)\phi_r^{(1)}(x) - mx \|_{\ell^2} + \varepsilon
        \\
        &= \|-\nabla V(x) - \phi_r^{(1)}(x) - mx \|_{\ell^2} + \varepsilon
        \\
        &\leq \|-\nabla V(x) - \phi_r^{(2)}(x) - mx \|_{\ell^2} + \varepsilon
        \\
        &= \|-\nabla V(x) - \phi_r^{(3)}(x) \|_{\ell^2} + \varepsilon
        \\
        &\leq \varepsilon + \varepsilon 
        \\
        &\leq 2\varepsilon + G\|x\|_{\ell^2}.
    \end{align*}
Next, let $\|x\|_{\ell^1} \geq b$. In this case, we have
    \begin{align*}
        \|-\nabla V(x) - \phi(x) \|_{\ell^2} &= \|-\nabla V(x) - \tilde{\phi}^{(0)}_r(x) - mx \|_{\ell^2} \\
        &\leq \|-\nabla V(x) - f(x)\phi_r^{(1)}(x) - mx \|_{\ell^2} + \|f(x)\phi_r^{(1)}(x) - \tilde{\phi}^{(0)}_r(x)\|_{\ell^2} \\
        &\leq \|-\nabla V(x) - f(x)\phi_r^{(1)}(x) - mx \|_{\ell^2} + \sqrt{d}\|f(x)\phi_r^{(1)}(x) - \tilde{\phi}^{(0)}_r(x)\|_{\ell^\infty} 
        \\
        &\leq \|-\nabla V(x) - f(x)\phi_r^{(1)}(x) - mx \|_{\ell^2} + \varepsilon \\
        &\leq \|-\nabla V(x) - 0 - mx \|_{\ell^2} + \varepsilon
        \\
        &= \|-\nabla V(x) - mx \|_{\ell^2} + \varepsilon
        \\
        &\leq G\|x\|_{\ell^2} + \varepsilon.
    \end{align*}
    Finally, let $r \leq \|x\|_{\ell^1} \leq b$. We have not made a choice for $b$ yet. Now (and in every line of the proof before), we let 
    \begin{align}
       b &= r + \frac{\varepsilon}{\max\{L_{\phi_r^{(1)}},M\}}, \\ \hat{x} &= \frac{rx}{\|x\|_{\ell^1}}.
    \end{align}
    Note that $\|x-\hat{x}
    \|_{\ell^2} \leq b-r$. 
    Then, we have
    \begin{align*}
        &\|-\nabla V(x) - \phi(x) \|_{\ell^2} \\
        &\leq \| -\nabla V(x) + \nabla V(\hat{x}) \|_{\ell^2} + \| -\nabla V (\hat{x}) - \phi(\hat{x}) \|_{\ell^2} + \| \phi(\hat{x})-\phi(x) \|_{\ell^2}.
        \end{align*}
        We treat the terms of the triangle inequality separately. First, note that due to the Lipschitz continuity of $\nabla V$ and the definition of $b$ we have
        \begin{align*}
            \| -\nabla V(x) + \nabla V(\hat{x}) \|_{\ell^2} \leq M\|x-\hat{x}\|_{\ell^2} \leq \varepsilon.
        \end{align*}
        Consider next the second term. Since $\hat{x}$ is an element of $B_r^1(0)$ by construction, the approximation properties of the network guarantee that
        \begin{align*}
            \| -\nabla V (\hat{x}) - \phi(\hat{x}) \|_{\ell^2} \leq 2\varepsilon,
        \end{align*}
        which follows from case 1. The remaining term captures the growth of the network on the ``slope domain'' between $B_r^1(0)$ and $B_b^1(0)$, and can be bounded by using the Lipschitz continuity of $\phi^{(1)}_r$ and the definition of $b$. We have
        \begin{align*}
            \| \phi(\hat{x})-\phi(x) \|_{\ell^2} 
            &= \| \tilde{\phi}^{(0)}_r(\hat{x}) + m\hat{x}-\tilde{\phi}^{(0)}_r(x) -mx\|_{\ell^2} \\
            &\leq \left\| \phi^{(1)}_r(\hat{x}) + m\hat{x}-\left( 1-\frac{\|x\|_{\ell^1}-r}{b-r}\right)\phi_r^{(1)}(x) -mx\right\|_{\ell^2} \\ &~~~~~+ \| \tilde{\phi}^{(0)}_r(\hat{x}) - f(\hat{x})\phi_r^{(1)}(\hat{x}) \|_{\ell^2} + \| \tilde{\phi}^{(0)}_r(x) - f(x)\phi_r^{(1)}(x) \|_{\ell^2} \\
            &\leq \left\| \phi^{(1)}_r(\hat{x}) + m\hat{x}-\left( 1-\frac{\|x\|_{\ell^1}-r}{b-r}\right)\phi_r^{(1)}(x) -mx\right\|_{\ell^2} \\
            &~~~~~+ \sqrt{d}\| \tilde{\phi}^{(0)}_r(\hat{x}) - f(\hat{x})\phi_r^{(1)}(\hat{x}) \|_{\ell^\infty} + \sqrt{d}\| \tilde{\phi}^{(0)}_r(x) - f(x)\phi_r^{(1)}(x) \|_{\ell^\infty}
             \\
            &\leq \left\| \phi^{(1)}_r(\hat{x}) + m\hat{x}-\left( 1-\frac{\|x\|_{\ell^1}-r}{b-r}\right)\phi_r^{(1)}(x) -mx\right\|_{\ell^2} + 2\varepsilon \\
            &\leq \left\| \phi^{(1)}_r(\hat{x}) -\left( 1-\frac{\|x\|_{\ell^1}-r}{b-r}\right)\phi_r^{(1)}(x) \right\|_{\ell^2} +m\|x-\hat{x}\|_{\ell^2}+ 2\varepsilon \\
            &\leq \left\| \phi^{(1)}_r(\hat{x}) -\left( 1-\frac{\|x\|_{\ell^1}-r}{b-r}\right)\phi_r^{(1)}(x) \right\|_{\ell^2} +3\varepsilon.
        \end{align*}
        Thus, we have traced the error in $\phi$ back to an error in the network $\phi_r^{(1)}$, which can be bounded by
        \begin{align*}
            &\left\| \phi^{(1)}_r(\hat{x}) -\left( 1-\frac{\|x\|_{\ell^1}-r}{b-r}\right)\phi_r^{(1)}(x) \right\|_{\ell^2} \\&\leq \left\| \phi^{(1)}_r(\hat{x}) -\phi_r^{(1)}(x) \right\|_{\ell^2} + \left(\frac{\|x\|_{\ell^1}-r}{b-r}\right) \left\| 
             \phi_r^{(1)}(x)\right\|_{\ell^2}\\
             &\leq L_{\phi^{(1)}_r}\| x-\hat{x} \| + \left\| \phi_r^{(1)}(x) \right\|_{\ell^2}\\
             &\leq \varepsilon + \left( \left\| 
             \phi_r^{(1)}(\hat{x})\right\|_{\ell^2} + L_{\phi^{(1)}_r}\|x-\hat{x}\|_{\ell^2} \right) \\
             &\leq 2\varepsilon + \|-\nabla V (\hat{x}) - m\hat{x} - \phi_r^{(1)}(\hat{x})\|_{\ell^2} + \left\| 
             - \nabla V(\hat{x}) - m \hat{x}\right\|_{\ell^2} \\
             &\leq 3\varepsilon + G\|x\|_{\ell^2}.
        \end{align*}
        Collecting the previous inequalities, we finally obtain
        \begin{align*}
            \|-\nabla V(x) - \phi(x) \|_{\ell^2} \leq 9\varepsilon + G\|x\|_{\ell^2},
        \end{align*}
        yielding the claim.
\end{proof}

\begin{proof}[Proof of \cref{thm:dnn_lmc_lipschitz2}]
    We assume without loss of generality that the unique minimizer of $V$ is given by $x^\ast = 0$ and hence $\nabla V (0) = 0$.
    For any $r > 0$ there exists by \cref{prop:combined_assump} a ReLU FCNN $\phi_r$ with $N$ parameters such that 
    \begin{align*}
         &\| -\nabla V - \mathcal{R}\phi_r \|_{L^{\infty}(B^2_r(0))} \leq \varepsilon/\sqrt{2d}, \\
         &\|-\nabla V (x) - \mathcal{R}\phi_r(x)\|_{\ell^2} \leq 9\varepsilon/\sqrt{2d} + \sqrt{M^2-m^2}\|x\|_{\ell^2}, \quad \forall x\in \mathbb{R^d}.
    \end{align*}
Let $\Phi \coloneqq \{ \phi_r \}_{k=0}^{K-1}$. Let $Y^{\Phi}\colon \Omega \times [0,Kh]\rightarrow\mathbb{R}^d$ be the stochastic process driven by $\Phi$, i.e.
\begin{equation*}
    Y^{\Phi}_t = Y_0 + \int_{0}^t \mathcal{R}\phi_{\frac{1}{h}\chi_h(s)}\left(Y^{\Phi}_{\chi_h(s)}\right) \mathrm{d}s + \sqrt{2}W_t.
\end{equation*}
   Let $G\coloneqq \sqrt{M^2-m^2}$. By \cref{prop:subgaussian_nn_process}, $Y^{\Phi}_{kh} \sim \mu^{\Phi}_{kh}$ is sub-Gaussian for all $k=0,\ldots,K$ with variance proxy $\sigma_k^2$ bounded by
    \begin{equation}
        \sigma_k^2 \leq \frac{2h}{1-(c + hG)}+\dfrac{81h^2\varepsilon^2}{2d} +\sigma_0^2 = 
\dfrac{2}{m-\sqrt{M^2-m^2}} +\dfrac{81h^2\varepsilon^2}{2d} +\sigma_0^2.
    \end{equation}
    Note that the right hand side is in $\mathcal{O}(1)$ as $\varepsilon\rightarrow 0$, $d\rightarrow \infty$. In particular, we find for $\varepsilon \in (0,1)$ and $d \geq 1$ that
    \begin{align*}
        \sigma_k \leq \dfrac{2}{m-\sqrt{M^2-m^2}} +\dfrac{81h^2}{2} +\sigma_0^2  = : \sigma.
    \end{align*}
    Now, note that $\mu^{\Phi}_{kh}(\R^d\setminus B_r(0)) = \mathbb P(\|X\|_2 \geq r)$ and apply \cref{prop:lp_subgaussian} to get $\mu^{\Phi}_{kh}(\R^d \setminus B^2_r(0)) \leq \exp\left(-\frac{r^2}{2d^2\sigma_k^2}\right)$. Hence, using the fact that for all $a,b\in\mathbb{R}$ it holds that $(a+b)^2 \leq 2(a^2 + b^2)$, we have
    \ifnum\classstyle=1
    \begin{align*}
        \| -\nabla V - \mathcal{R}\phi_{r} \|^2_{L^2_{\mu_{kh}^\Phi} (\mathbb{R}^d;\mathbb{R}^d)} 
        &= \int_{B^2_r(0)} \| -\nabla V(x) - \mathcal{R}\phi_{r} \|^2_{\ell^2} \mathrm{d}\mu_{kh}^\Phi(x) + \int_{\mathbb{R}^d \backslash B^2_r(0)} \| -\nabla V(x) - \mathcal{R}\phi_{r} \|^2_{\ell^2} \mathrm{d}\mu_{kh}^\Phi(x)\\
        &= \int_{B^2_r(0)} \left(\sqrt{d}\frac{\varepsilon}{\sqrt{2d}}\right)^2 \mathrm{d}\mu_{kh}^\Phi(x) + \int_{\mathbb{R}^d \backslash B^2_r(0)} (9\varepsilon/\sqrt{2d} + \sqrt{M^2-m^2}\|x\|_{\ell^2})^2 \mathrm{d}\mu_{kh}^\Phi(x)\\
        &\leq \frac{\varepsilon^2}{2} + \dfrac{81\varepsilon^2}{d}\mu^{\Phi}_{kh}(\R^d \setminus B^2_r(0)) +  2\int_{\mathbb{R}^d \backslash B^2_r(0)}  (M^2-m^2)\|x\|_{\ell^2}^2 \mathrm{d}\mu_{kh}^\Phi(x)\\
        &\leq \frac{\varepsilon^2}{2} + \left(\dfrac{81\varepsilon^2}{d} +  2(M^2-m^2)(2d^2\sigma^2 + r^2)\right)\exp\left( -\dfrac{r^2}{2d^2\sigma^2} \right),
    \end{align*}
    \fi
    \ifnum\classstyle=0
    {\footnotesize
    \begin{align*}
        \| -\nabla V - &\mathcal{R}\phi_{r} \|^2_{L^2_{\mu_{kh}^\Phi} (\mathbb{R}^d;\mathbb{R}^d)} \\
        &= \int_{B^2_r(0)} \| -\nabla V(x) - \mathcal{R}\phi_{r} \|^2_{\ell^2} \mathrm{d}\mu_{kh}^\Phi(x) + \int_{\mathbb{R}^d \backslash B^2_r(0)} \| -\nabla V(x) - \mathcal{R}\phi_{r} \|^2_{\ell^2} \mathrm{d}\mu_{kh}^\Phi(x)\\
        &= \int_{B^2_r(0)} \left(\sqrt{d}\frac{\varepsilon}{\sqrt{2d}}\right)^2 \mathrm{d}\mu_{kh}^\Phi(x) + \int_{\mathbb{R}^d \backslash B^2_r(0)} (9\varepsilon/\sqrt{2d} + \sqrt{M^2-m^2}\|x\|_{\ell^2})^2 \mathrm{d}\mu_{kh}^\Phi(x)\\
        &\leq \frac{\varepsilon^2}{2} + \dfrac{81\varepsilon^2}{d}\mu^{\Phi}_{kh}(\R^d \setminus B^2_r(0)) +  2\int_{\mathbb{R}^d \backslash B^2_r(0)}  (M^2-m^2)\|x\|_{\ell^2}^2 \mathrm{d}\mu_{kh}^\Phi(x)\\
        &\leq \frac{\varepsilon^2}{2} + \left(\dfrac{81\varepsilon^2}{d} +  2(M^2-m^2)(2d^2\sigma^2 + r^2)\right)\exp\left( -\dfrac{r^2}{2d^2\sigma^2} \right),
    \end{align*}}
    \fi
where the layer cake representation (\Cref{lem:layer_cake}) was used in the last inequality. Now, we use \cref{lem:r_inequality} with $a = 2d^2\sigma^2$, $b= 2(M^2-m^2)$, $c = 81\varepsilon^2/d + 4(M^2-m^2)d^2\sigma^2$, to see that
\begin{align*}
    \left(\dfrac{81\varepsilon^2}{d} +  2(M^2-m^2)(2d^2\sigma^2 + r^2)\right)\exp\left( -\dfrac{r^2}{2d^2\sigma^2} \right) < \frac{\varepsilon^2}{2}
\end{align*}
is satisfied if 
\begin{equation}\label{eq:r_suitable}
    \begin{aligned}
        r = \left[ 2d^2\sigma^2\ln \left( \dfrac{4(81\varepsilon^2/d + 4(M^2-m^2)d^2\sigma^2) + 16\cdot 4(M^2-m^2)d^2\sigma^2}{\varepsilon^4} \right) \right]^{\frac{1}{2}}.
    \end{aligned}
\end{equation}
Hence, it holds for $r$ as in \eqref{eq:r_suitable} that
\begin{align*}
    \| -\nabla V - \mathcal{R}\phi_{r} \|_{L^2_{\mu_{kh}^\Phi} (\mathbb{R}^d;\mathbb{R}^d)} < \varepsilon.
\end{align*}
Applying \cref{thm:nn_approx_lmc} for $h< \frac{2}{m+M}$ and $\Phi = \{\phi_r\}_{k=1}^{K}$, we get
    \begin{equation*}
        \mathcal W_2(\mu_{\infty},\mu_{Kh}^{\Phi}) \leq (1-mh)^K\mathcal W_2(\mu_{\infty}, \mu_{0}) + \frac{7\sqrt 2}{6} \frac M m \sqrt{hd} + \dfrac{1-(1-mh)^K}{m}\varepsilon.
    \end{equation*}
    Finally, \cref{prop:resnet_equals_driven_process} guarantees the existence of a network $\psi$ with number of parameters equal to the number of parameters of $\phi_r$ such that for $\Psi \coloneqq \{\psi\}_{k=1}^{K}$ it holds that $\mu^{\Phi}_{Kh} = \mu^{\Psi}$, where $\xi = (\xi_1,\ldots,\xi_{K})$ and $\tilde{\mathcal{R}}\Psi(Y_0,\xi) \sim \mu^{\Psi}$. Towards the asymptotic complexity of $r$ for $\varepsilon\rightarrow 0$ and $d \rightarrow \infty$, note that
\begin{equation}\label{eq:r_complexity}
    \begin{aligned}
        &\left[ 2d^2\sigma^2\ln \left( \dfrac{4(81\varepsilon^2/d + 4(M^2-m^2)d^2\sigma^2) + 16\cdot 4(M^2-m^2)d^2\sigma^2}{\varepsilon^4} \right) \right]^{\frac{1}{2}}\\
    &= \left[ 2d^2\sigma^2\ln \left( \dfrac{324\varepsilon^2/d + 80(M^2-m^2)d^2\sigma^2}{\varepsilon^4} \right) \right]^{\frac{1}{2}} \\
    &\leq \left[ 2d^2\sigma^2\ln \left( (324 + 80(M^2-m^2)\sigma^2)\varepsilon^{-4}d^2 \right) \right]^{\frac{1}{2}}\\
&=\sqrt{2}d\sigma\left[\ln\left( 324 + 80(M^2-m^2)\sigma^2\right)+\ln\left(\varepsilon^{-4}d^2 \right)\right]^{\frac{1}{2}} \\
&\leq \sqrt{2}d\sigma\left[\ln\left( 324 + 80(M^2-m^2)\sigma^2\right)\right]^{\frac{1}{2}}+\sqrt{2}d\sigma\left[\ln\left(\varepsilon^{-4}d^2 \right)\right]^{\frac{1}{2}} \\
&= \mathcal{O}\left(d\right) + \mathcal{O}\left(d\ln\left( \varepsilon^{-4}d^2 \right)^{\frac{1}{2}}\right)\\
&= \mathcal{O}\left(d \left( 1 + \ln\left( \varepsilon^{-4}d^2 \right)^{\frac{1}{2}}\right)\right).
    \end{aligned}
\end{equation}
    This yields the claim.
\end{proof}

\subsection{Auxiliary results: layer cake representation and suitable radius}

\begin{lemma}[Layer cake representation]\label{lem:layer_cake}
    It holds that
    \begin{equation*}
        \int_{\mathbb{R}^d \backslash B^2_r(0)}  \|x\|_{\ell^2}^2 \mathrm{d}\mu_{kh}^\Phi(x) = (2d^2\sigma^2 + r^2)\exp\left(-\frac{r^2}{2d^2\sigma^2}\right).
    \end{equation*}
\end{lemma}

\begin{proof}
    The layer cake representation asserts that
    \begin{equation}
        \int_{\Omega} f(x) \mathrm{d}\mu(x) = \int_{0}^{\infty} \mu(\{ x\in\Omega | f(x)>s \}) \mathrm{d}s.
    \end{equation}
    With $\Omega = \{\| x \|_{\ell^2}>r\}$ and $f(x)=\|x\|^2_{\ell^2}$, it holds that
    \ifnum\classstyle=1
    \begin{equation}
        \begin{aligned}
        \label{eq: int of norm sq}
        \int_{\| x \|_{\ell^2}>r} \|x\|_{\ell^2}^2 \mathrm{d}\mu(x) &= \int_{0}^{\infty} \mu(\{ \| x \|_{\ell^2}>r,  \| x \|^2_{\ell^2}>s \}) \mathrm{d}s\\
        &= \int_{r^2}^{\infty} \mu(\{ \| x \|^2_{\ell^2}>s \}) \mathrm{d}s + \int_{0}^{r^2} \mu(\{ \| x \|_{\ell^2}>r \}) \mathrm{d}s \\
        &\leq \int_{r^2}^{\infty} \exp\left(-\frac{s}{2d^2\sigma^2}\right)\mathrm{d}s + r^2\exp\left(-\frac{r^2}{2d^2\sigma^2}\right) \\&=   2d^2\sigma^2\exp\left(-\frac{r^2}{2d^2\sigma^2}\right) + r^2\exp\left(-\frac{r^2}{2d^2\sigma^2}\right) = (2d^2\sigma^2 + r^2)\exp\left(-\frac{r^2}{2d^2\sigma^2}\right).
        \end{aligned}
    \end{equation}
    \fi
    \ifnum\classstyle=0
    {\footnotesize
    \begin{equation}
        \begin{aligned}
        \label{eq: int of norm sq}
        \int_{\| x \|_{\ell^2}>r} \|x\|_{\ell^2}^2 \mathrm{d}\mu(x) &= \int_{0}^{\infty} \mu(\{ \| x \|_{\ell^2}>r,  \| x \|^2_{\ell^2}>s \}) \mathrm{d}s\\
        &= \int_{r^2}^{\infty} \mu(\{ \| x \|^2_{\ell^2}>s \}) \mathrm{d}s + \int_{0}^{r^2} \mu(\{ \| x \|_{\ell^2}>r \}) \mathrm{d}s \\
        &\leq \int_{r^2}^{\infty} \exp\left(-\frac{s}{2d^2\sigma^2}\right)\mathrm{d}s + r^2\exp\left(-\frac{r^2}{2d^2\sigma^2}\right) \\&=   2d^2\sigma^2\exp\left(-\frac{r^2}{2d^2\sigma^2}\right) + r^2\exp\left(-\frac{r^2}{2d^2\sigma^2}\right) = (2d^2\sigma^2 + r^2)\exp\left(-\frac{r^2}{2d^2\sigma^2}\right).
        \end{aligned}
    \end{equation}}
    \fi
\end{proof}

\begin{lemma}[Condition on radius]\label{lem:r_inequality}
        Let $a,b,c > 0$. Then 
        \begin{align*}
            \left[ br^2+c \right]e^{-r^2/a} < \frac{\varepsilon^2}{2}
        \end{align*}
        is fulfilled for
        \begin{align*}
            r = \sqrt{a\ln\left( \dfrac{4c + 16ba}{\varepsilon^4} \right)}.
        \end{align*}
    \end{lemma}

\begin{proof} Let $x>0$ be such that $r = \sqrt{a\ln(x)}$. Then,
    \begin{align}
        \left[ ba\ln(x)+c \right]\frac{1}{x} < \frac{\varepsilon^2}{2}
    \end{align}
    holds if
    \begin{align}
        \dfrac{c}{x} < \frac{\varepsilon^2}{4}\qquad\text{and}\qquad
        ba\frac{\ln(x)}{x} < \frac{\varepsilon^2}{4}.
    \end{align}
    The first of these inequalities leads to the condition 
    \begin{align}
        x > \dfrac{4c}{\varepsilon^2}.
    \end{align}
    Regarding the second inequality, note that $\ln(x)/x < 1/\sqrt{x}$ for all $x>0$. Hence, the second inequality is satisfied if $\frac{ba}{\sqrt{x}} < \frac{\varepsilon^2}{4}$ which is equivalent to $x > \frac{16ba}{\varepsilon^4}$. In total, 
    \begin{equation}
     x > \max\left\{  \dfrac{4c}{\varepsilon^2},\frac{16ba}{\varepsilon^4}\right\}  
    \end{equation}
    leads to the desired inequalities. For $\varepsilon<1$, this is satisfied for
    \begin{equation}
        x = \dfrac{4c+16ba}{\varepsilon^4}.
    \end{equation}
    Using this $x$ in the expression for $r$ leads to
    \begin{align}
        r = \sqrt{a\ln\left( \dfrac{4c+16ba}{\varepsilon^4} \right)}.
        \end{align}
\end{proof}
\section{Standard Results for FCNNs}

\begin{proposition}[Representation of the identity by ReLU neural networks]\label{prop:represent_identity}
    Let $\sigma$ be the ReLU activation function and $d \in \N$. Then, there exists a fully connected neural network $\phi$ with $\mathcal P(\phi) = 4d$ and $\mathbb L(\phi)=1$ such that
    \begin{equation*}
        \mathcal R\phi = \operatorname{Id}_{\R^d}.
    \end{equation*}
\end{proposition}
\begin{proof}
    Let
    \begin{align*}
        A_0 &:= \left(\begin{array}{c}
            I_d\\
            \hline
            -I_d
        \end{array}\right) \in \R^{2d \times d},\\
        b_0 &:= 0_{\R^{2d}},\\
        A_1 &:= \left(\begin{array}{c|c}
            I_d & -I_d 
        \end{array}\right) \in \R^{d \times 2d},\\
        b_1 &:= 0_{\R^d}.
    \end{align*}
    Then, define $\phi := ((A_0, b_0), (A_1, b_1))$. We have
    \begin{align*}
        \mathcal R\phi(x) &= A_1\sigma(A_0 x + b_0) + b_1\\
        &= \left(\begin{array}{c|c}
            I_d & -I_d 
        \end{array}\right) \sigma\left(\begin{array}{c}
            x\\
            \hline
            -x
        \end{array}\right)\\
        &= \sigma(x) - \sigma(-x)\\
        &= x.
    \end{align*}
\end{proof}

\begin{proposition}[{Sum of neural networks \cite[Lemma~2.17]{gribonval2020approximation}}]\label{prop:sum_nn}
    Let $\phi_1, \dots, \phi_n$ be $n$ fully connected neural networks with $d$ inputs and $k$ outputs. Then, there exists a neural network $\psi$ such that
    \begin{align*}
        \mathcal R\psi &= \sum_{i=1}^n \mathcal R\phi_i,\\
        \mathcal P(\psi) &\leq \delta + \sum_{i=1}^n \mathcal P(\phi_i),\\
        \mathbb L(\psi) &= \max_{i=1\dots n}\mathbb L(\phi_i),
    \end{align*}
    where $\delta := \min(d,k)(\max_i \mathbb L(\phi_i) - \min_i \mathbb L(\phi_i))$.
\end{proposition}

\begin{proposition}[Representation of the Euclidean 1-norm]
\label{prop:nn_ell1}
    There exists a fully connected neural network $\phi$ with $d$ inputs and $1$ output such that
    \begin{equation*}
        \mathcal R\phi(x) = \|x\|_{\ell^1}
    \end{equation*}
    and
    \begin{align*}
        \mathcal P(\phi) &=4d,\\
        \mathbb L(\phi) &= 1.
    \end{align*}
\end{proposition}
\begin{proof}
    Let
    \begin{align*}
        A_0 &:= \left(\begin{array}{c}
            I_d\\
            \hline
            -I_d
        \end{array}\right) \in \R^{2d \times d},\\
        b_0 &:= 0_{\R^{2d}},\\
        A_1 &:=(1,\dots,1) \in \R^{1 \times 2d},\\
        b_1 &:= 0.
    \end{align*}
    Then, define $\phi := ((A_0, b_0), (A_1, b_1))$. We have
    \begin{align*}
        \mathcal R\phi(x) &= A_1\sigma(A_0 x + b_0) + b_1\\
        &=(1, \dots, 1) \cdot \sigma\left(\begin{array}{c}
            x\\
            \hline
            -x
        \end{array}\right)\\
        &= \sum_{k=1}^d (\sigma(x_k) + \sigma(-x_k))\\
        &= \sum_{k=1}^d |x_k|\\
        &= \|x\|_{\ell^1}.
    \end{align*}
\end{proof}

\begin{proposition}[Approximation of the indicator function on $B_r(0)$]
\label{prop:nn_approx_indicator}
    Let $\delta > 0$. Then, there exists a fully connected neural network $\phi$ with $d$ inputs and $1$ output such that
    \begin{equation*}
        \mathcal R(\phi)(x) = \begin{cases}
            1 &\text{if}~x \in B_r(0)\\
            \frac{r+\delta-\|x\|_{\ell^1}}{\delta} &\text{if}~x \in B_r(0) \cap B_{r+\delta}(0)\\
            0 &\text{otherwise}
        \end{cases}
    \end{equation*}
    and
    \begin{align*}
        \mathcal P(\phi) &=4d + 7,\\
        \mathbb L(\phi) &=3.
    \end{align*}
\end{proposition}
\begin{proof}
    By \cref{prop:nn_ell1}, if
    \begin{align*}
        A_0 &:= \left(\begin{array}{c}
            I_d\\
            \hline
            -I_d
        \end{array}\right) \in \R^{2d \times d},\\
        b_0 &:= 0_{\R^{2d}},\\
        A_1 &:= (1,\dots,1) \in \R^{1 \times 2d},\\
        b_1 &:= 0,
    \end{align*}
    then the fully connected neural network $\tilde\phi := ((A_0, b_0), (A_1, b_1))$ satisfies $\mathcal R \tilde\phi = \|\cdot\|_{\ell^1}$. Let
    \begin{align*}
        A_2 &:= \left(\begin{array}{c}
            1\\
            \hline
            1
        \end{array}\right) \in \R^{2 \times 1},\\
        b_2 &:= \left(\begin{array}{c}
            -r\\
            -(r+\delta)
        \end{array}\right) \in \R^{2},\\
        A_3&:=-\frac{1}{\delta} (1,-1) \in \R^{1 \times 2},\\
        b_3 &:= 1.
    \end{align*}
    Then, define $\phi := ((A_0,b_0),(A_1,b_1),(A_2,b_2),(A_3,b_3))$. We have
    \begin{align*}
        \mathcal{R}\phi(x) &=A_3\sigma(A_2\|x\|_{\ell^1}+b_2) + b_3\\
        &=A_3\sigma(A_2\|x\|_{\ell^1}  + b_2) + b_3\\
        &=-\frac{1}{\delta} (1,-1) \cdot \sigma\left(\begin{array}{c}
            \|x\|_{\ell^1}-r\\
            \hline
            \|x\|_{\ell^1}-(r+\delta)
        \end{array}\right) + 1\\
        &= -\frac{\sigma(\|x\|_{\ell^1}-r)-\sigma(\|x\|_{\ell^1}-(r+\delta))}{\delta} + 1.
    \end{align*}
\end{proof}

\begin{lemma}[{\cite[Proposition 3]{yarotsky2017error}}]\label{lem:network mult}
    For $M>0$ and $\varepsilon\in (0,1)$ there is a ReLU network $\phi^{\text{mult}}$ with $\mathcal{R}\phi:\mathbb{R}^2 \to \mathbb{R}$ such that
    \begin{enumerate}
        \item $|\mathcal{R}\phi^{\text{mult}} (x,y) - xy| \leq \varepsilon$ for all $x,y \in [-M,M]$,
        \item $\mathcal{R}\phi^{\text{mult}}(x,y) = 0$, if $x=0$ or $y=0$,
        \item $\mathbb{L}(\phi^{\text{mult}}), \mathcal{P}(\phi^{\text{mult}}) \in \mathcal{O}(\log (1/\varepsilon) + \log(M))$.
    \end{enumerate}
\end{lemma}

\begin{proposition}[Element-wise multiplication of neural networks]\label{prop:nn_elementwise_mult}
    Let $\varepsilon > 0$. Then there exists a fully connected neural network $\phi$ such that
    \begin{equation*}
        \left\|xy - \mathcal R\phi((x,y)) \right\|_{\ell^1} \leq \varepsilon
    \end{equation*}
    where $x \in [A_1,B_1]$ and $y \in [A_2,B_2]^d$, and $\phi$ satisfies
    \begin{align*}
        \mathcal P(\phi) &= \mathcal O(d\log(dr/\varepsilon)),\\
        \mathbb L(\phi) &= \mathcal O(\log(dr/\varepsilon)),
    \end{align*}
    where $t = \min(A_1,A_2)$ and $r = \max((B_1 - t), (B_2 - t))$.
\end{proposition}
\begin{proof}
    Let $\widetilde\times = ((A_0, b_0),\dots,(A_L,b_L))$ \cref{lem:network mult} be a fully connected neural network which satisfies, for all $\alpha \in [A_1,B_1]$ and $\beta \in [A_2,B_2]$,
    \begin{equation*}
        |\mathcal R\widetilde\times(\alpha,\beta) - \alpha\beta|\leq \frac \varepsilon d.
    \end{equation*}

    \paragraph{Construction of a neural network that can extract $(x,y_j)$}
    Let $j \in \{1,\dots,d\}$ and $z := (x,y) \in \R^{d+1}$. Let
    \begin{align*}
        \Gamma_j &:= E_{1,1} + E_{2,j+1} \in \R^{2 \times (d+1)},\\
        \widetilde A_0^j &:= \left(\begin{array}{c}
            \Gamma_j\\
            \hline
            -\Gamma_j
        \end{array}\right) \in \R^{4 \times (d+1)},\\
        \widetilde b_0^j &:= 0_{\R^{4}},\\
        \widetilde A_1^j &:= \left(\begin{array}{cc|cc}
            1 & 0 & -1 & 0\\
            0 & 1 & 0 & -1
        \end{array}\right) \in \R^{2 \times 4},\\
        \widetilde b_1^j &:= 0_{\R^2},
    \end{align*}
    where $E_{m,n} = e_m^T e_n$ is zero everywhere except in $(m,n)$ where it is equal to 1. Let \ifnum\classstyle=0 \linebreak \fi $\widetilde\phi_j := ((\widetilde A_0^j, \widetilde b_0^j),(\widetilde A_1^j, \widetilde b_1^j))$. We have $\mathcal P(\widetilde \phi_j) = 8$ and $\mathbb L(\widetilde\phi_j) = 1$. Then,
    \begin{align*}
        \mathcal R\widetilde\phi_j(z) &= \widetilde A_1^j\sigma(\widetilde A_0^j z + \widetilde b_0^j) + \widetilde b_1^j\\
        &= \left(\begin{array}{cc|cc}
            1 & 0 & -1 & 0\\
            0 & 1 & 0 & -1
        \end{array}\right) \cdot \left(\begin{array}{c}
            \sigma(x)\\
            \sigma(y_j)\\
            \sigma(-x)\\
            \sigma(-y_j)
        \end{array}\right)\\
        &= \left(\begin{array}{c}
            \sigma(x) - \sigma(-x)\\
            \sigma(y_j) - \sigma(-y_j)
        \end{array}\right)\\
        &= \left(\begin{array}{c}
            x\\
            y_j
        \end{array}\right).
    \end{align*}

    \paragraph{Construction of a neural network that can approximate $xy_j$}
    By concatenating $\widetilde\phi_j$ and $\widetilde\times$ \cite[Definition~2.2]{petersen2018optimal}, there exists a neural network $\phi_j := ((\widetilde A_0^j, \widetilde b_0^j), (A_0\widetilde A_1^j, A_0\widetilde b_1^j + b_0), (A_1, b_1), \dots, (A_L, b_L))$ which satisfies $\mathcal R\phi_j = \mathcal R\widetilde\times \circ \mathcal R\widetilde\phi_j$, $\mathcal P(\phi_j) \leq 2(8 + \mathcal P(\widetilde\times))$
    and $\mathbb L(\phi_j) = \mathbb L(\widetilde\times) + 1$. And so, for $z = (x,y_j) \in [A_1,B_1] \times [A_2,B_2]$,
    \begin{align*}
        |\mathcal R\phi_j(x,y) - xy_j| = |\mathcal R\widetilde\times(x,y_j) - xy_j| \leq \frac \varepsilon d.
    \end{align*}

    \paragraph{Parallelization of $\phi_j$} By parallelizing the $(\phi_j)_{j=1}^d$ \cite[Defintion~2.7]{petersen2018optimal}, there exists a neural network $\phi$ which satisfies
    \begin{align*}
        \mathcal R\phi &= (\mathcal R\phi_1,\dots,\mathcal R\phi_d),\\
        \mathcal P(\phi) &= \sum_{j=1}^d \mathcal P(\phi_j)\\
        &= 2d(8 + \mathcal P(\widetilde\times)),\\
        \mathbb L(\phi) &= \mathbb L(\widetilde\times) + 1.
    \end{align*}
    Finally, for $x \in [A_1,B_1]$ and $y \in [A_2,B_2]^d$,
    \begin{align*}
        \left\|xy - \mathcal R\phi((x,y)) \right\|_{\ell^1} = \sum_{j=1}^d |\mathcal R\phi_j((x,y_j)) - xy_j| \leq \varepsilon.
    \end{align*}
\end{proof}

\bibliographystyle{abbrv}
\bibliography{ref}

\begin{thebibliography}{10}

\bibitem{altschuler2022concentration}
J.~M. Altschuler and K.~Talwar.
\newblock Concentration of the langevin algorithm's stationary distribution, 2022.

\bibitem{ardizzone2018analyzing}
L.~Ardizzone, J.~Kruse, C.~Rother, and U.~Köthe.
\newblock Analyzing inverse problems with invertible neural networks.
\newblock In {\em International Conference on Learning Representations}, 2019.

\bibitem{bach2017breaking}
F.~Bach.
\newblock Breaking the curse of dimensionality with convex neural networks.
\newblock {\em The Journal of Machine Learning Research}, 18(1):629--681, 2017.

\bibitem{barron1993universal}
A.~R. Barron.
\newblock Universal approximation bounds for superpositions of a sigmoidal function.
\newblock {\em IEEE Transactions on Information theory}, 39(3):930--945, 1993.

\bibitem{barron1994approximation}
A.~R. Barron.
\newblock Approximation and estimation bounds for artificial neural networks.
\newblock {\em Machine learning}, 14:115--133, 1994.

\bibitem{Berner_2022}
J.~Berner, P.~Grohs, G.~Kutyniok, and P.~Petersen.
\newblock {\em The Modern Mathematics of Deep Learning}, page 1–111.
\newblock Cambridge University Press, Dec. 2022.

\bibitem{bolcskei2019optimal}
H.~Bolcskei, P.~Grohs, G.~Kutyniok, and P.~Petersen.
\newblock Optimal approximation with sparsely connected deep neural networks.
\newblock {\em SIAM Journal on Mathematics of Data Science}, 1(1):8--45, 2019.

\bibitem{bond2021deep}
S.~Bond-Taylor, A.~Leach, Y.~Long, and C.~G. Willcocks.
\newblock Deep generative modelling: A comparative review of vaes, gans, normalizing flows, energy-based and autoregressive models.
\newblock {\em IEEE transactions on pattern analysis and machine intelligence}, 2021.

\bibitem{Hackbusch}
D.~Braess and W.~Hackbusch.
\newblock A new convergence proof for the multigrid method including the {V}-cycle.
\newblock {\em Siam Journal on Numerical Analysis - SIAM J NUMER ANAL}, 20:967--975, 10 1983.

\bibitem{brooks2011handbook}
S.~Brooks, A.~Gelman, G.~Jones, and X.-L. Meng.
\newblock {\em Handbook of Markov Chain Monte Carlo}.
\newblock CRC press, 2011.

\bibitem{carstensen2012review}
C.~Carstensen, M.~Eigel, R.~H. Hoppe, and C.~L{\"o}bhard.
\newblock A review of unified a posteriori finite element error control.
\newblock {\em Numerical Mathematics: Theory, Methods and Applications}, 5(4):509--558, 2012.

\bibitem{chau2021stochastic}
N.~H. Chau, {\'E}.~Moulines, M.~R{\'a}sonyi, S.~Sabanis, and Y.~Zhang.
\newblock On stochastic gradient langevin dynamics with dependent data streams: The fully nonconvex case.
\newblock {\em SIAM Journal on Mathematics of Data Science}, 3(3):959--986, 2021.

\bibitem{chen2018neural}
R.~T. Chen, Y.~Rubanova, J.~Bettencourt, and D.~K. Duvenaud.
\newblock Neural ordinary differential equations.
\newblock {\em Advances in neural information processing systems}, 31, 2018.

\bibitem{cheng2018kl}
X.~Cheng and P.~Bartlett.
\newblock Convergence of langevin mcmc in kl-divergence.
\newblock In F.~Janoos, M.~Mohri, and K.~Sridharan, editors, {\em Proceedings of Algorithmic Learning Theory}, volume~83 of {\em Proceedings of Machine Learning Research}, pages 186--211. PMLR, 07--09 Apr 2018.

\bibitem{cheng2018sharp}
X.~Cheng, N.~S. Chatterji, Y.~Abbasi-Yadkori, P.~L. Bartlett, and M.~I. Jordan.
\newblock Sharp convergence rates for langevin dynamics in the nonconvex setting.
\newblock {\em arXiv preprint arXiv:1805.01648}, 2018.

\bibitem{Dalayan2016theoretical}
A.~S. Dalalyan.
\newblock {Theoretical Guarantees for Approximate Sampling from Smooth and Log-Concave Densities}.
\newblock {\em Journal of the Royal Statistical Society Series B: Statistical Methodology}, 79(3):651--676, 04 2016.

\bibitem{dalalyan2017stronger}
A.~S. Dalalyan.
\newblock Further and stronger analogy between sampling and optimization: Langevin monte carlo and gradient descent, 2017.

\bibitem{Dalalyan2017UserfriendlyGF}
A.~S. Dalalyan and A.~G. Karagulyan.
\newblock User-friendly guarantees for the langevin monte carlo with inaccurate gradient.
\newblock {\em ArXiv}, abs/1710.00095, 2017.

\bibitem{dandekar2022bayesian}
R.~Dandekar, K.~Chung, V.~Dixit, M.~Tarek, A.~Garcia-Valadez, K.~V. Vemula, and C.~Rackauckas.
\newblock Bayesian neural ordinary differential equations, 2022.

\bibitem{del2006sequential}
P.~Del~Moral, A.~Doucet, and A.~Jasra.
\newblock Sequential monte carlo samplers.
\newblock {\em Journal of the Royal Statistical Society Series B: Statistical Methodology}, 68(3):411--436, 2006.

\bibitem{durmus2019high}
A.~Durmus and {\'E}.~Moulines.
\newblock High-dimensional bayesian inference via the unadjusted langevin algorithm.
\newblock {\em Bernoulli}, 2016.

\bibitem{durmus2017unadjusted}
A.~Durmus and {\'E}.~Moulines.
\newblock {Nonasymptotic convergence analysis for the unadjusted Langevin algorithm}.
\newblock {\em The Annals of Applied Probability}, 27(3):1551 -- 1587, 2017.

\bibitem{dwivedi2018fast}
R.~Dwivedi, Y.~Chen, M.~J. Wainwright, and B.~Yu.
\newblock Log-concave sampling: Metropolis-hastings algorithms are fast!
\newblock In S.~Bubeck, V.~Perchet, and P.~Rigollet, editors, {\em Proceedings of the 31st Conference On Learning Theory}, volume~75 of {\em Proceedings of Machine Learning Research}, pages 793--797. PMLR, 06--09 Jul 2018.

\bibitem{eigel2023adaptive}
M.~Eigel, N.~Farchmin, S.~Heidenreich, and P.~Trunschke.
\newblock Adaptive nonintrusive reconstruction of solutions to high-dimensional parametric pdes.
\newblock {\em SIAM Journal on Scientific Computing}, 45(2):A457--A479, 2023.

\bibitem{EigelGittelson2014asgfem}
M.~Eigel, C.~J. Gittelson, C.~Schwab, and E.~Zander.
\newblock Adaptive stochastic galerkin {FEM}.
\newblock {\em Computer Methods in Applied Mechanics and Engineering}, 270:247--269, Mar. 2014.

\bibitem{stochafem}
M.~Eigel, C.~J. Gittelson, C.~Schwab, and E.~Zander.
\newblock A convergent adaptive stochastic {Galerkin} finite element method with quasi-optimal spatial meshes.
\newblock {\em ESAIM: M2AN}, 49(5):1367--1398, 2015.

\bibitem{eigel2024less}
M.~Eigel, R.~Gruhlke, and D.~Sommer.
\newblock Less interaction with forward models in langevin dynamics: Enrichment and homotopy.
\newblock {\em SIAM Journal on Applied Dynamical Systems}, 23(3):1870--1908, 2024.

\bibitem{eigel2018adaptive}
M.~Eigel, M.~Marschall, M.~Pfeffer, and R.~Schneider.
\newblock Adaptive stochastic galerkin fem for lognormal coefficients in hierarchical tensor representations, 2020.

\bibitem{elbrachter2022dnn}
D.~Elbr{\"a}chter, P.~Grohs, A.~Jentzen, and C.~Schwab.
\newblock Dnn expression rate analysis of high-dimensional pdes: Application to option pricing.
\newblock {\em Constructive Approximation}, 55(1):3--71, 2022.

\bibitem{flamary2021pot}
R.~Flamary, N.~Courty, A.~Gramfort, M.~Z. Alaya, A.~Boisbunon, S.~Chambon, L.~Chapel, A.~Corenflos, K.~Fatras, N.~Fournier, L.~Gautheron, N.~T. Gayraud, H.~Janati, A.~Rakotomamonjy, I.~Redko, A.~Rolet, A.~Schutz, V.~Seguy, D.~J. Sutherland, R.~Tavenard, A.~Tong, and T.~Vayer.
\newblock Pot: Python optimal transport.
\newblock {\em Journal of Machine Learning Research}, 22(78):1--8, 2021.

\bibitem{2013PASP..125..306F}
D.~{Foreman-Mackey}, D.~W. {Hogg}, D.~{Lang}, and J.~{Goodman}.
\newblock {emcee: The MCMC Hammer}.
\newblock {\em Publications of the Astronomical Society of the Pacific}, 125(925):306, Mar. 2013.

\bibitem{garbuno2020interacting}
A.~Garbuno-Inigo, F.~Hoffmann, W.~Li, and A.~M. Stuart.
\newblock Interacting {L}angevin diffusions: {G}radient structure and ensemble {K}alman sampler.
\newblock {\em SIAM Journal on Applied Dynamical Systems}, 19(1):412--441, 2020.

\bibitem{garbuno2020affine}
A.~Garbuno-Inigo, N.~Nusken, and S.~Reich.
\newblock Affine invariant interacting {L}angevin dynamics for {B}ayesian inference.
\newblock {\em SIAM Journal on Applied Dynamical Systems}, 19(3):1633--1658, 2020.

\bibitem{moritz}
M.~Geist, P.~Petersen, M.~Raslan, R.~Schneider, and G.~Kutyniok.
\newblock Numerical solution of the parametric diffusion equation by deep neural networks.
\newblock {\em Journal of Scientific Computing}, 88:22--88, 2021.

\bibitem{girolami2011riemann}
M.~Girolami and B.~Calderhead.
\newblock Riemann manifold langevin and hamiltonian monte carlo methods.
\newblock {\em Journal of the Royal Statistical Society Series B: Statistical Methodology}, 73(2):123--214, 2011.

\bibitem{goodfellow2016nips}
I.~Goodfellow.
\newblock Nips 2016 tutorial: Generative adversarial networks.
\newblock {\em arXiv preprint arXiv:1701.00160}, 2016.

\bibitem{goodman2010ensemble}
J.~Goodman and J.~Weare.
\newblock Ensemble samplers with affine invariance.
\newblock {\em Communications in applied mathematics and computational science}, 5(1):65--80, 2010.

\bibitem{gribonval2020approximation}
R.~Gribonval, G.~Kutyniok, M.~Nielsen, and F.~Voigtlaender.
\newblock Approximation spaces of deep neural networks, 2020.

\bibitem{gühring_raslan_kutyniok_2022}
I.~Gühring, M.~Raslan, and G.~Kutyniok.
\newblock {\em Expressivity of Deep Neural Networks}, page 149–199.
\newblock Cambridge University Press, 2022.

\bibitem{he2016deep}
K.~He, X.~Zhang, S.~Ren, and J.~Sun.
\newblock Deep residual learning for image recognition.
\newblock In {\em Proceedings of the IEEE conference on computer vision and pattern recognition}, pages 770--778, 2016.

\bibitem{cosi}
C.~Heiß, I.~Gühring, and M.~Eigel.
\newblock Multilevel cnns for parametric pdes.
\newblock {\em Journal of Machine Learning Research}, 24(373):1--42, 2023.

\bibitem{Jentzen_2021}
A.~Jentzen, D.~Salimova, and T.~Welti.
\newblock A proof that deep artificial neural networks overcome the curse of dimensionality in the numerical approximation of kolmogorov partial differential equations with constant diffusion and nonlinear drift coefficients.
\newblock {\em Communications in Mathematical Sciences}, 19(5):1167--1205, 2021.

\bibitem{kaipio2006statistical}
J.~Kaipio and E.~Somersalo.
\newblock {\em Statistical and computational inverse problems}, volume 160.
\newblock Springer Science \& Business Media, 2006.

\bibitem{NIPS2016_ddeebdee}
D.~P. Kingma, T.~Salimans, R.~Jozefowicz, X.~Chen, I.~Sutskever, and M.~Welling.
\newblock Improved variational inference with inverse autoregressive flow.
\newblock In D.~Lee, M.~Sugiyama, U.~Luxburg, I.~Guyon, and R.~Garnett, editors, {\em Advances in Neural Information Processing Systems}, volume~29. Curran Associates, Inc., 2016.

\bibitem{kingma2013auto}
D.~P. Kingma and M.~Welling.
\newblock Auto-encoding variational bayes.
\newblock {\em arXiv preprint arXiv:1312.6114}, 2013.

\bibitem{gitta}
G.~Kytyniok, P.~Petersen, M.~Raslan, and R.~Schneider.
\newblock A theoretical analysis of deep neural networks and parametric pdes.
\newblock {\em Constructive Approximation}, 55:73--125, 2022.

\bibitem{freeenergy}
T.~Lelièvre, M.~Rousset, and G.~Stoltz.
\newblock {\em Free Energy Computations: A Mathematical Perspective}.
\newblock World Scientific, 06 2010.

\bibitem{LUKE197241}
Y.~L. Luke.
\newblock Inequalities for generalized hypergeometric functions.
\newblock {\em Journal of Approximation Theory}, 5(1):41--65, 1972.

\bibitem{luu2021sampling}
T.~D. Luu, J.~Fadili, and C.~Chesneau.
\newblock Sampling from non-smooth distributions through langevin diffusion.
\newblock {\em Methodology and Computing in Applied Probability}, 23:1173--1201, 2021.

\bibitem{majka2020nonasymptotic}
M.~B. Majka, A.~Mijatović, and Łukasz Szpruch.
\newblock {Nonasymptotic bounds for sampling algorithms without log-concavity}.
\newblock {\em The Annals of Applied Probability}, 30(4):1534 – 1581, 2020.

\bibitem{marcati2022exponential}
C.~Marcati and C.~Schwab.
\newblock Exponential convergence of deep operator networks for elliptic partial differential equations, 2022.

\bibitem{markowich2000trend}
P.~A. Markowich and C.~Villani.
\newblock On the trend to equilibrium for the fokker-planck equation: an interplay between physics and functional analysis.
\newblock {\em Mat. Contemp}, 19:1--29, 2000.

\bibitem{adaptivenochetto}
R.~H. Nochetto, K.~G. Siebert, and A.~Veeser.
\newblock Theory of adaptive finite element methods: An introduction.
\newblock In R.~DeVore and A.~Kunoth, editors, {\em Multiscale, Nonlinear and Adaptive Approximation}, pages 409--542, Berlin, Heidelberg, 2009. Springer Berlin Heidelberg.

\bibitem{perekrestenko2018universal}
D.~Perekrestenko, P.~Grohs, D.~Elbr{\"a}chter, and H.~B{\"o}lcskei.
\newblock The universal approximation power of finite-width deep relu networks.
\newblock {\em arXiv preprint arXiv:1806.01528}, 2018.

\bibitem{petersen2018optimal}
P.~Petersen and F.~Voigtlaender.
\newblock Optimal approximation of piecewise smooth functions using deep relu neural networks.
\newblock {\em Neural Networks}, 108:296--330, 2018.

\bibitem{Petersen2018}
P.~Petersen and F.~Voigtlaender.
\newblock Optimal approximation of piecewise smooth functions using deep {ReLU} neural networks.
\newblock {\em Neural Networks}, 108:296--330, Dec. 2018.

\bibitem{raginsky2017non}
M.~Raginsky, A.~Rakhlin, and M.~Telgarsky.
\newblock Non-convex learning via stochastic gradient langevin dynamics: a nonasymptotic analysis.
\newblock In {\em Conference on Learning Theory}, pages 1674--1703. PMLR, 2017.

\bibitem{pmlr-v32-rezende14}
D.~J. Rezende, S.~Mohamed, and D.~Wierstra.
\newblock Stochastic backpropagation and approximate inference in deep generative models.
\newblock In E.~P. Xing and T.~Jebara, editors, {\em Proceedings of the 31st International Conference on Machine Learning}, volume~32 of {\em Proceedings of Machine Learning Research}, pages 1278--1286, Bejing, China, 22--24 Jun 2014. PMLR.

\bibitem{robert2011mcmc}
C.~Robert and G.~Casella.
\newblock {A Short History of Markov Chain Monte Carlo: Subjective Recollections from Incomplete Data}.
\newblock {\em Statistical Science}, 26(1):102 -- 115, 2011.

\bibitem{roberts2004general}
G.~O. Roberts and J.~S. Rosenthal.
\newblock General state space {M}arkov chains and {MCMC} algorithms.
\newblock {\em Probability surveys}, 1:20--71, 2004.

\bibitem{roberts2002langevin}
G.~O. Roberts and O.~Stramer.
\newblock Langevin diffusions and metropolis-hastings algorithms.
\newblock {\em Methodology and computing in applied probability}, 4:337--357, 2002.

\bibitem{roberts1996convergence}
G.~O. Roberts and R.~L. Tweedie.
\newblock Exponential convergence of langevin distributions and their discrete approximations.
\newblock {\em Bernoulli}, 2(4):341--363, 1996.

\bibitem{rossky1978brownian}
P.~J. Rossky, J.~D. Doll, and H.~L. Friedman.
\newblock Brownian dynamics as smart monte carlo simulation.
\newblock {\em The Journal of Chemical Physics}, 69(10):4628--4633, 1978.

\bibitem{ruthotto2021introduction}
L.~Ruthotto and E.~Haber.
\newblock An introduction to deep generative modeling.
\newblock {\em GAMM-Mitteilungen}, 44(2):e202100008, 2021.

\bibitem{sander2022residual}
M.~Sander, P.~Ablin, and G.~Peyr{\'e}.
\newblock Do residual neural networks discretize neural ordinary differential equations?
\newblock {\em Advances in Neural Information Processing Systems}, 35:36520--36532, 2022.

\bibitem{schutte2024adaptive}
J.~E. Sch{\"u}tte and M.~Eigel.
\newblock Adaptive multilevel neural networks for parametric {PDE}s with error estimation.
\newblock In {\em ICLR 2024 Workshop on AI4DifferentialEquations In Science}, 2024.

\bibitem{schutte2024multilevelcnnsparametricpdes}
J.~E. Schütte and M.~Eigel.
\newblock Multilevel cnns for parametric pdes based on adaptive finite elements, 2024.

\bibitem{song2020score}
Y.~Song, J.~Sohl-Dickstein, D.~P. Kingma, A.~Kumar, S.~Ermon, and B.~Poole.
\newblock Score-based generative modeling through stochastic differential equations.
\newblock {\em arXiv preprint arXiv:2011.13456}, 2020.

\bibitem{stuart2010inverse}
A.~M. Stuart.
\newblock Inverse problems: a {B}ayesian perspective.
\newblock {\em Acta numerica}, 19:451--559, 2010.

\bibitem{yang2022diffusion}
L.~Yang, Z.~Zhang, Y.~Song, S.~Hong, R.~Xu, Y.~Zhao, W.~Zhang, B.~Cui, and M.-H. Yang.
\newblock Diffusion models: A comprehensive survey of methods and applications.
\newblock {\em ACM Computing Surveys}, 2022.

\bibitem{yarotsky2017error}
D.~Yarotsky.
\newblock Error bounds for approximations with deep relu networks.
\newblock {\em Neural Networks}, 94:103--114, 2017.

\bibitem{yarotsky2018optimal}
D.~Yarotsky.
\newblock Optimal approximation of continuous functions by very deep relu networks.
\newblock In {\em Conference on learning theory}, pages 639--649. PMLR, 2018.

\bibitem{yarotsky2020phase}
D.~Yarotsky and A.~Zhevnerchuk.
\newblock The phase diagram of approximation rates for deep neural networks.
\newblock {\em Advances in neural information processing systems}, 33:13005--13015, 2020.

\bibitem{zhang2020wasserstein}
K.~S. Zhang, G.~Peyr{\'e}, J.~Fadili, and M.~Pereyra.
\newblock Wasserstein control of mirror langevin monte carlo.
\newblock In {\em Conference on Learning Theory}, pages 3814--3841. PMLR, 2020.

\bibitem{zhang2023nonasymptotic}
Y.~Zhang, {\"O}.~D. Akyildiz, T.~Damoulas, and S.~Sabanis.
\newblock Nonasymptotic estimates for stochastic gradient langevin dynamics under local conditions in nonconvex optimization.
\newblock {\em Applied Mathematics \& Optimization}, 87(2):25, 2023.

\end{thebibliography}

\end{document}